%% file: neurips_2021_FedDR_FinalVersion.tex
\documentclass{article}

\usepackage[nonatbib,preprint]{neurips_2021}

\usepackage[utf8]{inputenc} 
\usepackage[T1]{fontenc}    
\usepackage{hyperref}       
\usepackage{url}            
\usepackage{booktabs}       
\usepackage{amsfonts}       
\usepackage{nicefrac}       
\usepackage{microtype}      
\usepackage{xcolor}         

\input{my_preamble.tex}

\title{
FedDR -- Randomized Douglas-Rachford Splitting Algorithms for Nonconvex Federated Composite Optimization
}

\author{%
Quoc Tran-Dinh {~\normalfont and~} Nhan H. Pham\\ 
  Department of Statistics and Operations Research, The University of North Carolina at Chapel Hill\\
  318 Hanes Hall, UNC-Chapel Hill, NC 27599-3260 \\
  \texttt{quoctd@email.unc.edu, nhanph@live.unc.edu}
  \And
   Dzung T. Phan {~\normalfont and~}  Lam M. Nguyen\\
   IBM Research, Thomas J. Watson Research Center, Yorktown Heights, NY, USA.\\
   \texttt{phandu@us.ibm.com, lamnguyen.mltd@ibm.com}
}

\begin{document}
\maketitle

\begin{abstract}
We develop two new algorithms, called, \textbf{FedDR} and \textbf{asyncFedDR}, for solving a fundamental nonconvex composite optimization problem in federated learning.
Our algorithms rely on a novel combination between a nonconvex Douglas-Rachford splitting method, randomized block-coordinate strategies, and asynchronous implementation.
They can also handle convex regularizers.
Unlike recent methods in the literature, e.g., FedSplit and FedPD, our algorithms update only a subset of users at each communication round, and possibly in an asynchronous manner, making them more practical.
These new algorithms can handle statistical and system heterogeneity, which are the two main challenges in federated learning, while achieving the best known communication complexity.
In fact, our new algorithms match the communication complexity lower bound up to a constant factor under standard assumptions.
Our numerical experiments illustrate the advantages of our methods over existing algorithms on synthetic and real datasets. 
\end{abstract}

\input{intro.tex}
\input{problem.tex}
\input{main.tex}
\input{experiment.tex}

\begin{ack}
The work of Quoc Tran-Dinh is partially supported by the Office of Naval Research (ONR), grant No. N00014-20-1-2088.
The authors would also like to thank all the anonymous reviewers and the ACs for their constructive comments to improve the paper.
%
%
\end{ack}

\bibliographystyle{plain}

\input{refs_list}

\clearpage
\newpage
\appendix

\input{appendix}

\end{document}

%% file: my_preamble.tex
\usepackage{amsthm}
\usepackage{amsmath}
\usepackage{amssymb}
\usepackage{graphicx}
\usepackage{algorithmicx}
\usepackage{algorithm,algpseudocode}
\usepackage{paralist}
\usepackage{multirow}
\usepackage{tikz}
\usepackage{pgfplots}

\theoremstyle{plain}
\newtheorem{theorem}{Theorem}[section]
\newtheorem{corollary}{Corollary}[section]
\newtheorem{lemma}{Lemma}[section]
\theoremstyle{definition}
\newtheorem{definition}{Definition}[section]
\newtheorem{assumption}{Assumption}[section]
\theoremstyle{remark}
\newtheorem{remark}{Remark}[section]

\newcommand{\R}{\mathbb{R}}

\newcommand{\set}[1]{\left\{#1\right\}}
\newcommand{\sets}[1]{\{#1\}}
\newcommand{\norm}[1]{\left\Vert#1\right\Vert}
\newcommand{\norms}[1]{\Vert#1\Vert}

\newcommand{\Eproof}{\hfill $\square$}
\newcommand{\prox}{\mathrm{prox}}

\newcommand{\argmin}{\mathrm{arg}\!\displaystyle\min}

\newcommand{\dom}[1]{\mathrm{dom}(#1)}

\newcommand{\zero}[1]{{\boldsymbol{0}}}

\newcommand{\Sc}{\mathcal{S}}

\newcommand{\Dc}{\mathcal{D}}
\newcommand{\Lc}{\mathcal{L}}

\newcommand{\Tc}{\mathcal{T}}

\newcommand{\Fc}{\mathcal{F}}
\newcommand{\Gc}{\mathcal{G}}

\newcommand{\Cc}{\mathcal{C}}

\newcommand{\mbf}[1]{\mathbf{#1}}

\newcommand{\iprods}[1]{\langle #1\rangle}

\newcommand{\Exp}[1]{\mathbb{E}\left[#1\right]}

\newcommand{\BigO}[1]{\mathcal{O}\left(#1\right)}

\usepackage[normalem]{ulem}

%% file: intro.tex
\section{Introduction}\label{sec:intro}
Training machine learning models in a centralized fashion becomes more challenging and marginally inaccessible for a large number of users, especially when the size of datasets and models is growing substantially larger.
Consequently, training algorithms using decentralized and distributed approaches comes in as a natural replacement. 
Among several approaches, federated learning (FL) has received tremendous attention in the past few years since it was first introduced in \cite{konevcny2016federated,mcmahan_ramage_2017}.
In this setting, a central server coordinates between many local users (also called agents or devices) to perform their local updates, then the global model will get updated, e.g., by averaging or aggregating local models. 

\noindent\textbf{Challenges.}
FL provides a promising solution for many machine learning applications such as learning over smartphones or across organizations, and internet of things, where privacy protection is one of the most critical requirements.
However, this training mechanism faces a number of fundamental challenges, see, e.g., \cite{mcmahan2021advances}.
First, when the number of users gets substantially large, it creates \textit{communication bottleneck} during model exchange process between server and users. 
Second, the local data stored in each local user may be different in terms of sizes and distribution which poses a challenge: \textit{data or statistical heterogeneity}. 
Third, the variety of users with different local storage, computational power, and network connectivity participating into the system also creates a major challenge, known as \textit{system heterogeneity}. 
This challenge also causes unstable connection between server and users, where some users may be disconnected from the server or simply dropped out during training. 
In practice, we can expect only a subset of users to participate in each round of communication. 
Another challenge in FL is \textit{privacy concern}.
Accessing and sharing local raw data is not permitted in FL.
In addition, distributed methods exchange the objective gradient of local users, and private data can be exposed from the shared model such as the objective gradients \cite{zhu2019deep}. 
Therefore, FL methods normally send the global model to each user at the start of each communication round, each user will perform its local update and send back only the necessary update for aggregation.

\noindent\textbf{Our goal and approach.}
Our goal in this paper is to further and simultaneously address these fundamental challenges by proposing two new algorithms to train the underlying common optimization model in FL.
Our approach relies on a novel combination between randomized block-coordinate strategy, nonconvex Douglas-Rachford (DR) splitting, and asynchronous implementation.
While each individual technique or partial combinations is not new, our combination of three as in this paper appears to be the first in the literature.
To the best of our knowledge, this is the first work developing randomized block-coordinate DR splitting methods for nonconvex composite FL optimization models, and they are fundamentally different from some works in the convex setting, e.g.,  \cite{combettes2018asynchronous, combettes2015stochastic}.

\noindent\textbf{Contribution.}
Our contribution can be summarized as follows.
\begin{compactitem}
\item[(a)] We develop a new FL algorithm, called \textbf{FedDR} (\textbf{Fed}erated \textbf{D}ouglas-\textbf{R}achford), by combining  the well-known DR splitting technique and randomized block-coordinate strategy for the common nonconvex  composite optimization problem in FL. 
Our algorithm can handle nonsmooth convex regularizers and allows inexact evaluation of the underlying proximal operators as in FedProx or FedPD.
It also achieves the best known $\BigO{\varepsilon^{-2}}$ communication complexity for finding a stationary point under standard assumptions (Assumptions~\ref{ass:A1}-\ref{ass:A2}), where $\varepsilon$ is a given accuracy.
More importantly, unlike FedSplit \cite{pathak2020fedsplit} and FedPD \cite{zhang2020fedpd}, which require full user participation to achieve convergence, our analysis does allow partial participation by selecting a subset of users to perform update at each communication round.

\item[(b)] Next, we propose an asynchronous algorithm, \textbf{asyncFedDR}, where each user can asynchronously perform local update and periodically send the update to the server for proximal aggregation. 
We show that \textbf{asyncFedDR} achieves the same communication complexity  $\BigO{\varepsilon^{-2}}$ as \textbf{FedDR} (up to a constant factor) under the same standard assumptions.
This algorithm is expected to simultaneously address all challenges discussed above. 
\end{compactitem}
Let us emphasize some key points of our contribution.
First, the best known $\BigO{\varepsilon^{-2}}$ communication complexity of our methods matches the lower bound complexity up to a constant factor as shown in \cite{zhang2020fedpd}, even with inexact evaluation of the objective proximal operators.
Second, our methods rely on a DR splitting technique for nonconvex optimization and can handle possibly nonsmooth convex regularizers, which allows us to deal with a larger class of applications and with constraints \cite{yuan2020federated}.
Furthermore, it can also handle both statistical and system heterogeneity as discussed in FedSplit \cite{pathak2020fedsplit} and FedPD \cite{zhang2020fedpd}.
However, FedSplit only considers the convex case, and both FedSplit and FedPD require all users to update at each communication round, making them less practical and applicable in FL. 
Our methods only require a subset of users or even one user to participate in each communication round as in FedAvg or FedProx.
In addition, our aggregation step on the server is  different from most existing works due to a proximal step on the regularizer. It is also different from \cite{yuan2020federated}.
Third, as FedProx \cite{Li_MLSYS2020}, we allow inexact evaluation of users' proximal operators with any local solver (e.g., local SGD or variance reduced methods) and with adaptive accuracies. 
Finally, requiring synchronous aggregation at the end of each communication round may lead to slow-down in training due to the heterogeneity in computing power and communication capability of local users. 
It is natural to have asynchronous update from local users as in, e.g., \cite{peng2016arock,Recht2011,stich2018local}. 
Our asynchronous variant, \textbf{asyncFedDR}, can fairly address this challenge. 
Moreover, it uses a general probabilistic model recently introduced in \cite{cannelli2019asynchronous}, which allows us to capture the variety of asynchronous environments and architectures compared to existing methods, e.g., \cite{stich2018local,xie2019asynchronous}. 

\noindent\textbf{Related work and comparison.}
Federated Averaging (FedAvg) is perhaps the earliest method used in FL.
In FedAvg, users perform stochastic gradient descent (SGD) updates for a number of epochs then send updated models to server for aggregation. 
FedAvg's practical performance has been shown in many early works, e.g., \cite{konevcny2016federated, mcmahan2017communication, zhang2016parallel} and tends to become the most popular method for solving FL applications.
\cite{lin2018don} show that local SGD where users perform a number of local updates before global communication takes place as in FedAvg may offer benefit over minibatch SGD. 
Similar comparison between minibatch SGD and local SGD has been done in \cite{woodworth2020minibatch,woodworth2020local}. 
Analyzing convergence of FedAvg was very challenging at its early time due to the complexity in its update as well as data heterogeneity. 
One of the early attempt to show the convergence of FedAvg is in \cite{stich2018local} for convex problems under the iid data setting and a set of assumptions. 
\cite{yu2019parallel} also considers local SGD in the nonconvex setting. 
Without using an additional bounded gradient assumption as in \cite{stich2018local,yu2019parallel}, \cite{wang2018cooperative} improves the complexity for the general nonconvex setting while \cite{haddadpour2019local} uses a Polyak-{\L}ojasiewicz (PL) condition to improve FedAvg's convergence results. 
In heterogeneous data settings, \cite{khaled2019first} analyzes local GD, where users performs gradient descent (GD) updates instead of SGD. 
The analysis of FedAvg for non-iid data is given in \cite{li2019convergence}. 
The analysis of local GD/SGD for nonconvex problems has been studied in \cite{haddadpour2019convergence}. 
However, FedAvg might not converge with non-iid data as shown in \cite{pathak2020fedsplit, zhang2020fedpd, zhao2018federated}. 

FedProx \cite{Li_MLSYS2020} is an extension of FedAvg, which deals with heterogeneity in federated networks by introducing a proximal term to the objective in local updates to improve stability. 
FedProx has been shown to achieve better performance than FedAvg in heterogeneous setting. 
Another method to deal with data heterogeneity is SCAFFOLD \cite{karimireddy2020scaffold} which uses a control variate to correct the ``client-drift" in local update of FedAvg. 
MIME \cite{karimireddy2020mime} is another framework that uses control variate to improve FedAvg for heterogeneous settings. 
However, SCAFFOLD and MIME require to communicate extra information apart from local models. 
Compared to aforementioned works, our methods deal with nonconvex problems under standard assumptions and with composite settings.

FedSplit \cite{pathak2020fedsplit} instead employs a Peaceman-Rachford splitting scheme to solve a constrained reformulation of the original problem. 
In fact, FedSplit can be viewed as a variant of Tseng's splitting scheme \cite{Bauschke2011} applied to FL.
\cite{pathak2020fedsplit} show that FedSplit can find a solution of the FL problem under only convexity without imposing any additional assumptions on system or data homogeneity. 
\cite{zhang2020fedpd} proposes FedPD, which is essentially a variant of the standard augmented Lagrangian method in nonlinear optimization.
Other algorithms for FL can be found, e.g., in \cite{charles2021convergence, gorbunov2021local, haddadpour2021federated, hanzely2020lower, li2021fedbn, yu2020fed+}.

Our approach in this paper relies on nonconvex DR splitting method, which can handle the heterogeneity as discussed in \cite{pathak2020fedsplit}.
While the DR method is classical, its nonconvex variants have been recently studied e.g., in \cite{dao2019lyapunov,li2016douglas,themelis2020douglas}.
However, the combination of DR and randomized block-coordinate strategy remains limited \cite{combettes2018asynchronous,combettes2015stochastic} even in the convex settings.
Alternatively, asynchronous algorithms have been extensively studied in the literature, also for FL, see, e.g., \cite{Bertsekas1989b,peng2016arock,Recht2011}. 
For instance, a recent work \cite{xie2019asynchronous} analyzes an asynchronous variant of FedAvg under bounded delay assumption and constraint on the number of local updates. 
\cite{stich2018local} proposes an asynchronous local SGD to solve convex problems under iid data. 
However, to our best knowledge, there exists no asynchronous method using DR splitting techniques with convergence guarantee for FL.
In addition, most existing algorithms only focus on non-composite settings.
Hence, our work here appears to be the first.

\noindent\textbf{Content.}
The rest of this paper is organized as follows.
Section~\ref{sec:problem} states our FL optimization model and our assumptions.
Section~\ref{sec:FedDR} develops \textbf{FedDR} and analyzes its convergence.
Section~\ref{sec:asynFedDR} considers an asynchronous variant, \textbf{asyncFedDR}.
Section~\ref{sec:num_exp} is devoted for numerical experiments.
Due to space limit, all technical details and proofs can be found in Supplementary Document (Supp. Doc.).

%% file: problem.tex
\section{Nonconvex Optimization Models in Federated Learning}\label{sec:problem}
The underlying optimization model of many FL applications can be written into the following form:
\begin{equation}\label{eq:fed_prob}
\displaystyle\min_{x\in\R^p} \Big\{ F(x) := f(x) + g(x) = \frac{1}{n}\sum_{i=1}^n f_i(x) + g(x) \Big\},
\end{equation}
where $n$ is the number of users, and each $f_i$ is a local loss of the $i$-th user, which is assumed to be nonconvex and $L$-smooth (see Assumptions~\ref{ass:A1} and \ref{ass:A2} below), and $g$ is a proper, closed, and convex regularizer.
Apart from these assumptions, we will not make any additional assumption on \eqref{eq:fed_prob}.
We emphasize that the use of regularizers $g$ has been motivated in several works, including \cite{yuan2020federated}.

Let $\dom{F} := \set{x\in\R^p : F(x) < +\infty}$ be the domain of $F$ and $\partial{g}$ be the subdifferential of $g$ \cite{Bauschke2011}.
Since \eqref{eq:fed_prob} is nonconvex, we only expect to find a stationary point, which is characterized by the following optimality condition.
\begin{definition}
If $0 \in \nabla f(x^{*}) + \partial{g}(x^{*}) $, then $x^{*}$ is called a [first-order] stationary point of \eqref{eq:fed_prob}.
\end{definition}

The algorithms for solving \eqref{eq:fed_prob} developed in this paper will rely on the following assumptions.
\begin{assumption}[Boundedness from below]\label{ass:A1}
$\dom{F} \neq\emptyset$ and $F^{\star} := \inf_{x\in\R^p}F(x) > -\infty$.
\end{assumption}
\begin{assumption}[$L$-smoothness]\label{ass:A2}
All functions $f_i(\cdot)$ for $i \in [n] := \sets{1,\cdots,n}$ are $L$-smooth, i.e., $f_i$ is continuously differentiable and there exists $L \in (0, +\infty)$ such that
\begin{equation}\label{eq:L_smooth}
\norms{\nabla{f}_i(x) - \nabla{f}_i(y)} \leq L\norms{x - y}, \quad \forall x, y\in\dom{f_i}.
\end{equation}
\end{assumption}
Assumptions~\ref{ass:A1} and \ref{ass:A2} are very standard in nonconvex optimization.  
Assumption~\ref{ass:A1} guarantees the well-definedness of \eqref{eq:fed_prob} and is independent of algorithms.
Assuming the same Lipschitz constant $L$ for all $f_i$ is not restrictive since if $f_i$ is $L_i$-smooth, then by scaling variables of its constrained formulation (see \eqref{eq:constr_reform} in Supp. Doc.), we can get the same Lipschitz constant $L$ of all $f_i$.

\noindent\textbf{Proximal operators and evaluation.}
Our methods make use of the proximal operators of both $f_i$ and $g$.
Although $f_i$ is $L$-smooth and nonconvex, we still define its proximal operator as
\begin{equation}\label{eq:prox_oper}
\prox_{\eta f_i}(x) := \argmin_{y}\big\{f_i(y) + \tfrac{1}{2\eta}\norms{y - x}^2 \big\},
\end{equation}
where $\eta > 0$.
Even $f_i$ is nonconvex, under Assumption~\ref{ass:A2}, if we choose $0 < \eta < \frac{1}{L}$, then $\prox_{\eta f_i}$ is well-defined and single-valued.
Evaluating $\prox_{\eta f_i}$ requires to solve a strongly convex program.
If $\prox_{\eta f_i}$ can only be computed approximately up to an accuracy $\epsilon \geq 0$ to obtain $z$, denoted by $x_{+} :\approx \prox_{\eta f_i}(x)$, if $\norms{x_{+} - \prox_{\eta f_i}(x)} \leq \epsilon_i$.
Note that instead of absolute error, one can also use a relative error as $\norms{x_{+} - \prox_{\eta f_i}(x)} \leq \epsilon_i\norms{x_{+} - x}$ as in \cite{Rockafellar1976b}.
For the convex function $g$, its proximal operator $\prox_{\eta g}$ is defined in the same way as \eqref{eq:prox_oper}.
Evaluating $\prox_{\eta f_i}$ can be done by various existing methods, including local SGD and accelerated GD-type algorithms.
However, this is not our focus in this paper, and therefore we do not specify the subsolver for evaluating $\prox_{\eta f_i}$.

\noindent\textbf{Gradient mapping.}
As usual, let us define the following gradient mapping of $F$ in \eqref{eq:fed_prob}.
\begin{equation}\label{eq:grad_mapping}
\Gc_{\eta}(x) := \tfrac{1}{\eta}\big(x - \prox_{\eta g}(x - \eta \nabla{f}(x))\big), \quad \eta > 0.
\end{equation}
Then, the optimality condition $0 \in \nabla{f}(x^{*}) + \partial{g}(x^{*})$ of \eqref{eq:fed_prob} is equivalent to $\Gc_{\eta}(x^{*}) = 0$.
However, in practice, we often wish to find an $\varepsilon$-approximate stationary point to \eqref{eq:fed_prob} defined as follows.

\begin{definition}\label{de:eps_sol}
If $\tilde{x} \in \dom{F}$ satisfies $\mathbb{E}\big[ {\norm{\Gc_{\eta}(\tilde{x})}^2}\big] \le \varepsilon^2$, then $\tilde{x}$ is called an $\varepsilon$-stationary point of \eqref{eq:fed_prob}, where the expectation is taken overall the randomness generated by the underlying algorithm.
\end{definition}
Note that, for $\Gc_{\eta}(\tilde{x})$ to be well-defined, we require $\tilde{x}\in\dom{F}$.
In our algorithms below, this requirement is fulfilled if $\tilde{x}\in\dom{f}$, which is often satisfied in practice as $\dom{f} = \R^p$.

%% file: main.tex
\section{FedDR Algorithm and Its Convergence Guarantee}\label{sec:FedDR}
Prior to our work, FedSplit \cite{pathak2020fedsplit} exploits similar update steps as ours by adopting the Peaceman-Rachford splitting method to solve the convex and non-composite instances of \eqref{eq:fed_prob}.
FedSplit can overcome some of the key challenges as discussed earlier.
Following this idea, we take the advantages of the DR splitting method to first derive a new variant to handle the nonconvex composite problem \eqref{eq:fed_prob}.
This new algorithm is synchronous and we call it \textbf{FedDR}.
The central idea is as follows: First, we reformulate \eqref{eq:fed_prob} into \eqref{eq:fed_comp_prob} by duplicating variables.
Next, we apply a DR splitting scheme to the resulting problem.
Finally, we combine such a scheme with a randomized block-coordinate strategy.

The complete algorithm is presented in Algorithm~\ref{alg:A1}, where its full derivation is in Supp. Doc. \ref{subsec:FedDR_derivation}.

\begin{algorithm}[hpt!]\caption{(FL with Randomized DR (\textbf{FedDR}))}\label{alg:A1}
\normalsize
\begin{algorithmic}[1]
   \State\label{step:i0}{\bfseries Initialization:} Take $x^0 \in\dom{F}$. Choose $\eta > 0$ and $\alpha > 0$, and accuracies $\epsilon_{i,0}\geq 0$ ($i \in [n]$).
   \Statex\hspace{1.6ex} Initialize the server with $\bar{x}^0 := x^0$ and  $\tilde{x}^0 := x^0$.
   \Statex\hspace{1.6ex} Initialize each user $i\in [n]$ with $y_i^0 := x^0$, \ $x_i^0 :\approx \prox_{\eta f_i}(y_i^0)$, and $\hat{x}^0_i := 2x_i^0 - y_i^0$.
   \State\hspace{0ex}\label{step:o1}{\bfseries For $k := 0,\cdots, K$ do}
   \vspace{0.25ex}   
   \State\hspace{2ex}\label{step:o2}  [\textit{Active users}] Generate a proper realization $\Sc_k\subseteq [n]$ of $\hat{\Sc}$ (see Assumption~\ref{ass:A3}).
      \vspace{0.25ex}   
   \State\hspace{2ex}\label{step:o2}  [\textit{Communication}] Each user $i\in\Sc_k$ receives $\bar{x}^k$ from the server.  
   \State\hspace{2ex}\label{step:o3}  [\textit{Local update}] \textbf{For each user $i\in\Sc_k$ do}: Choose $\epsilon_{i,k+1} \geq 0$ and update
   \vspace{-0.5ex}
   \begin{equation*}
   y^{k+1}_i  :=  y_i^k + \alpha(\bar{x}^{k} - x^{k}_i), \quad x_i^{k+1} :\approx \prox_{\eta f_i}(y_i^{k+1}), \quad\text{and} \quad \hat{x}^{k+1}_i  := 2x^{k+1}_i - y^{k+1}_i.
   \vspace{-1ex}
    \end{equation*}
   \State\hspace{2ex}\label{step:04}[\textit{Communication}] Each user $i \in \Sc_k$ sends $\Delta{\hat{x}}^k_i := \hat{x}^{k+1}_i - \hat{x}^k_i$ back to the server. 
   \State\hspace{2ex}\label{step:o5}[\textit{Sever aggregation}] The server aggregates $\tilde{x}^{k+1} \! := \tilde{x}^k + \frac{1}{n}\sum_{i\in\Sc_k}\Delta{\hat{x}}^{k}_i$.
   \State\hspace{2ex}\label{step:o6}[\textit{Sever update}] Then, the sever updates $\bar{x}^{k+1} := \prox_{\eta g}\big( \tilde{x}^{k+1}\big)$.
   \State\hspace{0ex}{\bfseries End For}
\end{algorithmic}
\end{algorithm}

Let us make the following remarks.
Firstly, \textbf{FedDR} mainly updates of three sequences $\sets{\bar{x}^k}$, $\sets{x^k_i}$ and $\sets{y^k_i}$. 
While $\bar{x}^k$ is an averaged model to approximately minimize the global objective function $F$, $x^k_i$ act as local models trying to optimize a regularized local loss function w.r.t. its local data distribution, and $y^k_i$ keeps track of the residuals from the local models to the global one.
Secondly, we allow $x_i^{k}$ to be an approximation of $\prox_{\eta f_i}(y_i^{k})$ up to an accuracy $\epsilon_{i,k} \geq 0$ as defined in \eqref{eq:prox_oper}, i.e., $\Vert x_i^{k} - \prox_{\eta f_i}(y^{k}_i)\Vert \leq \epsilon_{i,k}$ for all $i\in [n]$ if $k= 0$ and for all $i \in \Sc_{k-1}$ if $k > 0$.
If $\epsilon_{i,k} = 0$, then we get the exact evaluation $x_i^k := \prox_{\eta f_i}(y_i^{k})$.
Approximately evaluating $\prox_{\eta f_i}$ can be done, e.g., by local SGD as in FedAvg.
Thirdly, Algorithm~\ref{alg:A1} is different from existing randomized proximal gradient-based methods since we rely on a DR splitting scheme and can handle composite settings.
Here, three iterates $y^k_i$, $x_i^k$, and $\hat{x}^k_i$ at  Step~\ref{step:o3} are updated sequentially, making it challenging to analyze convergence.
Lastly, the subset of active users $\Sc_k$ is sampled from a random set-valued mapping $\hat{\Sc}$.
As specified in Assumption~\ref{ass:A3}, this sampling mechanism covers a wide range of sampling strategies.  
Clearly, if $\Sc_k = [n]$ and $g = 0$, then Algorithm~\ref{alg:A1} reduces to FedSplit, but for the nonconvex case.
Hence, our convergence guarantee below remains applicable, and the guarantee is sure.
Note that both our model \eqref{eq:fed_prob} and Algorithm \ref{alg:A1} are completely different from \cite{yuan2020federated}.

\subsection{Convergence of Algorithm~\ref{alg:A1}}
Let us consider a proper sampling scheme $\hat{\Sc}$ of $[n]$, which is a random set-valued mapping with values in $2^{[n]}$, the collection of all subsets of $[n]$.
Let $\Sc_k$ be an iid realization of $\hat{\Sc}$ and $\Fc_k := \sigma(\Sc_0, \cdots, \Sc_{k})$ be the $\sigma$-algebra generated by $\Sc_0, \cdots, \Sc_{k}$.
We first impose the following assumption about the distribution of our sampling scheme $\hat{\Sc}$.

\begin{assumption}\label{ass:A3}
There exist $\mbf{p}_1,\cdots, \mbf{p}_n > 0$ such that $\mathbb{P}\big( i \in \hat{\Sc}\big) = \mbf{p}_i > 0$ for all $i\in [n]$.
\end{assumption}
This assumption covers a large class of sampling schemes as discussed in \cite{richtarik2016parallel}, including non-overlapping uniform and doubly uniform.
This assumption guarantees that every user has a non-negligible probability to be updated.
Note that $\mbf{p}_i = \sum_{\Sc : i\in \Sc}\mathbb{P}(\Sc)$ due to Assumption~\ref{ass:A3}. 
For the sake of notation, we also denote $\hat{\mbf{p}} : = \min\{  \mbf{p}_i  :  i \in [n] \} > 0$.

The following theorem characterizes convergence of Algorithm~\ref{alg:A1} with inexact evaluation of $\prox_{\eta f_i}$. 
Due to space limit, we refer the reader to Lemma~\ref{lem:dr_key_est_inexact} in Sup. Doc. for more details about the choice of stepsizes and related constants. 
The proof of this theorem is defered to Sup. Doc.~\ref{subsec:inexact_alg1}.

\begin{theorem}\label{thm:feddr_convergence_inexact}
Suppose that Assumptions~\ref{ass:A1}, \ref{ass:A2}, and \ref{ass:A3} hold.
Let $\{(x^{k}_i,y^{k}_i, \hat{x}^{k}_i,\bar{x}^{k})\}$ be generated by Algorithm~\ref{alg:A1} using stepsizes $\alpha$ and $\eta$ defined in \eqref{eq:stepsizes_choice}. 
Then, the following holds
\begin{equation}\label{eq:dr_thm1_convergence_inexact}
\frac{1}{K+1}\sum_{k=0}^K \Exp{\norms{\Gc_{\eta}(\bar{x}^k)}^2} \leq  \frac{C_1[ F(x^0) - F^{\star}] }{K+1} +  \frac{1}{n(K+1)}\sum_{k=0}^K\sum_{i=1}^n\big(C_2\epsilon_{i,k}^2 + C_3\epsilon_{i,k+1}^2 \big), 
\end{equation}
where $\beta$, $\rho_1$, and $\rho_2$ are explicitly defined by \eqref{eq:beta_consts}, and
\begin{equation*}
\arraycolsep=0.2em
\begin{array}{lcl}
C_1 := \frac{2(1+\eta L)^2(1+\gamma_2) }{\eta^2\beta}, \quad C_2 := \rho_1C_1, \ \ \text{and} \ \  C_3 :=  \rho_2C_1 +  \frac{(1 + \eta L)^2(1+\gamma_2)}{\eta^2\gamma_2}.
\end{array}
\end{equation*}
Let $\tilde{x}^K$ be  selected uniformly at random from $\sets{\bar{x}^0, \cdots, \bar{x}^K}$ as the output of Algorithm~\ref{alg:A1}.
Let the accuracies $\epsilon_{i,k}$ for all $i\in [n]$ and $k\geq 0$ at Step~\ref{step:o3} be chosen such that $\frac{1}{n}\sum_{i=1}^n\sum_{k=0}^{K+1}\epsilon_{i,k}^2 \leq M$ for a given constant $M > 0$ and all $K\geq 0$.
Then, if we run Algorithm~\ref{alg:A1} for at most 
\begin{equation*}
K := \left\lfloor \frac{C_1[ F(x^0) - F^{\star}] + (C_2 + C_3)M}{\varepsilon^2} \right\rfloor \equiv \BigO{\varepsilon^{-2}} 
\end{equation*}
iterations, then $\tilde{x}^K$ is an $\varepsilon$-stationary point of \eqref{eq:fed_prob} in the sense of Definition~\ref{de:eps_sol}.
\end{theorem}

\begin{remark}\label{re:accuracies}[\textbf{Choice of accuracies $\epsilon_i^k$}]
To guarantee $\frac{1}{n}\sum_{i=1}^n\sum_{k=0}^{K+1}\epsilon_{i,k}^2 \leq M$ in Theorem~\ref{thm:feddr_convergence_inexact} for a given constant $M > 0$ and for all $K\geq 0$, one can choose, e.g., $\epsilon_{i,k}^2 := \frac{M}{2(k+1)^2}$ for all $i\in [n]$ and $k\geq 0$.
In this case, we can easily show that $\frac{1}{n}\sum_{i=1}^n\sum_{k=0}^{K+1}\epsilon_{i,k}^2 = \frac{M}{2}\sum_{k=0}^{K+1}\frac{1}{(k+1)^2} \leq M$.
Note that, instead of using absolute accuracies, one can also use relative accuracies as $\norms{\epsilon_{i,k}}^2 \le \theta \norms{x^{k+1}_i - x^k_i}^2$ for a given constant $\theta > 0$, which is more practical, while still achieving a similar convergence guarantee.
Such an idea has been widely used in the literature, including \cite{liu2021acceleration} (see Supp. Doc.~\ref{subsec:warmstart}).
\end{remark}

\begin{remark}[\textbf{Comparison}]
Since \eqref{eq:fed_prob} is nonconvex, our $\BigO{\varepsilon^{-2}}$ communication complexity is the state-of-the-art, matching the lower bound complexity (up to a constant factor) \cite{zhang2020fedpd}.
However, different from the convergence analysis of FedSplit and FedPD \cite{zhang2020fedpd}, our flexible sampling scheme allows us to update a subset of users at each round and still obtains convergence.
This can potentially further resolve the communication bottleneck \cite{li2020federated}.
We note that FedSplit is a variant of the Peaceman-Rachford splitting method, i.e. $\alpha = 2$ and only considers convex non-composite case while we use a relaxation parameter $\alpha < 2$ and for a more general nonconvex composite problem \eqref{eq:fed_prob}.
\end{remark}

The following corollary specifies the convergence of Algorithm~\ref{alg:A1} with a specific choice of stepsizes and exact evaluation of $\prox_{\eta f_i}$, whose proof is in Sup. Doc.~\ref{subsec:proof_cor1}.

\begin{corollary}\label{cor:special_case}
Suppose that Assumptions~\ref{ass:A1}, \ref{ass:A2}, and \ref{ass:A3} hold.
Let $\{(x^{k}_i,y^{k}_i, \hat{x}^{k}_i,\bar{x}^{k})\}$ be generated by Algorithm~\ref{alg:A1} using stepsizes $\alpha = 1$, $\eta = \frac{1}{3L}$, and $p_i = \frac{1}{n}$. Under exact evaluation of $\prox_{\eta f_i}$, i.e. $\epsilon_{i,k} = 0$ for all $i \in [n]$ and $k \ge 0$, the following bound holds
\begin{equation}\label{eq:dr_cor1_convergence_exact}
\frac{1}{K+1}\sum_{k=0}^K \Exp{\norms{\Gc_{\eta}(\bar{x}^k)}^2} \leq  \frac{160 Ln}{3(K+1)}[ F(x^0) - F^{\star}].
\end{equation}
Let $\tilde{x}^K$ be  selected uniformly at random from $\sets{\bar{x}^0, \cdots, \bar{x}^K}$ as the output of Algorithm~\ref{alg:A1}. 
Then after at most
\begin{equation*}
K := \left\lfloor \frac{160Ln[ F(x^0) - F^{\star}]}{3\varepsilon^2} \right\rfloor \equiv \BigO{\varepsilon^{-2}},
\end{equation*}
communication rounds, $\tilde{x}^K$ becomes an $\varepsilon$-stationary point of $\eqref{eq:fed_prob}$ (defined by Definition~\ref{de:eps_sol}).
\end{corollary}

\section{AsyncFedDR and Its Convergence Guarantee}\label{sec:asynFedDR}
\noindent\textbf{Motivation.}
Although \textbf{FedDR} has been shown to converge, it is more practical to account for the system heterogeneity of local users. 
Requiring synchronous aggregation at the end of each communication round may lead to slow down in training. 
It is natural to have asynchronous update from local users as seen, e.g., in \cite{Recht2011,stich2018local}. 
However, asynchronous implementation remains limited in FL.
Here, we propose \textbf{asyncFedDR}, an asynchronous variant of \textbf{FedDR}, and analyze its convergence guarantee.
For the sake of our analysis, we only consider $\Sc_k := \sets{i_k}$, the exact evaluation of $\prox_{\eta f_i}$, and bounded delay, but extensions to general $\Sc_k$ and inexact $\prox_{\eta f_i}$ are similar to  Algorithm~\ref{alg:A1}.

\subsection{Derivation of asyncFedDR}\label{subsec:asynFedDR_derivation}
Let us first explain the main idea of \textbf{asyncFedDR}.
At each iteration $k$, each user receives a delay copy $\bar{x}^{k - d^k_{i_k}}$ of $\bar{x}^k$ from the server with a delay $d^k_{i_k}$.
The active user $i_k$ will update its own local model $(y_i^k, x^k_i, \hat{x}^k_i)$ in an asynchronous mode without waiting for others to complete.
Once completing its update, user $i_k$ just sends an increment $\Delta{\hat{x}}^k_{i_k}$ to the server to update the global model, while others may be reading.
Overall, the complete \textbf{asyncFedDR} is presented in Algorithm~\ref{alg:A2}. 

\begin{algorithm}[pht!]\caption{(Asynchronous FedDR (\textbf{asyncFedDR}))}\label{alg:A2}
\normalsize
\begin{algorithmic}[1]
\State{\bfseries Initialization:} Take $x^0 \!\in\!\dom{F}$ and choose $\eta > 0$ and $\alpha > 0$.
\Statex\hspace{0.7ex} Initialize the server with $\bar{x}^0 := x^0$ and $\tilde{x}^0 := 0$.
\Statex\hspace{0.7ex} Initialize each user $i\in [n]$ with $y_i^0 := x^0$, \ $x_i^0 := \prox_{\eta f_i}(y_i^0)$, and $\hat{x}^0_i := 2x_i^0 - y_i^0$.
\State\hspace{0ex}\label{step:A2o1a}{\bfseries For $k := 0,\cdots, K$ do}
\State\hspace{1ex}\label{step:A2o2}Select $i_k$ such that $(i_k,d^k)$ is a realization of $(\hat{i}_k,\hat{d}^k)$.
\State\hspace{1ex}\label{step:A2o1}[\textit{Communication}] User $i_k$ receives $\bar{x}^{k-d^k_{i_k}}$, a delayed version of $\bar{x}^k$ with the delay $d^k_{i_k}$.
\State\hspace{1ex}\label{step:A2o3}[\textit{Local update}] User $i_k$ updates
\vspace{-1ex}
\hspace{2ex}
\begin{equation*}
y^{k+1}_{i_k} :=  y_{i_k}^k + \alpha(\bar{x}^{k-d^k_{i_k}} - x^{k}_{i_k}), \quad
x_{i_k}^{k+1} :=  \mathrm{prox}_{\eta f_{i_k}}(y_{i_k}^{k+1}), \  \text{and} \ 
\hat{x}^{k+1}_{i_k} := 2x^{k+1}_{i_k} - y^{k+1}_{i_k}.
\vspace{-1ex}
\end{equation*}
\Statex\hspace{1ex}Other users maintain $y^{k+1}_i := y^k_i$, \ $x_i^{k+1} := x^k_i$, and $\hat{x}^{k+1}_i  := \hat{x}_i^k$ for $i\neq i_k$.
\State\hspace{1ex}\label{step:A204}[\textit{Communication}] User $i_k$ sends $\Delta^k_{i_k} := \hat{x}^{k+1}_{i_k} - \hat{x}^k_{i_k}$ back to the server. 
\State\hspace{1ex}\label{step:A2o5}[\textit{Sever aggregation}] The server aggregates $\tilde{x}^{k+1} := \tilde{x}^k + \frac{1}{n}\Delta_{i_k}^k$.
\State\hspace{1ex}\label{step:A2o6}[\textit{Sever update}] Then, the sever updates $\bar{x}^{k+1} := \prox_{\eta g}\big( \tilde{x}^{k+1} \big)$.
\State\hspace{0ex}{\bfseries End For}
\end{algorithmic}
\end{algorithm}

In our analysis below, a transition of iteration from $k$ to $k+1$ is triggered whenever a user completes its update.
Moreover, at Step \ref{step:A2o2}, active user $i_k$ is chosen from a realization $(i_k,d^k)$ of a joint random vector $(\hat{i}_k, \hat{d}^k)$ at the $k$-th iteration.
Here, we do not assume $i_k$ to be uniformly random or independent of the delay $d^k$.
This allows Algorithm~\ref{alg:A2} to capture the variety of asynchronous implementations and architectures.
Note that $\bar{x}^{k-d^k_{i_k}}$ at Step~\ref{step:A2o1} is a delayed version of $\bar{x}^k$, which only exists on the server when user $i_k$ is reading.
However, right after, $\bar{x}^k$ may be updated by another user. 

\noindent\textbf{Illustrative example.}
To better understand the update of \textbf{asyncFedDR}, Figure~\ref{fig:async_update} depicts a simple scenario where there are 4 users ($C1$ - $C4$) asynchronously perform updates and with $g(\cdot) = 0$. 
At iteration $k=4$, user $C4$ finishes its update so that the server performs updates. 
During this process, user $C1$ starts its update by receiving a global model $\bar{x}^{4-d_{i_4}^4}$ from server which is the average of $(\hat{x}^4_1,\hat{x}^4_2,\hat{x}^4_3,\hat{x}^4_4)$. 
At iteration $t=7$, $C1$ finishes its update. Although $\hat{x}_1$ and $\hat{x}_4$ do not change during this time, i.e. $\hat{x}_1^6 = \hat{x}_1^4$ and $\hat{x}_4^6 = \hat{x}_4^4$, $\hat{x}_2$ and $\hat{x}_3$ have been updated at $k=5,6$ from user $C2$ and $C3$, respectively. 
Therefore, the global model $\bar{x}^k$ used to perform the update at $k=7$ is actually aggregated from $(\hat{x}_1^6,\hat{x}_2^4,\hat{x}_3^5,\hat{x}_4^6)$ not $(\hat{x}_1^6,\hat{x}_2^6,\hat{x}_3^6,\hat{x}_4^6)$. 
In other words, each user receives a delay estimate $\bar{x}^{k - d^k}$ where $d^k=(d^k_1,\cdots,d^k_n)$ is a delay vector and $d^k_i = \max \sets{t\in [k] : i_t = i}$, i.e. the last time $\hat{x}_i$ gets updated up to iteration $k$. 
Note that when $d^k_i = 0$ for all $i$, Algorithm~\ref{alg:A2} reduces to its synchronous variant, i.e. a special variant of Algorithm~\ref{alg:A1} with $\Sc_k = \sets{i_k}$.

\begin{figure*}[hpt!]
\vspace{-0.5ex}
\begin{center}
\includegraphics[width = 0.9\textwidth]{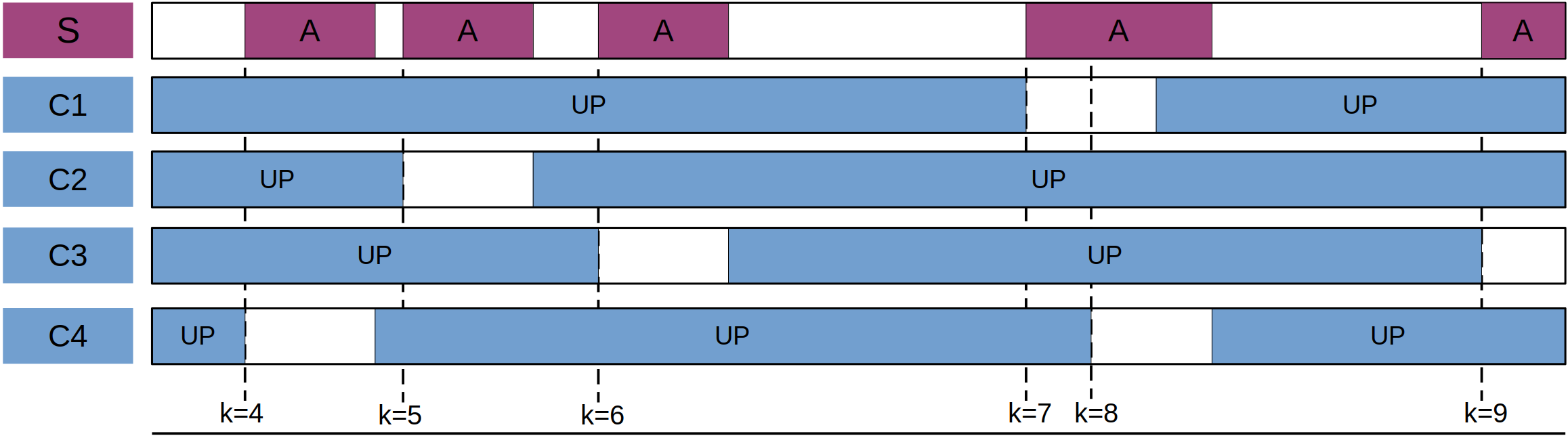}\vspace{-1ex}
\caption{Asynchronous update with 4 users and without regularizer $g$. 
Here, ``A'' blocks represent server process and ``UP'' blocks represent user process; $C_1$-$C_4$ are communication rounds.}\label{fig:async_update}
\end{center}
\vspace{-2ex}
\end{figure*}

\subsection{Convergence analysis}\label{subsec:asynFedDR_analysis}
Since we treat the active user $i_k$ and the delay vector $d^k$  jointly at each iteration $k$ as a realization of a joint random vector $(\hat{i}_k, \hat{d}^k)$, we adopt the probabilistic model from  \cite{cannelli2019asynchronous} to analyze Algorithm~\ref{alg:A2}. 
This new model allows us to cope with a more general class of asynchronous variants of our  method.

\noindent\textbf{Probabilistic model.}
Let $\xi^k := (i_k, d^k)$ be a realization of a random vector $\hat{\xi}^k := (\hat{i}_k, \hat{d}^k)$ containing the user index $\hat{i}_k \in [n]$ and the delay vector $\hat{d}^k = (\hat{d}^k_1,\cdots, \hat{d}^k_n) \in \Dc := \set{0,1,\cdots, \tau}^n$ presented at the $k$-the iteration, respectively. 
We consider $k+1$ random variables that form a random vector $\hat{\xi}^{0:k} := (\hat{\xi}^0, \cdots, \hat{\xi}^k)$.
We also use $\xi^{0:k} = (\xi^0, \xi^1, \cdots, \xi^k)$ for $k+1$ possible values of the random vector $\hat{\xi}^{0:k}$.
Let $\Omega$ be the sample space of all sequences $\omega := \{(i_k, d^k)\}_{k\geq 0}$.
We define a cylinder $\Cc_k(\xi^{0:k}) := \sets{\omega \in\Omega : (\omega_0, \cdots, \omega_k) = \xi^{0:k}}$ and $\Cc_k$ is the set of all possible $\Cc_k(\xi^{0:k})$ when $\xi^t$, $t=0,\cdots, k$ take all possible values, where $\omega_l$ is the $l$-th element of $\omega$.
Let $\Fc_k := \sigma(\Cc_k)$ be the $\sigma$-algebra generated by $\Cc_k$ and $\Fc := \sigma(\cup_{k=0}^{\infty}\Cc_k)$.
For each $\Cc_k(\xi^{0:k})$ we also equip with a probability $\mathbf{p}(\xi^{0:k}) := \mathbb{P}(\Cc_k(\xi^{0:k}))$.
Then, $(\Omega, \Fc, \mathbb{P})$ forms a probability space.
Assume that $\mbf{p}(\xi^{0:k}) := \mathbb{P}(\hat{\xi}^{0:k} = \xi^{0:k}) > 0$.
Our conditional probability is defined as $\mathbf{p}( (i, d) \mid \xi^{0:k}) := \mathbb{P}(\Cc_{k+1}(\xi^{0:k+1}))/\mathbb{P}(\Cc_k(\xi^{0:k}))$, where $\mathbf{p}( (i, d) \mid \xi^{0:k}) := 0$ if $\mathbf{p}(\xi^{0:k}) = 0$.
We refer to Supp. Doc. \ref{apdx:subsec:prob_model} for more details of our probabilistic model.

To analyze Algorithm~\ref{alg:A2}, we impose Assumption~\ref{ass:A4} on the implementation below.
\begin{assumption}\label{ass:A4}
For all $i\in [n]$ and $\omega\in\Omega$, there exists at least one $t \in \sets{0, 1, \cdots, T}$ with $T > 0$, such that
\begin{equation}\label{eq:ass_lower_bound_prob}
\sum_{d\in\Dc} \mathbf{p}((i, d) \mid \xi^{0:k+t-1}) \ge \hat{\mbf{p}} \quad \text{if}~ \mbf{p}(\xi^{0:k}) > 0,
\end{equation}
for a given $\hat{\mbf{p}} > 0$ and any $k\geq 0$.
Assume also that $d^k_i \le \tau$ and  $d^k_{i_k} = 0$ for all $k \geq 0$ and $i, i_k \in [n]$.
\end{assumption}
Assumption~\ref{ass:A4} implies that during an interval of $T$ iterations, every user has a non-negligible positive probability to be updated. 
Note that if the user $i_k$ is active, then it uses recent value with no delay, i.e., $d^k_{i_k} = 0$ as in Assumption~\ref{ass:A4}.
Moreover, the bounded delay assumption $d_i^k \leq \tau$ is standard to analyze convergence of asynchronous algorithms, see e.g.,  \cite{cannelli2019asynchronous,nguyen2018sgd,peng2016arock,Recht2011,xie2019asynchronous}.

Suppose that we choose $0 < \alpha < \bar{\alpha}$ and $0 < \eta < \bar{\eta}$ in Algorithm~\ref{alg:A2}, where  $c := \frac{2\tau^2 - n }{n^2}$ is given, and $\bar{\alpha} > 0$ and $\bar{\eta} > 0$ are respectively computed by
\vspace{-0.5ex}
\begin{equation}\label{eq:choice_of_params0}
\hspace{-0.1ex}
\begin{array}{l}
\bar{\alpha} := \begin{cases}
1 &\text{if}~2\tau^2 \leq n, \\ 
\frac{2}{2 + c} &\text{otherwise},
\end{cases}
 \quad \text{and} \quad
 \bar{\eta} := \begin{cases}
\frac{\sqrt{16 - 8\alpha - 7\alpha^2} - \alpha}{2L(2+\alpha)} &\text{if}~2\tau^2 \leq n, \vspace{0.5ex}\\ 
\tfrac{\sqrt{16 - 8\alpha - (7 + 4c + 4c^2)\alpha^2} - \alpha}{2L[2 + (1+c)\alpha]} &\text{otherwise}.
\end{cases}
\end{array}
\hspace{-3.5ex}
\vspace{-0.5ex}
\end{equation}
Next, we introduce the following  two constants:
\begin{equation}\label{eq:rho_theta}
\arraycolsep=0.2em
\begin{array}{lcl}
\rho & := & \begin{cases}
\frac{2(1-\alpha) - (2+\alpha)L^2\eta^2 - L\alpha\eta}{\alpha\eta n} &\text{if} \quad 2\tau^2 \leq n, \vspace{1ex} \\
\frac{n^2[2(1-\alpha) - (2+\alpha)L^2\eta^2 - L\alpha\eta] - \alpha(1+\eta^2L^2)(2\tau^2 - n)}{\alpha\eta n^3} &\text{otherwise}.
\end{cases}
\vspace{1ex}\\
D & := & \frac{8\alpha^2(1 + L^2\eta^2)(\tau^2 + 2Tn\hat{\mbf{p}}) \ + \ 8n^2(1 + L^2\eta^2 + T\alpha^2\hat{\mbf{p}})}{\hat{\mbf{p}}\alpha^2n^2}. 
\end{array}
\end{equation}
Then, both $\rho$ and $D$ are positive.
We emphasize that though these formulas look complicated, they are computed explicitly without any tuning.
Theorem~\ref{thm:asdr_convergence} proves the convergence of Algorithm~\ref{alg:A2}, whose analysis is in  Supp. Doc.~\ref{apdx:sec:asynFedDR}.

\begin{theorem}\label{thm:asdr_convergence}
Suppose that Assumption~\ref{ass:A1}, \ref{ass:A2}, and \ref{ass:A4} hold for \eqref{eq:fed_prob}.
Let $\bar{\alpha}$, $\bar{\eta}$, $\rho$, and $D$ be given by \eqref{eq:choice_of_params0} and \eqref{eq:rho_theta}, respectively.
Let $\sets{ (x^{k}_i, y^{k}_i, \bar{x}^k)}$ be generated by Algorithm~\ref{alg:A2} with stepsizes $\alpha \in (0, \bar{\alpha})$ and $\eta \in (0, \bar{\eta})$.
Then, the following bound holds:
\begin{equation}\label{eq:asdr_thm_key_est}
\frac{1}{K+1}\sum_{k=0}^K\mathbb{E}\big[ \norms{\Gc_{\eta}(\bar{x}^k)}^2 \big] \leq \frac{\hat{C}\big[F(x^0) - F^{\star}\big]}{K+1},
\end{equation}
where $\hat{C} :=  \ \frac{2(1+\eta L)^2D}{n\eta^2\rho} > 0$ depending on $n, L, \eta, \alpha, \tau, T,$ and $\hat{\mbf{p}}$.

Let $\tilde{x}_K$ be selected uniformly at random from $\sets{\bar{x}^0, \cdots, \bar{x}^K}$ as the output of Algorithm~\ref{alg:A2}.
Then, after at most $K := \BigO{\varepsilon^{-2}}$ iterations, $\tilde{x}^K$ is an $\varepsilon$-stationary point of \eqref{eq:fed_prob} as in Definition~\ref{de:eps_sol}.
\end{theorem}

\begin{remark}
From Theorem~\ref{thm:asdr_convergence}, we can see that \textbf{asyncFedDR} achieves the same worst-case communication complexity $\BigO{\varepsilon^{-2}}$ (up to a constant factor) as \textbf{FedDR}, but with smaller $\alpha$ and $\eta$.
\end{remark}

%% file: experiment.tex
\section{Numerical Experiments}\label{sec:num_exp}
To evaluate the performance of \textbf{FedDR} and \textbf{asyncFedDR}, we conduct multiple experiments using both synthetic and real datasets. 
Since most existing methods are developed for non-composite problems, we also implement three other methods: \textbf{FedAvg}, \textbf{FedProx}, and \textbf{FedPD} to compare for this setting. 
We use training loss, training accuracy, and test accuracy as our performance metrics.  

\noindent\textbf{Implementation.} 
To compare synchronous algorithms, we reuse the implementation of FedAvg and FedProx in \cite{Li_MLSYS2020} and implement FedDR and FedPD on top of it. To conduct the asynchronous examples, we implement our algorithms based on the asynchronous framework in \cite{distBelief}. 
All experiments are run on a Linux-based server with multiple nodes and configuration: 24-core 2.50GHz Intel processors, 30M cache, and 256GB RAM.

\noindent\textbf{Models and hyper-parameters selection.} 
Our models are neural networks, and their detail is given in Supp. Doc.~\ref{app:add_num_exp}.
As in \cite{Li_MLSYS2020}, we use the same local solver (SGD) for all algorithms and run the local updates for $20$ epochs.
Parameters for each algorithm such as $\mu$ for FedProx, $\eta$ for FedPD, and $\alpha$ and $\eta$ for FedDR are tuned from a wide range of values. 
For each dataset, we pick the parameters that work best for each algorithm and plot their performance on the chosen parameters. 

\noindent\textbf{Results on synthetic datasets.}
We compare these algorithms using synthetic dataset in both iid and non-iid settings. 
We follow the data generation procedures described in \cite{Li_MLSYS2020,shamir2014communication} to generate one iid dataset \texttt{synthetic-iid} and three non-iid datasets: \texttt{synthetic-($r$,$s$)} for $(r, s) = \{(0,0),(0.5,0.5),(1,1) \}$. We first compare these algorithms without using the user sampling scheme, i.e. all users perform update at each communication round, and for non-composite model of \eqref{eq:fed_prob}.

\begin{figure*}[ht!]
\begin{center}
\includegraphics[width = 1\textwidth]{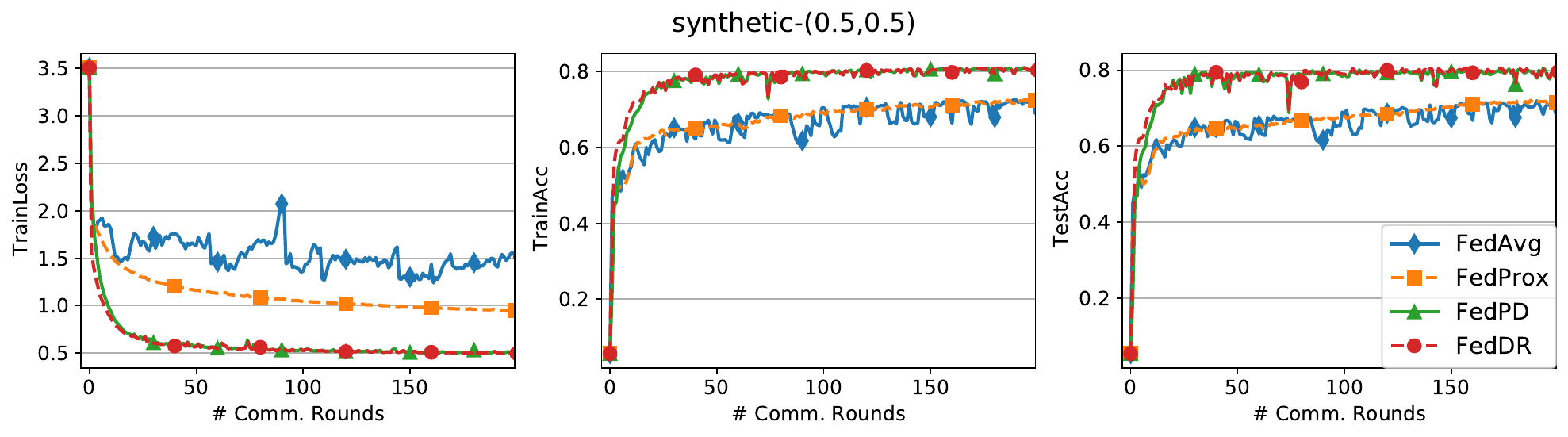}\vspace{-1ex}
    \vspace{-1ex}
    \caption{The performance of 4 algorithms on non-iid synthetic datasets without user sampling}\label{fig:exp_synth_non_iid}
\end{center}
\vspace{-2ex}
\end{figure*}

We report the performance of these algorithms on one non-iid dataset in Figure~\ref{fig:exp_synth_non_iid}, but more results can be found in Sup. Doc.~\ref{app:add_num_exp}. FedDR and FedPD are comparable in these datasets and they both outperform FedProx and FedAvg. FedProx works better than FedAvg which aligns with the results in \cite{Li_MLSYS2020}. 
However, when comparing on more datasets, our algorithm overall performs better than others.

Now we compare these algorithms where we sample 10 users out of 30 to perform update at each communication round for FedAvg, FedProx, and FedDR while we use all users for FedPD since FedPD only has convergence guarantee for this setting. In this test, the evaluation metric is plotted in terms of the number of bytes communicated between users and server at each communication round. Note that using user sampling scheme in this case can save one-third of communication cost each round. Figure~\ref{fig:exp_synth_stoc} depicts the performance of 4 algorithms on one dataset, see also Sup. Doc.~\ref{app:add_num_exp}.

\begin{figure*}[ht!]
\vspace{-1ex}
\begin{center}
\includegraphics[width = 1\textwidth]{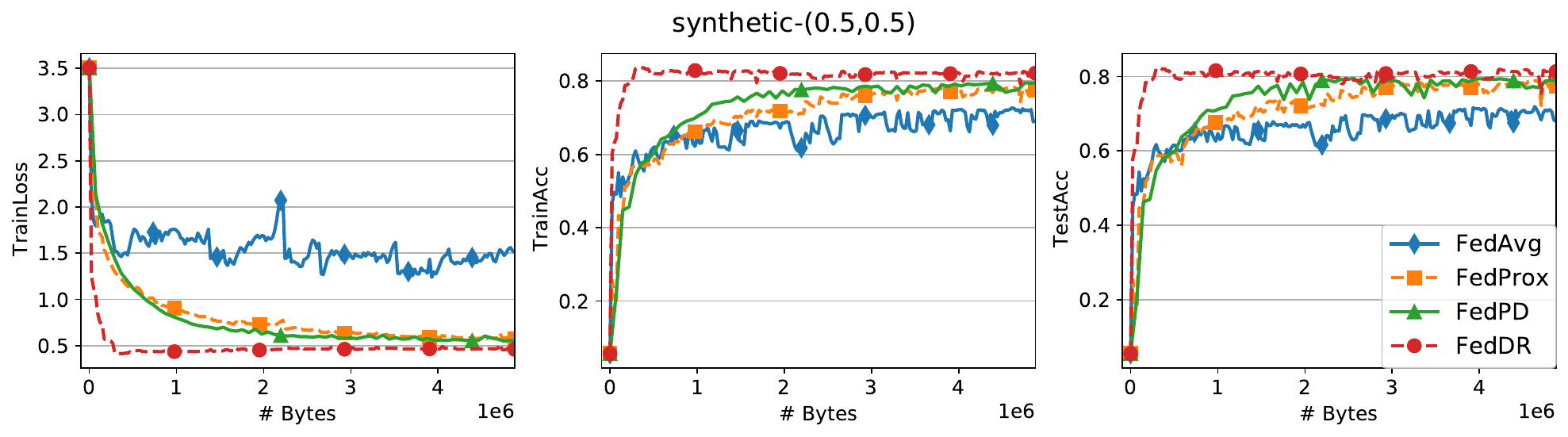}\vspace{-1ex}
\vspace{-0.5ex}
    \caption{The performance of 4 algorithms with user sampling scheme on non-iid synthetic datasets.}\label{fig:exp_synth_stoc}
\end{center}
\vspace{-1.5ex}
\end{figure*}

From Figure~\ref{fig:exp_synth_stoc}, FedDR performs well compared to others. FedProx using user sampling scheme performs better and is slightly behind FedPD while FedDR, FedPD, and FedProx outperform FedAvg. 

\noindent\textbf{Results on FEMNIST datasets.} 
FEMNIST \cite{caldas2018leaf} is an extended version of the MNIST dataset \cite{lecun1998gradient} where the data is partitioned by the writer of the digit/character. It has a total of 62 classes (10 digits, 26 upper-case and 26 lower-case letters) with over 800,000 samples.  In this example, there are total of 200 users and we sample 50 users to perform update at each round of communication for FedAvg, FedProx, and FedDR while we use all users to perform update for FedPD. 
Fig.~\ref{fig:exp_femnist} depicts the performance of $4$ algorithms in terms of communication cost. 
From Fig.~\ref{fig:exp_femnist}, FedDR can achieve lower loss value and higher training accuracy than other algorithms while FedPD can reach the same test accuracy as ours at the end. 
Overall, FedDR seems working better than other algorithms in this test.

\begin{figure*}[ht!]
\vspace{-1ex}
\begin{center}
    \includegraphics[width = .94\textwidth]{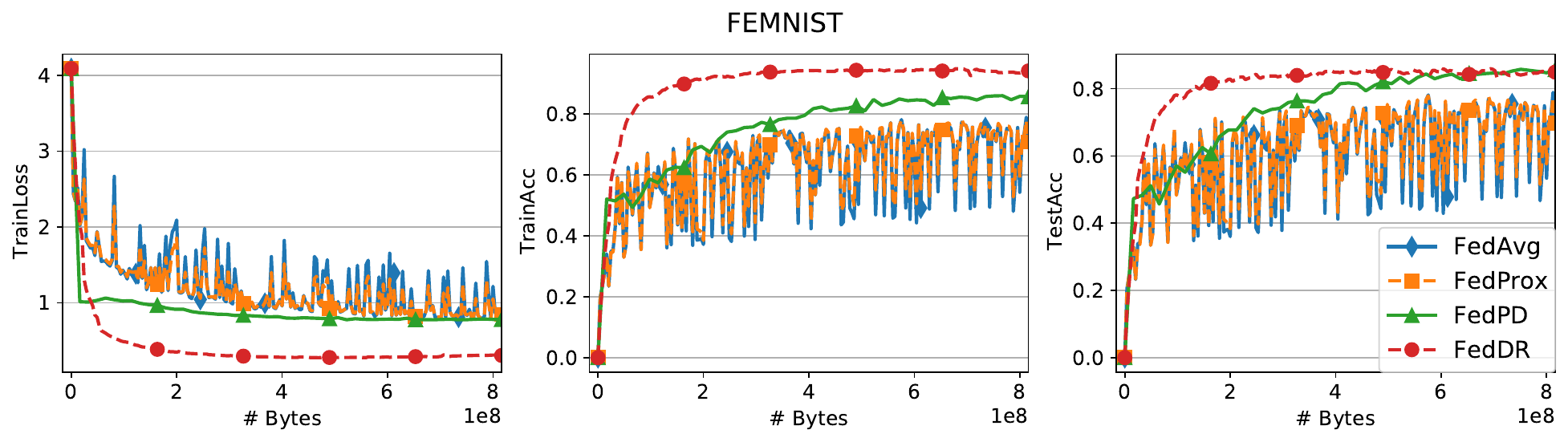}
    \vspace{-1.2ex}
    \caption{The performance of 4 algorithms on the FEMNIST dataset.}\label{fig:exp_femnist}
\end{center}
\vspace{-2.5ex}
\end{figure*}

\noindent\textbf{Results with the $\ell_1$-norm regularizer.}
We now consider the composite setting with $g(x) := 0.01 \norm{x}_1$ to verify Algorithm~\ref{alg:A1} on different inexactness levels $\epsilon_{i,k}$ by varying the learning rate (\textit{lr}) and the number of local SGD epochs to approximately evaluate $\prox_{\eta f_i}(y^k_i)$.
We run Algorithm~\ref{alg:A1} on the FEMNIST dataset, and the results are shown in Figure~ \ref{fig:exp_comp_femnist}.

\begin{figure*}[ht!]
\begin{center}
    \includegraphics[width = .94\textwidth]{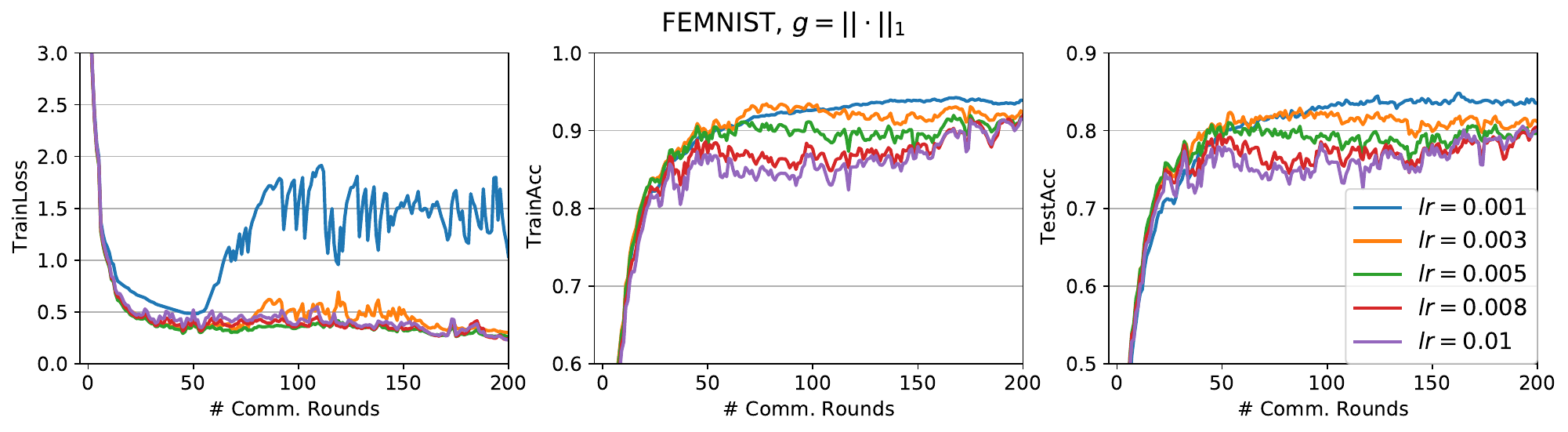}
    \includegraphics[width = .94\textwidth]{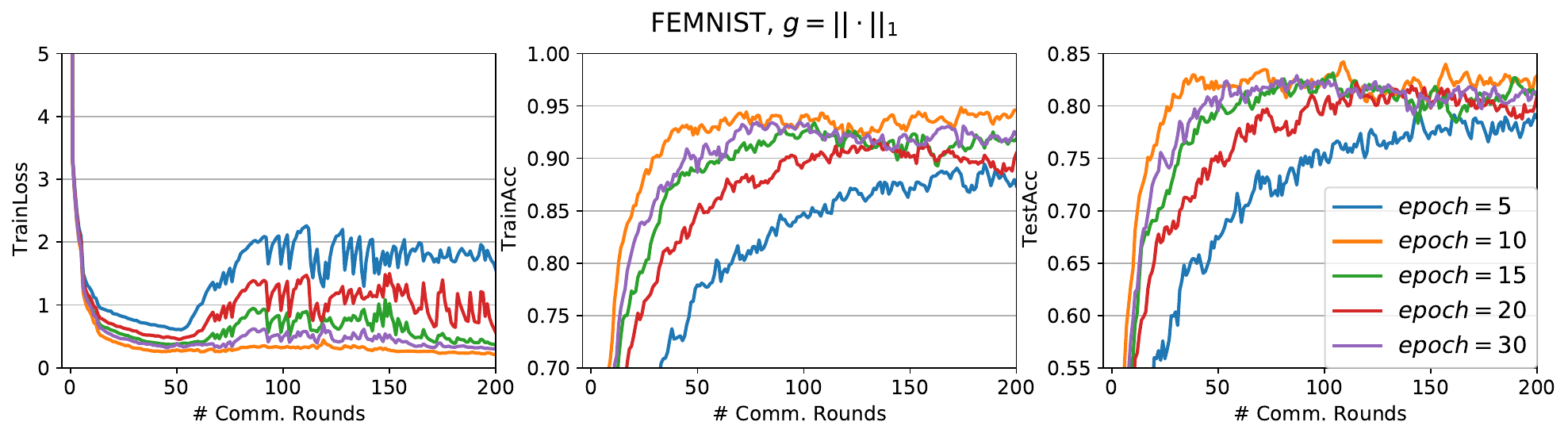}
    \vspace{-1ex}
    \caption{The performance of \textbf{FedDR} on \texttt{FEMNIST} dataset in composite setting.}\label{fig:exp_comp_femnist}
\end{center}
\vspace{-2.5ex}
\end{figure*}
We observe that Algorithm~\ref{alg:A1} works best when local learning rate is $0.003$ which aligns with \cite{Li_MLSYS2020} for the non-composite case.
It also performs better when we decrease $\epsilon_{i,k}$ by increasing the number of epochs in evaluating $\prox_{\eta f_i}$.
This performance confirms our theoretical results in Supp. Doc. \ref{subsec:inexact_alg1}.

\noindent\textbf{Results using asynchronous update.} 
To illustrate the advantage of asyncFedDR over FedDR, we conduct another example to train MNIST dataset using 20 users. Since we run these experiments on computing nodes with identical configurations, we simulate the case with computing power discrepancy between users by adding variable delay to each user's update process such that the difference between the fastest user may be up to twice as fast as the slowest one.

\begin{figure*}[hpt!]
\vspace{-1ex}
\begin{center}
    \includegraphics[width = .94\textwidth]{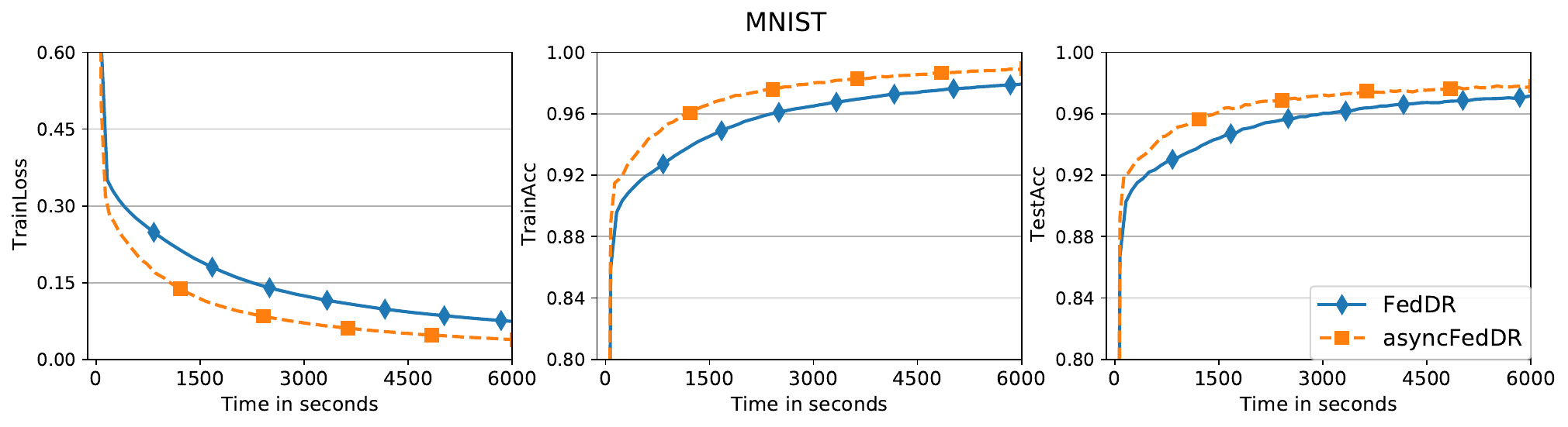}
    \vspace{-1ex}
    \caption{The performance of \textbf{FedDR} and \textbf{asyncFedDR} on the MNIST dataset.}\label{fig:exp_mnist}
\end{center}
\vspace{-2ex}
\end{figure*}

The results of two variants are presented in Figure~\ref{fig:exp_mnist}, see Supp. Doc.~\ref{app:add_num_exp} for more examples. 
We can see that asyncFedDR can achieve better performance than FedDR in terms of training time which illustrate the advantage of asynchronous update in heterogeneous computing power.

%% file: appendix.tex
\begin{center}
\textbf{\large Supplementary Documment} \vspace{0.5ex}\\
\textbf{\Large 
FedDR -- Randomized Douglas-Rachford Splitting Algorithms \vspace{0.5ex}\\ for Nonconvex Federated Composite Optimization
}
\end{center}
\vspace{1ex}

\section{The Analysis of Algorithm~\ref{alg:A1}: Randomized Coordinate Variant --- FedDR}\label{sec:apdx_FedDR}
In this Supplementary Document (Supp. Doc.), we first provide additional details in the derivation of Algorithm~\ref{alg:A1}, \textbf{FedDR}.
Then, we present the full proofs of the convergence results of Algorithm~\ref{alg:A1}.

\subsection{Derivation of Algorithm~\ref{alg:A1}}\label{subsec:FedDR_derivation}
Our first step is to recast \eqref{eq:fed_prob} into a constrained reformulation.
Next, we apply the classical Douglas-Rachford (DR) splitting scheme to this reformulation. 
Finally, we randomize its updates to obtain a randomized block-coordinate DR variant.

\textbf{(a)~Constrained reformulation.}
With a little abuse of notation, we can equivalently write \eqref{eq:fed_prob} into the following constrained minimization problem:
\begin{equation}\label{eq:constr_reform}
\left\{\begin{array}{ll}
\displaystyle\min_{x_1,\cdots,x_n}~&\Big\{ F(\mbf{x}) := f(\mbf{x}) + g(\mbf{x}) \equiv \displaystyle \frac{1}{n}\sum_{i=1}^n f_i(x_i) + g(x_1) \Big\} \\
\text{s.t.}~&x_2 = x_1, \ x_3 = x_1, \ \cdots, x_n = x_1.
\end{array}\right.
\end{equation}
where $\mbf{x} := [x_1, x_2, \cdots, x_n]$ concatenates $n$ duplicated variables $x_1, x_2, \cdots, x_n$ of $x$ in \eqref{eq:fed_prob} such that  it forms a column vector in $\R^{np}$.
Such duplications are characterized by $x_2 = x_1, x_3 = x_1, \cdots, x_n = x_1$, which define a linear subspace $\Lc := \sets{ \mbf{x} \in\R^{np} : x_2 = x_1, \ x_3 = x_1, \cdots,  x_n = x_1}$ in $\R^{np}$.

\textbf{(b)~Unconstrained reformulation.}
Let $\delta_{\Lc}$ be the indicator function of $\Lc$, i.e. $\delta_{\Lc}(\mbf{x}) = 0$ if $\mbf{x}\in\Lc$, and $\delta_{\Lc}(\mbf{x}) = +\infty$, otherwise.
Then, we can rewrite \eqref{eq:constr_reform} into the following unconstrained setting:
\begin{equation}\label{eq:fed_comp_prob}
\min_{\mbf{x}\in\R^{np}}\Big\{ F(\mbf{x}) := f(\mbf{x}) + g(\mbf{x}) + \delta_{\Lc}(\mbf{x}) \equiv \frac{1}{n}\sum_{i=1}^n f_i(x_i) + g(x_1) + \delta_{\Lc}(\mbf{x}) \Big\}.
\end{equation}
Clearly, \eqref{eq:fed_comp_prob} can be viewed as a composite nonconvex minimization problem of $f(\mbf{x})$ and $g(\mbf{x}) + \delta_{\Lc}(\mbf{x})$.
The first-order optimality condition of \eqref{eq:fed_comp_prob} can be written as
\begin{equation}\label{eq:opt_cond2}
0 \in \nabla{f}(\mbf{x}^{\star}) + \partial{g}(\mbf{x}^{\star}) +  \partial{\delta_{\Lc}}(\mbf{x}^{\star}),
\end{equation}
where $\partial{\delta_{\Lc}}$ is the subdifferential of $\delta_{\Lc}$, which is the normal cone of $\Lc$ (or, equivalently, $\partial{\delta_{\Lc}}(\mbf{x}) = \Lc^{\perp}$ if $\mbf{x}\in\Lc$, the orthogonal subspace of $\Lc$,  and $\partial{\delta_{\Lc}}(\mbf{x}) = \emptyset$, otherwise), and $\partial{g}$ is the subdifferential of $g$.
Note that since $f$ is nonconvex, \eqref{eq:opt_cond2} only provides a necessary condition for $\mbf{x}^{\star} := [x^{\star}_1, \cdots, x^{\star}_{n}]$ to be a local minimizer.
Any $\mbf{x}^{\star}$ satisfying \eqref{eq:opt_cond2} is called a (first-order) stationary point of \eqref{eq:fed_comp_prob}.
In this case, we have $x^{\star}_i = x_{1}^{\star}$ for all $i \in [n]$.
Hence, using \eqref{eq:opt_cond2}, we have $0 \in \nabla{f}(\mbf{x}^{\star}) + \partial{g}(\mbf{x}^{\star}) + \Lc^{\perp}$.
This condition is equivalent to $0 \in \frac{1}{n}\sum_{i=1}^n\nabla{f_i}(x_i^{\star}) + \partial{g}(x_1^{\star})$.
However, since $x^{\star}_i = x_1^{\star}$ for all $i \in [n]$, the last inclusion becomes $0 \in \frac{1}{n}\sum_{i=1}^n\nabla{f_i}(x_1^{\star}) + \partial{g}(x_1^{\star})$.
Equivalently, we have $x^{\star} := x^{\star}_1$ to be a stationary point of \eqref{eq:fed_prob}.

\textbf{(c)~Full parallel DR variant.}
Let us apply the DR splitting method to \eqref{eq:opt_cond2}, which can be written explicitly as follows:
\begin{equation}\label{eq:DR_full}
\arraycolsep=0.3em
\left\{\begin{array}{lcl}
\mbf{y}^{k+1} &:= & \mbf{x}^k + \alpha(\bar{\mbf{x}}^k - \mbf{x}^k), \vspace{1ex}\\
\mbf{x}^{k+1} &:= & \prox_{n \eta f}(\mbf{y}^{k+1}), \vspace{1ex}\\
\bar{\mbf{x}}^{k+1} &:= & \prox_{n\eta(g + \delta_{\Lc})}(2\mbf{x}^{k+1} - \mbf{y}^{k+1}),
\end{array}\right.
\end{equation}
where $\eta > 0$ is a given such that $n\eta $ is a step-size and $\alpha \in (0, 2]$ is a relaxation parameter \cite{themelis2020douglas}.
If $\alpha = 1$, then we recover the classical Douglas-Rachford scheme \cite{Lions1979} and if $\alpha = 2$, then we recover the Peaceman-Rachford splitting scheme \cite{Bauschke2011}.
Note that the classical DR scheme studied in \cite{Lions1979} was developed to solve monotone inclusions, and in our context, convex problems.
Recently, it has been extended to solve nonconvex optimization problems, see, e.g., \cite{li2015global,themelis2020douglas}.

Let us further exploit the structure of $f$, $g$, and $\delta_{\Lc}$ in \eqref{eq:fed_comp_prob} to obtain a special parallel DR variant.
\begin{compactitem}
\item First, since $f(\mbf{x}) = \frac{1}{n}\sum_{i=1}^nf_i(x_i)$, we have 
\begin{equation*}
\arraycolsep=0.3em
\begin{array}{lcl}
\displaystyle\min_{\mbf{x}}\Big\{ f(\mbf{x}) + \tfrac{1}{2n\eta}\norms{\mbf{x} - \mbf{y}^{k+1}}^2 \Big\} &= & \displaystyle\min_{\mbf{x}}\Big\{ \tfrac{1}{n}\sum_{i=1}^n\Big[ f_i(x_i) + \tfrac{1}{2\eta}\norms{x_i - y_i^{k+1}}^2 \Big] \Big\} \vspace{1ex}\\
&= & \frac{1}{n} \displaystyle\sum_{i=1}^n \displaystyle\min_{x_i} \Big\{ f_i(x_i) + \tfrac{1}{2\eta}\norms{x_i - y_i^{k+1}}^2 \Big\}.
\end{array}
\end{equation*}
Hence, we can decompose the computation of  $\mbf{x}^{k+1} :=  \prox_{n\eta f}(\mbf{y}^{k+1})$ from \eqref{eq:DR_full} as $x_i^{k+1} := \prox_{\eta f_i}(y^{k+1}_i)$ for all $i \in [n]$.
\item Next, we denote $\hat{\mbf{x}}^{k+1} := 2\mbf{x}^{k+1} - \mbf{y}^{k+1}$, or equivalently, in component-wise $\hat{x}^{k+1}_i := 2x^{k+1}_i - y_i^{k+1}$ for all $i \in [n]$.
\item Finally, the third line of \eqref{eq:DR_full} $\bar{\mbf{x}}^{k+1} := \prox_{n\eta(g + \delta_{\Lc})}(\hat{\mbf{x}}^{k+1})$ can be rewritten as
\begin{equation}\label{eq:prox_gL}
\hspace{-0.0ex}
\arraycolsep=0.2em
\bar{\mbf{x}}^{k+1} := \prox_{n\eta(g + \delta_{\Lc})}(\hat{\mbf{x}}^{k+1}) = \left\{\begin{array}{ll}
&{\mathrm{arg}\!\!\!\!\!\!\displaystyle\min_{[x_1,\cdots, x_n]}}\big\{ g(x_1) + \tfrac{1}{2n\eta}\sum_{i=1}^{n}\norms{x_i - \hat{x}_i^{k+1}}^2 \big\} \vspace{1ex}\\
&\text{s.t.} \quad x_i = x_1, \ \ \text{for all} \ i = 2, \cdots, n.
\end{array}\right.
\hspace{-4ex}
\end{equation}
\end{compactitem}
Let us solve\eqref{eq:prox_gL} explicitly.
First, we define a Lagrange function associated with \eqref{eq:prox_gL} as 
\begin{equation*}
\Lc(\mbf{x}, \mbf{z}) = g(x_1) + \frac{1}{2n\eta}\sum_{i=1}^{n}\norms{x_i - \hat{x}_i^{k+1}}^2 + \sum_{i=1}^{n-1}z_i^{\top}(x_{i+1} - x_1), 
\end{equation*}
where $z_i$ ($i = 1, \cdots, n-1)$ are the corresponding Lagrange multipliers.
Hence, the KKT condition of \eqref{eq:prox_gL} can be written as
\begin{equation*}
\arraycolsep=0.2em
\left\{\begin{array}{ll}
&\partial{g}(\bar{x}^{k+1}_1) + \tfrac{1}{n\eta}(\bar{x}^{k+1}_1 - \hat{x}_1^{k+1}) - \sum_{i=1}^{n-1}z_i  = 0, \vspace{1ex}\\
& \tfrac{1}{n\eta}(\bar{x}^{k+1}_{i+1} - \hat{x}_{i+1}^{k+1}) + z_i = 0, \quad \text{for all} \ i  = 1, \cdots, n-1, \vspace{1ex}\\
&\bar{x}^{k+1}_{i+1} = \bar{x}^{k+1}_1, \quad  \text{for all} \ i = 1, \cdots, n-1.
\end{array}\right.
\end{equation*}
Summing up the second line from $i=1$ to $i=n-1$ and combining the result with the last line of this KKT condition, we have 
\begin{equation*}
n\eta \sum_{i=1}^{n-1}z_i = \sum_{i=1}^{n-1}(\hat{x}^{k+1}_{i+1} - \bar{x}_{i+1}^{k+1}) = \sum_{i=2}^{n}\hat{x}^{k+1}_{i} - (n-1)\bar{x}^{k+1}_1.
\end{equation*}
Substituting this expression into the first line of the KKT condition, we get 
\begin{equation}\label{eq:kkt_cond2}
\sum_{i=1}^n\hat{x}_i^{k+1} - (n-1)\bar{x}_1^{k+1}  = \hat{x}_1^{k+1} + n\eta\sum_{i=1}^{n-1}z_i  \in \bar{x}^{k+1}_1 + n\eta\partial{g}(\bar{x}^{k+1}_1).
\end{equation}
This condition is equivalent to $\sum_{i=1}^{n}\hat{x}^{k+1}_i \in  n\bar{x}^{k+1}_1 + n\eta \partial{g}(\bar{x}^{k+1}_1)$.
By introducing a new notation $\bar{x}^{k+1} := \bar{x}^{k+1}_1$, we eventually obtain from the last inclusion that
\begin{equation*}
\begin{array}{l}
\bar{\mbf{x}}^{k+1} := [\bar{x}^{k+1}, \cdots, \bar{x}^{k+1}] \in \R^{np}, \quad \text{where} \quad \bar{x}^{k+1} := \prox_{\eta g}\left( \frac{1}{n}\sum_{i=1}^{n}\hat{x}_i^{k+1} \right).
\end{array}
\end{equation*}
If we introduce a new variable $\tilde{x}^{k+1} := \frac{1}{n}\sum_{i=1}^n \hat{x}_i^{k+1}$, then $\bar{x}^{k+1} := \prox_{\eta g}\big( \tilde{x}^{k+1}  \big)$.

Putting the above steps together, we obtain the following parallel DR variant for solving \eqref{eq:fed_prob}:
\begin{equation}\label{eq:ncvx_dr_scheme}
\arraycolsep=0.3em
\left\{\begin{array}{lcl}
y^{k+1}_i &:= & y_i^k + \alpha(\bar{x}^{k} - x^{k}_i), \quad \forall i\in [n] \vspace{1ex}\\
x_i^{k+1} &:= & \mathrm{prox}_{\eta f_i}(y_i^{k+1}), \quad \forall i\in [n] \vspace{1ex}\\
\hat{x}_i^{k+1} &:= & 2x^{k+1}_i - y^{k+1}_i, \quad \forall i\in [n] \vspace{1ex}\\
\tilde{x}^{k+1} &:= & \frac{1}{n}\sum_{i=1}^n\hat{x}_i^{k+1}, \vspace{1ex}\\
\bar{x}^{k+1} &:= & \prox_{\eta g}\big( \tilde{x}^{k+1} \big).
\end{array}\right.
\end{equation}
This variant can be implemented in parallel. 
It is also known as a special variant of Tseng's splitting method \cite{Bauschke2011} in the convex case.
This variant also covers \texttt{FedSplit} in \cite{pathak2020fedsplit} for FL as a special case when $g = 0$, $f_i$ is convex for all $i\in [n]$, and $\alpha = 2$.
In fact, \texttt{FedSplit} is a variant of the Peaceman-Rachford method, and is different from our algorithms due to $\alpha < 2$.
If $g = 0$ (i.e., without regularizer), then the last line of \eqref{eq:ncvx_dr_scheme} reduces to $\bar{x}^{k+1} =  \tilde{x}^{k+1}$.

\textbf{(d)~Inexact block-coordinate DR variant.} 
Instead of performing update for all users $i \in[n]$ as in \eqref{eq:ncvx_dr_scheme}, we propose a new block-coordinate DR variant, called  \texttt{FedDR}, where only a subset of users $\Sc_k \subseteq [n]$ performs local update then send its local model to server for aggregation. 
For user $i \notin \Sc_k$, the local model is unchanged, i.e., for all $i \notin\Sc_k$: $y^{k+1}_i = y^k_i$, $x^{k+1}_i = x^{k}_i$, and $\hat{x}^{k+1}_i = \hat{x}^{k}_i$. 
Hence, no communication with the server is needed for these users.
Furthermore, we assume that we can only approximate the proximal operator $\prox_{\eta f_i}$ up to a given accuracy for all $i \in [n]$.
In this case, we replace the exact proximal step $x_i^{k} := \prox_{\eta f_i}(y_i^{k})$ by its approximation $x_i^{k} :\approx \prox_{\eta f_i}(y_i^{k})$ up to a given accuracy $\epsilon_{i,k}\geq 0$ such that
\begin{equation}\label{eq:aprox_proxf}
\Vert x_i^{k} - \prox_{\eta f_i}(y_i^{k}) \Vert \leq \epsilon_{i,k}.
\end{equation}
Since $x_i^{k}$ is approximately computed from $\prox_{\eta f_i}(y^{k}_i)$ as in \eqref{eq:aprox_proxf}, we have
\begin{equation}\label{eq:exact_prox}
x_i^{k} = z^{k}_i + e_i^{k}, \quad \text{where}  \quad z_i^{k} := \prox_{\eta f_i}(y^{k}_i) \quad  \text{and} \quad \norms{e_i^{k}} \leq \epsilon_{i,k}.
\end{equation}
We will use this representation of $x_i^{k}$ and $x_i^{k+1}$ in our analysis in the sequel.

More specifically, the update of our inexact block-coordinate DR variant can be described as follows.
\begin{compactitem}
\item \textbf{Initialization:} 
Given an initial vector $x^0\in\dom{F}$ and accuracies $\epsilon_{i,0} \geq 0$.
\item[] Initialize the server with $\bar{x}^0 := x^0$.
\item[] Initialize all users $i\in [n]$ with $y^0_i := x^0$, \ $x_i^0 :\approx \prox_{\eta f_i}(y_i^0)$, and $\hat{x}_i^0 := 2x_i^0 - y_i^0$.
\item \textbf{The $k$-th iteration ($k\geq 0$):} Sample a proper subset $\Sc_k\subseteq [n]$ so that $\Sc_k$ presents as the subset of \textit{active users}.
\item (\textit{\textbf{Communication}}) Each user $i\in\Sc_k$ receives $\bar{x}^k$ from the server.
\item (\textit{\textbf{Local/user update}}) For each user $i\in\Sc_k$, given $\epsilon_{i,k+1} \geq 0$, it updates
   \begin{equation*}
	\left\{\begin{array}{lcl}
	y^{k+1}_i &:= & y_i^k + \alpha(\bar{x}^{k} - x^{k}_i) \vspace{1ex}\\
	x_i^{k+1} &:\approx & \prox_{\eta f_i}(y_i^{k+1})\vspace{1ex}\\
	\hat{x}^{k+1}_i &:= &2x^{k+1}_i - y^{k+1}_i.
	\end{array}\right.
	\vspace{-1ex}
    \end{equation*}
	Each user $i\notin\Sc_k$ does nothing, i.e.:
   \begin{equation*}
	\left\{\begin{array}{lcl}
	y^{k+1}_i &:= & y_i^k  \vspace{1ex}\\
	x_i^{k+1} &:= & x^{k}_i\vspace{1ex}\\
	\hat{x}^{k+1}_i &:= &\hat{x}^k_i.
	\end{array}\right.
    \end{equation*}
\item (\textit{\textbf{Communication}}) Each user $i\in\Sc_k$ sends only $\hat{x}^{k+1}_i$ to the server.
\item (\textit{\textbf{Global/Server update}}) The server aggregates $\tilde{x}^{k+1} := \frac{1}{n}\sum_{i=1}^n \hat{x}^{k+1}_i$, and then compute $\bar{x}^{k+1} := \prox_{\eta g}( \tilde{x}^{k+1} )$.
\end{compactitem}

This scheme is exactly Algorithm~\ref{alg:A1}.
However, the global update on $\tilde{x}^{k+1}$ can be simplified as 
\begin{equation*}
\arraycolsep=0.2em
\begin{array}{lcl}
\tilde{x}^{k+1} & := & \frac{1}{n}\sum_{i=1}^n\hat{x}^{k+1}_i =  \frac{1}{n}\sum_{i\in\Sc_k}^n\hat{x}^{k+1}_i +  \frac{1}{n}\sum_{i\not\in\Sc_k}^n\hat{x}^{k}_i \vspace{1ex}\\
& = &  \frac{1}{n}\sum_{i=1}^n\hat{x}^{k}_i +  \frac{1}{n}\sum_{i\in\Sc_k}^n(\hat{x}^{k+1}_i - \hat{x}_i^k) \vspace{1ex}\\
& = & \tilde{x}^k +  \frac{1}{n}\sum_{i\in\Sc_k}\Delta{\hat{x}}^{k}_i.
\end{array}
\end{equation*}
This step is implemented in Algorithm~\ref{alg:A1}.

To analyze convergence of Algorithm~\ref{alg:A1}, we conceptually introduce $z_i^0$ and $z_i^{k+1}$ for $i\in [n]$ as follows:
\begin{equation}\label{eq:z_i^k_var}
z_i^0 := \prox_{\eta f_i}(y_i^0), \quad
z_i^{k+1} := \begin{cases}
\prox_{\eta f_i}(x_i^{k+1}) & \text{if} \quad i \in \Sc_k \\
z_i^k &\text{if} \quad i \notin \Sc_k,
\end{cases}
\quad \text{and} \quad x_i^k := z_i^k + e_i^k.
\end{equation}
Here, $e_i^k$ is the vector of errors.
Note that $z_i^0$ and $z_i^{k+1}$ do not exist in actual implementation of Algorithm~\ref{alg:A1}, and we only have their approximations $x_i^0$ and $x_i^{k+1}$, respectively.
For any $k\geq 0$, since $x_i^{k+1} = x_i^k$ and $z_i^{k+1} = z_i^k$ for $i\notin\Sc_k$, we have $\Vert x_i^{k+1} - z_i^{k+1}\Vert = \Vert e_i^{k+1}\Vert = \Vert x_i^k - z_i^k\Vert = \Vert e_i^k\Vert$ for $i\notin\Sc_k$.
To guarantee $\Vert e_i^{k+1}\Vert = \Vert e_i^k\Vert$ for $i\notin\Sc_k$, we must choose $\epsilon_{i,k+1} := \epsilon_{i,k}$ for $i\notin\Sc_k$.

Note that in Algorithm~\ref{alg:A1}, we have not specified the choice of $\Sc_k$.
The subset $\Sc_k$ is an iid realization of a random set-valued mapping $\hat{\Sc}$ from $[n]$ to $2^{[n]}$, the collection of all subsets of $[n]$.
Moreover, $\hat{\Sc}$ is a proper sampling scheme in the sense that $\mathbf{p}_i := \mathbb{P}(i \in \hat{S}) > 0$ for all $i\in [n]$ as stated in Assumption~\ref{ass:A3}.
By specifying this probability distribution $\mathbf{p} := (\mathbf{p}_1, \cdots, \mathbf{p}_n)$, we obtain different sampling strategies ranging from uniform to non-uniform as discussed in \cite{richtarik2016parallel}.
Our analysis below holds for arbitrary sampling scheme that satisfies Assumption~\ref{ass:A3}.

\subsection{Further details of comparison}\label{susec:A3_comparison}
We have compared our methods, Algorithm~\ref{alg:A1} and Algorithm~\ref{alg:A2}, with various existing FL methods in the introduction (Section~\ref{sec:intro}).
Here, let us further elaborate this comparison in more detail.
Due to the rapid development of FL in the last few years, it is impossible to review a majority of works in this field. 
Hence, we only select a few algorithms that we find most related to our work in this paper.
\begin{compactitem}
\item \textbf{FedAvg:} FedAvg \cite{mcmahan2017communication} has become a de facto standard federated learning algorithm in practice.
However, it has several limitations as discussed in many papers, including  \cite{Li_MLSYS2020}. 
It is also difficult to analyze convergence of FedAvg, especially in the nonconvex case and heterogeneity settings (both statistical and system heterogeneity).  
Moreover, FedAvg originally specifies SGD with a fixed number of epochs and a fixed learning rate as its local solver, making it less flexible in practice.
Convergence analysis of FedAvg requires additional assumptions apart from the standard smoothness of $f_i$. 
Moreover, its extension to the composite setting, e.g., in \cite{yuan2020federated} only focuses on the convex case, and requires a set of strong assumptions, including bounded heterogeneity.
Since it was proposed, several attempts have been made to analyze convergence of FedAvg in both convex and nonconvex settings, see, e.g., \cite{gorbunov2021local,haddadpour2019local,li2019convergence,lin2018don,woodworth2020local}.

\item \textbf{FedProx:} FedProx proposed in \cite{Li_MLSYS2020}, on the one hand, can be viewed as an extension of FedAvg, but on the other hand, can be cast into a quadratic penalty-type method for the constrained reformulation \eqref{eq:constr_reform} of \eqref{eq:fed_prob}. Indeed, when $g = 0$, from \eqref{eq:constr_reform}, we can define a quadratic penalty function with a penalty parameter $\mu > 0$ as follows:
\begin{equation*}
P_{\mu}(\mbf{x}) := \frac{1}{n}\sum_{i=1}^n f_i(x_i) + \frac{\mu}{2n}\sum_{i=2}^n\norms{x_i - x_{n+1}}^2.
\end{equation*}
First, we apply an alternating minimization strategy to minimize $P_{\mu}$ over $[x_1, \cdots, x_n]$ and then over $x_{n+1}$.
Next, instead using the full minimization over all blocks $x_1, \cdots, x_n$, a block coordinate descent strategy is applied by selecting a subset of blocks $\Sc_k \subseteq [n]$ at random.
Finally, we replace the exact minimization problem of each block $x_i$ by its inexact computation.
This method exactly leads to FedProx in  \cite{Li_MLSYS2020}.
While FedProx can potentially handle a major heterogeneity challenge, it relies on a [local] dissimilarity assumption, which could be difficult to check.
In addition, this assumption limits the application of FedProx.

\item \textbf{Other methods:} \textbf{FedPD} proposed in \cite{zhang2020fedpd} is exactly an augmented Lagrangian method applying to the constrained reformulation  \eqref{eq:constr_reform} of \eqref{eq:fed_prob} when $g = 0$, combining with an alternating minimization strategy as in FedProx.
However, FedPD requires all users to update their computation and flips a biased coin to decide if a global communication is carried out.
This method essentially violates one crucial requirement of FL, which is known as system heterogeneity. 
Another FL method is \textbf{FedSplit} in \cite{pathak2020fedsplit}, which also requires all users to participate into each communication round.
This method also relies on Peaceman-Rachford splitting scheme \cite{Bauschke2011} and is different from our algorithms.
Its convergence analysis is only shown for convex problems in  \cite{pathak2020fedsplit}.
However, as shown in  \cite{pathak2020fedsplit}, this scheme can overcome the fundamental statistical heterogeneity challenge in FL.

\end{compactitem}
In contrast to the above methods, our methods developed in this paper always converges under standard assumptions (i.e., only the $L$-smoothness and boundedness from below).
The proposed methods can handle the majority of challenges in FL, including system and statistical heterogeneity. 
We also allow one to use any local solver to evaluate $\prox_{\eta f_i}$ up to a given adaptive accuracy.
Moreover, our methods can handle convex regularizers (in particular, convex constraints), and can be implemented in an asynchronous manner.

\subsection{Preparatory lemmas}\label{subsec:A2_preparation}
We first present a useful lemma to characterize the relationship between $x^k_i$ and  $y^k_i$ for all iteration $k$.
Then, we prove a sure descent lemma to establish the main results in the main text.

\begin{lemma}\label{lem:yk_xk}
Let $\{(y^k_i,x^k_i, z_i^k) \}$ be generated by Algorithm~\ref{alg:A1} and \eqref{eq:z_i^k_var} starting from $z^0_i := \prox_{\eta f_i}(y^0_i)$ for all $i\in [n]$ as in \eqref{eq:z_i^k_var}.
Then, for all $i \in [n]$ and $k \ge 0$, we have
\begin{equation}\label{eq:yk_xk_relation}
y^{k}_i = z^k_i + \eta \nabla f_i(z^k_i), \quad \text{and} \quad \hat{x}_i^k = 2x_i^k - y_i^k.
\end{equation}
\end{lemma}

\begin{proof}
We prove \eqref{eq:yk_xk_relation} by induction. 
For $k = 0$, due to the initialization step, Step~\ref{step:i0} of Algorithm~\ref{alg:A1} and \eqref{eq:z_i^k_var} with $z_i^0 := \prox_{\eta f_i}(y_i^0)$, we have $y^0_i = z^0_i + \eta \nabla f_i(z^0_i)$ and $\hat{x}_i^0 = 2x_i^0 - y_i^0$ as in \eqref{eq:yk_xk_relation}.

Suppose that \eqref{eq:yk_xk_relation} holds for all $k\geq 0$, i.e., $y^{k}_i = z^k_i + \eta \nabla f_i(z^k_i)$ and $\hat{x}_i^k = 2x_i^k - y_i^k$.
We will show that \eqref{eq:yk_xk_relation} holds for $k+1$, i.e. $y^{k+1}_i = z^{k+1}_i + \eta \nabla f_i(z^{k+1}_i)$ and $\hat{x}^{k+1}_i = 2x_i^{k+1} - y_i^{k+1}$ for all $i\in[n]$, respectively. 
We have two cases:
\begin{compactitem}
\item For any user $i \in\Sc_k$, from the optimality condition of \eqref{eq:z_i^k_var}, we have
\begin{equation*}
\nabla f_i(z^{k+1}_i) + \tfrac{1}{\eta}(z^{k+1}_i - y^{k+1}_i) = 0 \quad \Rightarrow \quad y^{k+1}_i = z^{k+1}_i + \eta \nabla f_i(z^{k+1}_i).
\end{equation*}
Moreover, $\hat{x}^{k+1}_i = 2x_i^{k+1} - y_i^{k+1}$ due to Step~\ref{step:o3} of Algoritihm~\ref{alg:A1}.
\item For any user $i \notin \Sc_k$, since $z_i^{k+1} := z_i^k$ due to \eqref{eq:z_i^k_var}, $x^{k+1}_i = x^k_i$, and $y^{k+1}_i = y^k_i$, we can also write $y^{k+1}_i$ as
\begin{equation*}
y^{k+1}_i = y^{k}_i \overset{(*)}{=} z^k_i + \eta \nabla f_i(z^k_i) = z^{k+1}_i + \eta \nabla f_i(z^{k+1}_i).
\end{equation*}
Here, $(*)$ follows from our induction assumption.
Moreover, for $i\notin\Sc_k$, we maintain $\hat{x}^{k+1}_i = \hat{x}_i^k$ in Algoritihm~\ref{alg:A1}.
By our induction assumption, and $x^{k+1}_i = x^k_i$ and $y^{k+1}_i = y^k_i$, we have $\hat{x}^{k+1}_i = \hat{x}_i^k = 2x_i^k - y_i^k = 2x_i^{k+1} - y_i^{k+1}$.
\end{compactitem}
In summary, both cases above imply that $y^{k+1}_i = z^{k+1}_i + \eta \nabla f_i(z^{k+1}_i)$ and $\hat{x}_i^k = 2x_i^k - y_i^k$ hold for all $i \in [n]$, which proves \eqref{eq:yk_xk_relation}.
\end{proof}

Our next lemma is to bound $\norms{\bar{x}^k - x^{k}_i}^2$ in terms of $\norms{x^{k+1}_i - x^k_i }^2$.
\begin{lemma}\label{lem:xbar_xk}
Let $\{(\bar{x}^k_i, z_i^k, x^k_i) \}$ be generated by Algorithm~\ref{alg:A1} and \eqref{eq:z_i^k_var}, and $\alpha > 0$.
Then, for all $i\in\Sc_k$ and any $\gamma_1 > 0$, we have
\begin{equation}\label{eq:xbar_xk_relation}
\arraycolsep=0.2em
\begin{array}{lcl}
\norms{\bar{x}^k - x^{k}_i}^2 & \leq & \frac{2(1 + \eta^2L^2)}{\alpha^2} \Big[ (1+\gamma_1)\norms{x_i^{k+1} - x_i^k}^2 +  \frac{2(1+\gamma_1)}{\gamma_1}\big( \norms{e^{k+1}_i}^2 + \norms{e_i^k}^2\big) \Big].
\end{array}
\end{equation}
In particular, if $e_i^k = e_i^{k+1} = 0$, then $\norms{\bar{x}^k - x^{k}_i}^2  \leq \frac{2(1 + \eta^2L^2)}{\alpha^2} \norms{x_i^{k+1} - x_i^k}^2$.
\end{lemma}

\begin{proof}
From the update of $y^{k+1}_i$ and Lemma~\ref{lem:yk_xk}, for $i \in \Sc_k$, we have
\begin{equation*} 
\bar{x}^{k} - x^{k}_i = \frac{1}{\alpha}(y^{k+1}_i - y^{k}_i) \overset{\eqref{eq:yk_xk_relation}}{=} \frac{1}{\alpha}(z^{k+1}_i - z^{k}_i) + \frac{\eta}{\alpha }(\nabla f_i(z^{k+1}_i) - \nabla f_i(z^{k}_i)).
\end{equation*}
Using this expression and $\norms{a+b}^2 \leq 2\norms{a}^2 + 2\norms{b}^2$, we can bound $\norms{\bar{x}^k - x^{k}_i}^2$ for all $i\in\Sc_k$ as
\begin{equation*}
\arraycolsep=0.2em
\begin{array}{lcl}
\norms{\bar{x}^k - x^{k}_i}^2 &= & \norms{\frac{1}{\alpha}(z^k_i - z^{k+1}_i) + \frac{\eta}{\alpha}(\nabla f_i(z^{k}_i) - \nabla f_i(z^{k+1}_i)}^2 \vspace{1ex}\\
& \leq &  \frac{2}{\alpha^2}\norms{z^k_i - z^{k+1}_i }^2+ \frac{2\eta^2}{\alpha^2}\norms{\nabla f_i(z^{k}_i) - \nabla f_i(z^{k+1}_i)}^2\vspace{1ex}\\
& \leq & \frac{2}{\alpha^2}\norms{z^{k+1}_i - z^k_i }^2+ \frac{2\eta^2 L^2}{\alpha^2}\norms{z^{k+1}_i - z^k_i }^2  \quad (\text{by the $L$-smoothness of $f_i$}) \vspace{1ex}\\
& =  & \frac{2(1+ \eta^2L^2)}{\alpha^2}\norms{x_i^{k+1} - x_i^k - e^{k+1}_i + e_i^k}^2 \quad \text{(by \eqref{eq:z_i^k_var})} \vspace{1ex}\\
&\leq & \frac{2(1 + \eta^2L^2)}{\alpha^2}\Big[ (1+\gamma_1)\norms{x_i^{k+1} - x_i^k}^2 +  \frac{2(1+\gamma_1)}{\gamma_1}\big( \norms{e^{k+1}_i}^2 + \norms{e_i^k}^2\big) \Big].
\end{array}
\end{equation*}
Here, we have used Young's inequality twice in the last inequality.
This proves  \eqref{eq:xbar_xk_relation}.
When $e_i^k = e_i^{k+1} = 0$, we can set $\gamma_1 = 0$ in the above estimate to obtain the last statement.
\end{proof}

We still need to link the norm $\sum_{i=1}^n\norms{x^k_i - \bar{x}^k}^2$ to the norm of gradient mapping $\norms{\Gc_{\eta}(\bar{x}^k)}$.
\begin{lemma}\label{le:grad_norm_bound}
Let $\{(\bar{x}^k_i, x^k_i, z_i^k) \}$ be generated by Algorithm~\ref{alg:A1} and \eqref{eq:z_i^k_var}, and $\alpha > 0$ and $\Gc_{\eta}$ be defined by \eqref{eq:grad_mapping}.
Then, for any $\gamma_2 > 0$, we have
\begin{equation}\label{eq:grad_mapping_bound}
\hspace{-0.1ex}
\Vert \Gc_{\eta}(\bar{x}^k) \Vert^2  \leq  \frac{1}{n\eta^2}\bigg\{ (1 + \eta L)^2 \sum_{i=1}^n\Big[ (1+\gamma_2) \norms{x_i^k  - \bar{x}^k}^2 + \frac{(1+\gamma_2)}{\gamma_2}\norms{e_i^k}^2 \Big]  \bigg\}.
\hspace{-2ex}
\end{equation}
In particular, if $e_i^k = 0$ for all $i\in [n]$, then we have $\Vert \Gc_{\eta}(\bar{x}^k) \Vert^2 \leq  \frac{(1 + \eta L)^2}{n\eta^2}  \sum_{i=1}^n \norms{x_i^k  - \bar{x}^k}^2$.
\end{lemma}

\begin{proof}
From Step~\ref{step:o5} of Algorithm~\ref{alg:A1} and \eqref{eq:yk_xk_relation}, we have
\begin{equation}\label{eq:lma5_proof1}
\arraycolsep=0.2em
\begin{array}{lcl}
\tilde{x}^k \overset{\tiny\text{Step~\tiny\ref{step:o5}}}{ =}  \frac{1}{n}\sum_{i=1}^n\hat{x}^k_i \overset{\tiny\eqref{eq:yk_xk_relation}}{=} \frac{1}{n}\sum_{i=1}^n(2x^k_i - y^k_i) \overset{\tiny\eqref{eq:yk_xk_relation}}{=} \frac{1}{n}\sum_{i=1}^n(2x_i^k - z_i^k - \eta\nabla{f_i}(z_i^k)).
\end{array}
\end{equation}
From the definition \eqref{eq:grad_mapping} of $\Gc_{\eta}$ and the update of $\bar{x}^k$, we have
\begin{equation*}
\arraycolsep=0.2em
\begin{array}{lcl}
\eta\norms{\Gc_{\eta}(\bar{x}^k)} & \overset{\tiny\eqref{eq:grad_mapping}}{=} & \Vert \bar{x}^k - \prox_{\eta g}(\bar{x}^k - \eta\nabla{f}(\bar{x}^k)) \Vert \vspace{1ex}\\
&= & \Vert \prox_{\eta g}\big( \tilde{x}^k \big) -  \prox_{\eta g}(\bar{x}^k - \eta\nabla{f}(\bar{x}^k)) \Vert \vspace{1ex}\\
&\leq & \Vert  \tilde{x}^k - \bar{x}^k + \eta \nabla{f}(\bar{x}^k) \Vert \vspace{1ex}\\
&\overset{\tiny\eqref{eq:lma5_proof1}}{=} & \frac{1}{n}\Vert \sum_{i=1}^n[ (2x_i^k - z_i^k - \bar{x}^k) + \eta(\nabla{f_i}(\bar{x}^k) - \nabla{f_i}(z_i^k)]\Vert,
\end{array}
\end{equation*}
where we have used the non-expansive property of $\prox_{g}$ in the first inequality and $\nabla{f}(\bar{x}^k) = \frac{1}{n}\sum_{i=1}^n\nabla{f_i}(\bar{x}^k)$ in the last line.

Finally, using the $L$-smoothness of $f_i$, we can derive from the last inequality that
\begin{equation*}
\arraycolsep=0.2em
\begin{array}{lcl}
\eta^2\Vert\Gc_{\eta}(\bar{x}^k)\Vert^2 &\leq & \frac{1}{n^2}\Big[ \sum_{i=1}^n \big( \norms{2x_i^k  - z_i^k - \bar{x}^k} + \eta L\norms{z_i^k - \bar{x}^k} \big) \Big]^2 \vspace{1ex}\\
&\leq & \frac{1}{n}  \sum_{i=1}^n \big( \norms{2x_i^k  - z_i^k - \bar{x}^k} + \eta L\norms{z_i^k - \bar{x}^k} \big)^2  \vspace{1ex}\\
&\leq & \frac{1}{n}  \sum_{i=1}^n \big[ (1+\eta L) \norms{x_i^k  - \bar{x}^k} + (1+\eta L)\norms{e_i^k} \big]^2 \vspace{1ex}\\
&\leq & \frac{1}{n} (1+\eta L)^2 \sum_{i=1}^n \big[ (1+\gamma_2) \norms{x_i^k  - \bar{x}^k}^2 + \frac{(1+\gamma_2)}{\gamma_2}\norms{e_i^k}^2 \big],
\end{array}
\end{equation*}
which proves \eqref{eq:grad_mapping_bound}, where $\gamma_2 > 0$.
Here, we have used Young's inequality in the second and the last inequalities, and $x_i^k = z_i^k + e_i^k$ from \eqref{eq:z_i^k_var} in the third line.
\end{proof}

To analyze convergence of Algorithm~\ref{alg:A1}, we introduce the following Lyapunov function:
\begin{equation}\label{eq:lyapunov_func}
V_{\eta}^{k}(\bar{x}^{k})  :=   g(\bar{x}^{k}) + \frac{1}{n}\sum_{i=1}^n \Big[ f_i(x^{k}_i) +  \iprods{\nabla f_i(x^{k}_i), \bar{x}^{k} - x^{k}_i} + \frac{1}{2\eta}\norms{\bar{x}^{k} - x^{k}_i}^2 \Big]. 
\end{equation}
First, we prove the following lemma.

\begin{lemma}\label{le:V_properties_inexact}
Suppose that Assumption~\ref{ass:A1}, \ref{ass:A2}, and \ref{ass:A3} hold.
Let $\{ (z_i^k, x_i^k, y_i^k, \hat{x}^{k}_i, \bar{x}^{k})\}$ be generated by Algorithm~\ref{alg:A1} and \eqref{eq:z_i^k_var}. 
Let $V_{\eta}^{k}$ be defined by \eqref{eq:lyapunov_func}.
Then, for any $\gamma_3 > 0$, we have
\begin{equation}\label{eq:V_pro2_inexact}
\hspace{-0ex}
\arraycolsep=0.2em
\begin{array}{lcl}
V_{\eta}^{k+1}(\bar{x}^{k+1}) & \leq & g(\bar{x}^k) + \frac{1}{n}\sum_{i=1}^n \big[ f_i(x^{k+1}_i) +  \iprods{\nabla f_i(x^{k+1}_i), \bar{x}^{k} - x^{k+1}_i} + \frac{1}{2\eta} \norms{\bar{x}^{k} - x^{k+1}_i}^2 \big] \vspace{1ex}\\
&& - {~} \frac{(1 - \gamma_3)}{2\eta}\norms{\bar{x}^{k+1} - \bar{x}^k}^2 +  \frac{(1+\eta^2L^2)}{\gamma_3\eta}E_{k+1}^2,
\end{array}
\hspace{-3ex}
\end{equation}
where $E_{k+1}^2 := \frac{1}{n}\sum_{i\notin\Sc_k}\norms{e_i^k}^2 + \frac{1}{n}\sum_{i\in\Sc_k}\norms{e_i^{k+1}}^2$.
If $E_{k+1} = 0$, then we allow $\gamma_3 = 0$.
\end{lemma}

\begin{proof}
First, from $\bar{x}^{k+1} = \prox_{\eta g}\big( \tilde{x}^{k+1} \big)$ at Step~\ref{step:o5} of Algorithm~\ref{alg:A1}, we have $\frac{1}{\eta}(\tilde{x}^{k+1} - \bar{x}^{k+1}) \in  \partial{g}(\bar{x}^{k+1})$. 
Using this expression and the convexity of $g$, we obtain
\begin{equation}\label{eq:proof_est1_inexact}
\arraycolsep=0.2em
\begin{array}{lcl}
g(\bar{x}^{k+1}) &\leq & g(\bar{x}^k) + \frac{1}{\eta}\iprods{\tilde{x}^{k+1} - \bar{x}^k, \bar{x}^{k+1} - \bar{x}^k}  -  \frac{1}{\eta}\norms{\bar{x}^{k+1} - \bar{x}^k}^2.
\end{array}
\end{equation}
Next, since $y_i^{k+1} = z^{k+1}_i + \eta\nabla{f_i}(z^{k+1}_i)$ due to \eqref{eq:yk_xk_relation} and $x^{k+1}_i = z^{k+1}_i + e_i^{k+1}$ due to \eqref{eq:z_i^k_var}, we have 
\begin{equation}\label{eq:proof_est2_inexact}
\arraycolsep=0.2em
\begin{array}{lcl}
x_i^{k+1} + \eta\nabla{f_i}(x_i^{k+1}) & \overset{\tiny\eqref{eq:z_i^k_var}}{=} & z_i^{k+1} + \eta \nabla{f_i}(z_i^{k+1}) + e_i^{k+1} + \eta(\nabla{f_i}(x_i^{k+1}) - \nabla{f_i}(z_i^{k+1})) \vspace{1ex}\\
&\overset{\tiny\eqref{eq:yk_xk_relation}}{ = } & y_i^{k+1} + e_i^{k+1} + \eta\xi_i^{k+1}, 
\end{array}
\end{equation}
where $\xi_i^{k+1} := \nabla{f_i}(x_i^{k+1}) - \nabla{f_i}(z_i^{k+1})$.
Using this relation, we can derive 
\begin{equation*}
\arraycolsep=0.2em
\begin{array}{lcl}
\Delta_{k+1} &:= &  \frac{1}{n}\sum_{i=1}^n \Big[ f_i(x^{k+1}_i) +  \iprods{\nabla f_i(x^{k+1}_i), \bar{x}^{k+1} - x^{k+1}_i} + \frac{1}{2\eta} \norms{\bar{x}^{k+1} - x^{k+1}_i}^2 \Big] \vspace{1ex}\\
&= & \frac{1}{n}\sum_{i=1}^n \Big[ f_i(x^{k+1}_i) +  \iprods{\nabla f_i(x^{k+1}_i), \bar{x}^{k} - x^{k+1}_i} + \frac{1}{2\eta}\norms{\bar{x}^k - x^{k+1}_i}^2 \Big] \vspace{1ex}\\
&& + {~} \frac{1}{n\eta}\sum_{i=1}^n\iprods{\bar{x}^k - 2x^{k+1}_i + (x^{k+1}_i + \eta\nabla{f_i}(x^{k+1}_i)), \bar{x}^{k+1} - \bar{x}^k} +  \frac{1}{2\eta}\norms{\bar{x}^{k+1} - \bar{x}^k}^2 \vspace{1ex}\\
&\overset{\tiny\eqref{eq:proof_est2_inexact}}{=} & \frac{1}{n}\sum_{i=1}^n \Big[ f_i(x^{k+1}_i) +  \iprods{\nabla f_i(x^{k+1}_i), \bar{x}^{k} - x^{k+1}_i} + \frac{1}{2\eta}\norms{\bar{x}^k - x^{k+1}_i}^2 \Big] \vspace{1ex}\\
&& + {~} \frac{1}{n\eta}\sum_{i=1}^n\iprods{\bar{x}^k - 2x^{k+1}_i + y_i^{k+1}, \bar{x}^{k+1} - \bar{x}^k} +  \frac{1}{2\eta}\norms{\bar{x}^{k+1} - \bar{x}^k}^2 \vspace{1ex}\\
&& + {~} \frac{1}{n\eta}\sum_{i=1}^n\iprods{e_i^{k+1} + \eta \xi_i^{k+1}, \bar{x}^{k+1} - \bar{x}^k} \vspace{1ex}\\
& \overset{\tiny\text{Step}~\ref{step:o5}}{=} & 
\frac{1}{n}\sum_{i=1}^n \Big[ f_i(x^{k+1}_i) +  \iprods{\nabla f_i(x^{k+1}_i), \bar{x}^{k} - x^{k+1}_i} + \frac{1}{2\eta}\norms{\bar{x}^k - x^{k+1}_i}^2 \Big] \vspace{1ex}\\
&& + {~} \frac{1}{\eta}\iprods{\bar{x}^k - \tilde{x}^{k+1}, \bar{x}^{k+1} - \bar{x}^k} +  \frac{1}{2\eta}\norms{\bar{x}^{k+1} - \bar{x}^k}^2 \vspace{1ex}\\
&& + {~} \frac{1}{n\eta}\sum_{i=1}^n\iprods{e_i^{k+1} + \eta \xi_i^{k+1}, \bar{x}^{k+1} - \bar{x}^k}.
\end{array}
\end{equation*}
Summing up this expression and \eqref{eq:proof_est1_inexact}, and using the definition of $V_{\eta}^k$ in \eqref{eq:lyapunov_func}, we get
\begin{equation*}
\arraycolsep=0.2em
\begin{array}{lcl}
V_{\eta}^{k+1}(\bar{x}^{k+1}) & = & \frac{1}{n}\sum_{i=1}^n \big[ f_i(x^{k+1}_i) +  \iprods{\nabla f_i(x^{k+1}_i), \bar{x}^{k+1} - x^{k+1}_i} + \frac{1}{2\eta}\norms{\bar{x}^{k+1} - x^{k+1}_i}^2\big] + g(\bar{x}^{k+1}) \vspace{1ex}\\
& \leq & g(\bar{x}^k) + \frac{1}{n}\sum_{i=1}^n \Big[ f_i(x^{k+1}_i) +  \iprods{\nabla f_i(x^{k+1}_i), \bar{x}^{k} - x^{k+1}_i} + \frac{1}{2\eta}\norms{\bar{x}^k - x^{k+1}_i}^2 \Big] \vspace{1ex}\\
&& - {~} \frac{1}{2\eta}\norms{\bar{x}^{k+1} - \bar{x}^k}^2 +  \frac{1}{n\eta}\sum_{i=1}^n\iprods{e_i^{k+1} + \eta \xi_i^{k+1}, \bar{x}^{k+1} - \bar{x}^k}.
\end{array}
\end{equation*}
By Young's inequality and $e_i^{k+1} = e_i^k$ for $i\notin\Sc_k$ due to \eqref{eq:z_i^k_var}, for any $\gamma_3 > 0$, we can estimate
\begin{equation*}
\arraycolsep=0.2em
\begin{array}{lcl}
\Tc_{[1]} & := & \frac{1}{n\eta}\sum_{i=1}^n\iprods{e_i^{k+1} + \eta \xi_i^{k+1}, \bar{x}^{k+1} - \bar{x}^k}  \vspace{1ex}\\
& \leq & \frac{1}{2n\eta}\sum_{i=1}^n\big[\frac{1}{\gamma_3}\norms{e_i^{k+1} + \eta\xi_i^{k+1}}^2 + \gamma_3 \norms{\bar{x}^{k+1} - \bar{x}^k}^2 \big] \vspace{1ex}\\
& \leq & \frac{\gamma_3}{2\eta} \norms{\bar{x}^{k+1} - \bar{x}^k}^2  + \frac{1}{n\eta\gamma_3}\sum_{i=1}^n\norms{e_i^{k+1}}^2 + \frac{\eta}{n\gamma_3}\sum_{i=1}^n\norms{\nabla{f_i}(x_i^{k+1}) - \nabla{f_i}(z_i^{k+1})}^2 \vspace{1ex}\\
& \overset{\tiny\eqref{eq:L_smooth}}{\leq} & \frac{\gamma_3}{2\eta} \norms{\bar{x}^{k+1} - \bar{x}^k}^2  + \frac{(1+\eta^2L^2)}{n\eta\gamma_3}\big[ \sum_{i\in\Sc_k}\norms{e_i^{k+1}}^2 + \sum_{i\notin\Sc_k}\norms{e_i^k}^2\big].
\end{array}
\end{equation*}
Substituting this inequality into the last estimate, we eventually obtain \eqref{eq:V_pro2_inexact}.
However, if $E_{k+1}^2 = 0$, then we can deduce from the above inequality that $\gamma_3$ can be set to zero.
\end{proof}

Now, we prove the following key result, which holds surely for any subset $\Sc_k$ of $[n]$.

\begin{lemma}[Sure descent lemma]\label{lem:dr_key_est0_inexact}
Suppose that Assumption~\ref{ass:A1}, \ref{ass:A2}, and \ref{ass:A3} hold.
Let $\{ (x^{k}_i, y^{k}_i, z_i^k, \hat{x}^{k}_i, \bar{x}^{k})\}$ be generated by Algorithm~\ref{alg:A1} and \eqref{eq:z_i^k_var}, and $V_\eta^{k}(\cdot)$ be defined by \eqref{eq:lyapunov_func}.
Then, the following estimate holds: 
\begin{equation}\label{eq:dr_descent_lem0_inexact}
\arraycolsep=0.2em
\begin{array}{lcl}
V_\eta^{k+1}(\bar{x}^{k+1}) & \leq & V_\eta^{k}(\bar{x}^k) - \frac{[ 2 - \alpha(L\eta + 1) - 2L^2\eta^2 - 4\alpha\gamma_4(1 + L^2\eta^2)]}{2\alpha\eta n}\sum_{i\in\Sc_k} \norms{x^{k+1}_i - x^{k}_i}^2 \vspace{1ex}\\
&& - {~}  \frac{(1 - \gamma_3)}{2\eta}\norms{\bar{x}^{k+1} - \bar{x}^k}^2  + \frac{(1+\eta^2L^2)}{\eta\gamma_3}E_{k+1}^2 \vspace{1ex}\\
&& + {~}   \frac{2(1 + \eta L)^2}{\gamma_4\eta\alpha^2 n }\sum_{i\in\Sc_k}[\norms{e_i^k}^2 + \norms{e_i^{k+1}}^2],
\end{array}
\end{equation}
where $E_{k+1}^2 := \frac{1}{n}\sum_{i\notin\Sc_k}\norms{e_i^k}^2 + \frac{1}{n}\sum_{i\in\Sc_k}\norms{e_i^{k+1}}^2$, and $\gamma_3, \gamma_4 > 0$.
In particular, if $E_{k+1}^2 = 0$, then we allow $\gamma_3 = 0$, and if $e^k_i = e_i^{k+1} = 0$ for all $i\in \Sc_k$, then we allow $\gamma_4 = 0$.
\end{lemma}

\begin{proof}
First, using \eqref{eq:V_pro2_inexact}, we can further derive
\begin{equation}\label{eq:dr_eq5_new_inexact}
\hspace{-0.5ex}
\arraycolsep=0.1em
\begin{array}{lcl}
V_\eta^{k+1}(\bar{x}^{k+1}) &\overset{\tiny\eqref{eq:V_pro2_inexact}}{\leq} & \frac{1}{n}\sum_{i=1}^n \big[ f_i(x^{k+1}_i) +  \iprods{\nabla f_i(x^{k+1}_i), \bar{x}^{k} - x^{k+1}_i} + \frac{1}{2\eta} \norms{\bar{x}^{k} - x^{k+1}_i}^2 \big] \vspace{1ex}\\
&& + {~} g(\bar{x}^k)  -  \frac{(1 - \gamma_3)}{2\eta}\norms{\bar{x}^{k+1} - \bar{x}^k}^2  +  \frac{(1+\eta^2L^2)}{\eta\gamma_3 }E_{k+1}^2 \vspace{1ex}\\
&\overset{(*)}{=} & \frac{1}{n}\sum_{i\in\Sc_k} f_i(x^{k+1}_i) + \frac{1}{n}\sum_{i\in\Sc_k} \iprods{\nabla f_i(x^{k+1}_i), x^{k}_i - x^{k+1}_i} \vspace{1ex}\\
&&+ {~}  \frac{1}{n}\sum_{i\in\Sc_k}\iprods{\nabla f_i(x^{k+1}_i), \bar{x}^{k} - x^{k}_i} +  \frac{1}{2\eta n}\sum_{i\in\Sc_k}\norms{\bar{x}^{k} - x^{k+1}_i}^2  \vspace{1ex}\\
&& + {~} \frac{1}{n}\sum_{i\notin\Sc_k} f_i(x^{k}_i) + \frac{1}{n}\sum_{i\notin\Sc_k} \iprods{\nabla f_i(x^{k}_i), \bar{x}^{k} - x^{k}_i} + \frac{1}{2\eta n}\sum_{i\notin\Sc_k} \norms{\bar{x}^{k} - x^{k}_i}^2 \vspace{1ex}\\
&& + {~} g(\bar{x}^k) -  \frac{(1 - \gamma_3)}{2\eta}\norms{\bar{x}^{k+1} - \bar{x}^k}^2  +  \frac{(1+\eta^2L^2)}{\eta\gamma_3 }E_{k+1}^2,
\end{array}
\hspace{-6ex}
\end{equation}
where in (*) we have used the fact that only users in $\Sc_k$ perform update and added/subtracted $x^k_i$ in the term $\iprods{\nabla f_i(x^{k+1}_i), \bar{x}^{k} - x^{k+1}_i}$.

On the other hand, from the $L$-smoothness of $f_i$, we have
\begin{equation*}
f_i(x^{k+1}_i) + \iprods{\nabla f(x^{k+1}_i), x^{k}_i - x^{k+1}_i} \le f_i(x^{k}_i) + \frac{L}{2}\norms{x^{k+1}_i - x^{k}_i}^2.
\end{equation*}
Substituting this inequality into \eqref{eq:dr_eq5_new_inexact}, we can further bound it as
\begin{equation}\label{eq:dr_eq5_2_new_inexact}
\hspace{-0.0ex}
\arraycolsep=0.2em
\begin{array}{lcl}
V_\eta^{k+1}(\bar{x}^{k+1}) &\leq & \frac{1}{n}\sum_{i\in\Sc_k} f_i(x^k_i) + \frac{L}{2n}\sum_{i\in\Sc_k} \norms{x^{k+1}_i - x^{k}_i}^2 + \frac{1}{n}\sum_{i\in\Sc_k}\iprods{\nabla f_i(x^{k+1}_i), \bar{x}^{k} - x^{k}_i} \vspace{1ex}\\
&& + {~} \frac{1}{2\eta n}\sum_{i\in\Sc_k} \norms{\bar{x}^{k} - x^{k+1}_i}^2   + \frac{1}{n}\sum_{i\notin\Sc_k} f_i(x^{k}_i) + \frac{1}{n}\sum_{i\notin\Sc_k} \iprods{\nabla f_i(x^{k}_i), \bar{x}^{k} - x^{k}_i}  \vspace{1ex}\\
&& + {~} \frac{1}{2\eta n}\sum_{i\notin\Sc_k} \norms{\bar{x}^{k} - x^{k}_i}^2 +  g(\bar{x}^k)  \vspace{1ex}\\
&& - {~}  \frac{(1 - \gamma_3)}{2\eta}\norms{\bar{x}^{k+1} - \bar{x}^k}^2 +  \frac{(1+\eta^2L^2)}{\eta\gamma_3}E_{k+1}^2 \vspace{1ex}\\
&= & \frac{1}{n}\sum_{i=1}^n f_i(x^k_i) + \frac{1}{n}\sum_{i=1}^n\iprods{\nabla f_i(x^{k}_i), \bar{x}^{k} - x^{k}_i} + \frac{L}{2n}\sum_{i\in\Sc_k} \norms{x^{k+1}_i - x^{k}_i}^2  \vspace{1ex}\\
&& + {~} \frac{1}{2\eta n}\sum_{i\in\Sc_k} \norms{\bar{x}^{k} - x^{k+1}_i}^2 +\frac{1}{n}\sum_{i\in\Sc_k} \iprods{\nabla f_i(x^{k+1}_i) - \nabla f_i(x^{k}_i), \bar{x}^{k} - x^{k}_i} \vspace{1ex}\\
&& + {~} \frac{1}{2\eta n}\sum_{i\notin\Sc_k} \norms{\bar{x}^{k} - x^{k}_i}^2 +  g(\bar{x}^k)  \vspace{1ex}\\
&& - {~}  \frac{(1 - \gamma_3)}{2\eta}\norms{\bar{x}^{k+1} - \bar{x}^k}^2 +  \frac{(1+\eta^2L^2)}{\eta\gamma_3}E_{k+1}^2,
\end{array}
\hspace{-6ex}
\end{equation}
where we have added and subtracted $\frac{1}{n}\sum_{i\in\Sc_k}\iprods{\nabla f_i(x^{k}_i), \bar{x}^{k} - x^{k}_i}$ to obtain the last equality. 

Next, using the following elementary expression 
\begin{equation*}
\norms{ \bar{x}^{k} - x^{k+1}_i}^2 = \norms{\bar{x}^k - x^k_i}^2 + 2\iprods{\bar{x}^k - x^k_i, x^k_i - x^{k+1}_i} + \norms{x_i^k - x^{k+1}_i}^2
\end{equation*}
into \eqref{eq:dr_eq5_2_new_inexact}, we can further derive
\begin{equation}\label{eq:dr_eq5_3_new_inexact}
\hspace{-0.0ex}
\arraycolsep=0.2em
\begin{array}{lcl}
V_\eta^{k+1}(\bar{x}^{k+1}) &\leq &  g(\bar{x}^k) + \frac{1}{n}\sum_{i=1}^n \Big[ f_i(x^k_i) +   \iprods{\nabla f_i(x^{k}_i), \bar{x}^{k} - x^{k}_i}  + \frac{1}{2\eta} \norms{\bar{x}^{k} - x^{k}_i}^2 \Big] \vspace{1ex}\\
&& + {~} \frac{1}{2\eta n}\sum_{i\in\Sc_k} \norms{x^{k+1}_i - x^{k}_i}^2 +  \frac{1}{\eta n}\sum_{i\in\Sc_k} \iprods{x^{k+1}_i -  x^k_i, x^k_i - \bar{x}^k}  \vspace{1ex}\\
&& + {~} \frac{1}{n}\sum_{i\in\Sc_k} \iprods{\nabla f_i(x^{k+1}_i) - \nabla f_i(x^{k}_i), \bar{x}^{k} - x^{k}_i} + \frac{L}{2n}\sum_{i\in\Sc_k} \norms{x^{k+1}_i - x^{k}_i}^2  \vspace{1ex}\\
&& - {~}  \frac{(1 - \gamma_3)}{2\eta}\norms{\bar{x}^{k+1} - \bar{x}^k}^2  +  \frac{(1+\eta^2L^2)}{\eta\gamma_3}E_{k+1}^2 \vspace{1ex}\\
&= & V_\eta^{k}(\bar{x}^k)  + \frac{1 + \eta L}{2\eta n}\sum_{i\in\Sc_k} \norms{x^{k+1}_i - x^{k}_i}^2 +  \frac{1}{\eta n}\sum_{i\in\Sc_k} \iprods{x^{k+1}_i -  x^k_i, x^k_i - \bar{x}^k}  \vspace{1ex}\\
&& - {~}   \frac{(1 - \gamma_3)}{2\eta}\norms{\bar{x}^{k+1} - \bar{x}^k}^2  +  \frac{(1+\eta^2L^2)}{\eta\gamma_3}E_{k+1}^2 \vspace{1ex}\\
&& + {~} \frac{1}{n}\sum_{i\in\Sc_k} \iprods{\nabla f_i(x^{k+1}_i) - \nabla f_i(x^{k}_i), \bar{x}^{k} - x^{k}_i}.
\end{array}
\hspace{-2ex}
\end{equation}
From the update of $y^{k+1}_i$, for $i \in \Sc_k$, and similar to the proof of \eqref{eq:proof_est2_inexact}, we have
\begin{equation*} 
\arraycolsep=0.2em
\begin{array}{lcl}
x^{k}_i - \bar{x}^{k} & = & \frac{1}{\alpha}(y^{k}_i - y^{k+1}_i) \vspace{1ex}\\
& \overset{\tiny\eqref{eq:proof_est2_inexact}}{=} & \frac{1}{\alpha}(z^{k}_i - z^{k+1}_i) + \frac{\eta}{\alpha }(\nabla f_i(z^{k}_i) - \nabla f_i(z^{k+1}_i)) \vspace{1ex}\\
& = & \frac{1}{\alpha}(x_i^k - x_i^{k+1}) + \frac{\eta}{\alpha }(\nabla f_i(x^{k}_i) - \nabla f_i(x^{k+1}_i)) + \frac{1}{\alpha}[ (e_i^{k+1} + \eta\xi_i^{k+1}) - (e_i^k + \eta\xi_i^k) ] \vspace{1ex}\\
& = & \frac{1}{\alpha}(x_i^k - x_i^{k+1}) + \frac{\eta}{\alpha }(\nabla f_i(x^{k}_i) - \nabla f_i(x^{k+1}_i)) + s_i^k,
\end{array}
\end{equation*}
where $s_i^k :=  \frac{1}{\alpha}[e_i^{k+1} + \eta\xi_i^{k+1} - (e_i^k + \eta\xi_i^k) )$ with $\xi_i^{k} := \nabla{f_i}(x_i^{k}) - \nabla{f_i}(z_i^{k})$.

Consequently, using the last expression and  the $L$-smoothness of $f_i$, we can further bound \eqref{eq:dr_eq5_3_new_inexact} as
\begin{equation*} 
\arraycolsep=0.2em
\begin{array}{lcl}
V_\eta^{k+1}(\bar{x}^{k+1}) & \leq & V_\eta^{k}(\bar{x}^k)  + \frac{(1 + \eta L)}{2\eta n}\sum_{i\in\Sc_k} \norms{x^{k+1}_i - x^{k}_i}^2  - \frac{1}{\alpha\eta n}\sum_{i\in\Sc_k} \norms{x^{k+1}_i -  x^k_i}^2  \vspace{1ex}\\
&& -  {~} \frac{1}{\alpha n}\sum_{i\in\Sc_k} \iprods{x^{k+1}_i -  x^k_i, \nabla f_i(x^{k+1}_i) - \nabla f_i(x^{k}_i)} + \frac{1}{\eta n}\sum_{i\in\Sc_k} \iprods{s_i^k, x^{k+1}_i -  x^k_i} \vspace{1ex}\\
&& + {~} \frac{1}{\alpha n}\sum_{i\in\Sc_k} \iprods{\nabla f_i(x^{k+1}_i) - \nabla f_i(x^{k}_i), x^{k+1}_i - x^{k}_i}   \vspace{1ex}\\
&& + {~}  \frac{\eta}{\alpha n}\sum_{i\in\Sc_k} \norms{\nabla f_i(x^{k+1}_i) - \nabla f_i(x^{k}_i)}^2  + \frac{1}{ n}\sum_{i\in\Sc_k} \iprods{s_i^k, \nabla f_i(x^{k+1}_i) - \nabla f_i(x^{k}_i)}\vspace{1ex}\\
&& - {~}   \frac{(1 - \gamma_3)}{2\eta}\norms{\bar{x}^{k+1} - \bar{x}^k}^2  +  \frac{(1+\eta^2L^2)}{\eta\gamma_3 }E_{k+1}^2 \vspace{1ex}\\
&= & V_\eta^{k}(\bar{x}^k) +  \frac{\eta}{\alpha n}\sum_{i\in\Sc_k} \norms{\nabla f_i(x^{k+1}_i) - \nabla f_i(x^{k}_i)}^2  +  \frac{[ \alpha(L\eta + 1) - 2] }{2\alpha\eta n}\sum_{i\in\Sc_k} \norms{x^{k+1}_i - x^{k}_i}^2 \vspace{1ex}\\
&& + {~}  \frac{1}{\eta n}\sum_{i\in\Sc_k} \iprods{s_i^k, (x^{k+1}_i -  x^k_i) + \eta(\nabla f_i(x^{k+1}_i) - \nabla f_i(x^{k}_i))} \vspace{1ex}\\
&& - {~}   \frac{(1 - \gamma_3)}{2\eta}\norms{\bar{x}^{k+1} - \bar{x}^k}^2  +  \frac{(1+\eta^2L^2)}{\eta\gamma_3 }E_{k+1}^2 \vspace{1ex}\\
&\overset{\eqref{eq:L_smooth}}{\leq} & V_\eta^{k}(\bar{x}^k) +  \frac{\eta L^2}{\alpha n}\sum_{i\in\Sc_k} \norms{x^{k+1}_i - x^{k}_i}^2  +  \frac{[ \alpha(L\eta + 1) - 2] }{2\alpha\eta n}\sum_{i\in\Sc_k} \norms{x^{k+1}_i - x^{k}_i}^2 \vspace{1ex}\\
&& - {~}   \frac{(1 - \gamma_3)}{2\eta}\norms{\bar{x}^{k+1} - \bar{x}^k}^2   +  \frac{(1+\eta^2L^2)}{\eta\gamma_3 }E_{k+1}^2 \vspace{1ex}\\
&& + \frac{1}{n\eta}\sum_{i\in\Sc}\big[\frac{1}{\gamma_4} \norms{s_i^k}^2 + 2\gamma_4\norms{x_i^k - x_i^{k+1}}^2 + 2\gamma_4\eta^2\norms{\nabla{f_i}(x_i^k) - \nabla{f_i}(x_i^{k+1})}^2 \big] \vspace{1ex}\\
&= & V_\eta^{k}(\bar{x}^k)  - \frac{[ 2 - \alpha(L\eta + 1) - 2L^2\eta^2] }{2\alpha\eta n}\sum_{i\in\Sc_k} \norms{x^{k+1}_i - x^{k}_i}^2 \vspace{1ex}\\
&& + {~} \frac{1}{n\gamma_4\eta}\sum_{i\in\Sc}\norms{s_i^k}^2 + \frac{2\gamma_4(1 + L^2\eta^2)}{n\eta} \sum_{i\in\Sc_k}\norms{x_i^{k+1} - x_i^{k}}^2 \vspace{1ex}\\
&& - {~}   \frac{(1 - \gamma_3)}{2\eta}\norms{\bar{x}^{k+1} - \bar{x}^k}^2  +  \frac{(1+\eta^2L^2)}{\eta\gamma_3 }E_{k+1}^2.
\end{array}
\end{equation*}
Finally, we bound $\norms{s_i^k}^2$ as follows:
\begin{equation*}
\arraycolsep=0.2em
\begin{array}{lcl}
\norms{s_i^k}^2 &= &  \frac{1}{\alpha^2}\norms{e_i^{k+1} + \eta\xi_i^{k+1} - (e_i^k + \eta\xi_i^k)}^2 \vspace{1ex}\\
&\leq & \frac{1}{\alpha^2}\big[ \norms{e_i^k} + \norms{e_i^{k+1}} +  \eta\norms{\nabla{f_i}(x_i^{k}) - \nabla{f_i}(z_i^{k})} + \eta\norms{\nabla{f_i}(x_i^{k+1}) - \nabla{f_i}(z_i^{k+1})} \big]^2 \vspace{1ex}\\
&\leq & \frac{2(1 + \eta L)^2}{\alpha^2}( \norms{e_i^k}^2 + \norms{e_i^{k+1}}^2).
\end{array}
\end{equation*}
Substituting this inequality into the last estimate, we obtain \eqref{eq:dr_descent_lem0_inexact}.
The last statement follows from the last statement of Lemmas~\ref{le:grad_norm_bound} and \ref{le:V_properties_inexact}.
\end{proof}

\subsection{The descent property of Algorithm~\ref{alg:A1}}
We prove a descent property of Algorithm~\ref{alg:A1}, where $\prox_{\eta f_i}$ is evaluated approximately.

\begin{lemma}\label{lem:dr_key_est_inexact}
Suppose that Assumption~\ref{ass:A1}, \ref{ass:A2}, and \ref{ass:A3} hold.
Let  $V_\eta^{k}(\cdot)$ be defined by \eqref{eq:lyapunov_func} and $\gamma_1, \gamma_2, \gamma_4 > 0$ be given.
Let $\{ (x^{k}_i, y^{k}_i,  \hat{x}^{k}_i, \bar{x}^{k})\}$ be generated by Algorithm~\ref{alg:A1} using 
\begin{equation}\label{eq:stepsizes_choice}
0 < \alpha < \frac{\min\{8, \sqrt{17+64\gamma_4} - 1\}}{4(1+4\gamma_4)} \quad \text{and} \quad 0 < \eta < \frac{\sqrt{(4-\alpha)^2 - 16\alpha^2\gamma_4(1+4\gamma_4)} - \alpha}{4L(1+2\alpha\gamma_4)}.
\end{equation}
Then, $V_{\eta}^k$ is bounded from bellow by $F^{\star}$, i.e. $V_{\eta}^k \geq F^{\star}$ and the following estimate holds:
\begin{equation}\label{eq:dr_descent_lem_inexact}
\frac{\beta}{2n}  \sum_{i=1}^n \norms{\bar{x}^k - x^{k}_i}^2  \leq  V_\eta^{k}(\bar{x}^k) - \Exp{V_{\eta}^{k+1}(\bar{x}^{k+1}) \mid \Fc_{k-1}}  +  \frac{1}{n}\sum_{i=1}^n(\rho_1 \epsilon_{i,k}^2 + \rho_2 \epsilon_{i,k+1}^2),
\end{equation}
where
\begin{equation}\label{eq:beta_consts}
\arraycolsep=0.2em
\left\{\begin{array}{lcl}
\beta & := & \frac{\hat{\mbf{p}}\alpha[ 2 - \alpha(L\eta + 1) - 2L^2\eta^2 - 4\gamma_4\alpha(1 + L^2\eta^2) ]}{2\eta(1+\gamma_1)(1+L^2\eta^2)} > 0, \vspace{1ex}\\
\rho_2 &:= & \frac{2(1 + \eta L)^2}{\gamma_4\eta\alpha^2  } + \frac{(1+\eta^2L^2)}{\eta} + \frac{\alpha [ 2 - \alpha(L\eta + 1) - 2L^2\eta^2 - 4\alpha\gamma_4(1 + L^2\eta^2)]}{2\eta(1 + L^2\eta^2)\gamma_1}, \vspace{1ex}\\
\rho_1 & := & \rho_2 +  \frac{(1+\eta^2L^2)}{\eta}.
\end{array}\right.
\end{equation}
Here, if $\epsilon_{i,k} = 0$ for all $i\in [n]$ and $k\geq 0$, then we allow $\gamma_1 = \gamma_2 = \gamma_4 = \rho_1 = \rho_2 = 0$.
\end{lemma}

\begin{proof}
First, to guarantee a descent property in \eqref{eq:dr_descent_lem0_inexact}, we need to choose $\eta > 0$ and $\alpha > 0$ such that $2 - \alpha(L\eta + 1) - 2L^2\eta^2 - 4\gamma_4\alpha(1 + L^2\eta^2) > 0$. 
We first need $\alpha$ such that $0 < \alpha < \frac{2}{1 + 4\gamma_4}$, the condition for $\eta$ is
\begin{equation*}
\arraycolsep=0.2em
\begin{array}{lcl}
0 < \eta < \bar{\eta} := \frac{\sqrt{(4-\alpha)^2 + 16\alpha^2\gamma_4(1+4\gamma_4)} - \alpha}{4L(1+2\alpha\gamma_4)}.
\end{array}
\end{equation*}
To guarantee $\bar{\eta} > 0$, we need to choose $0 < \alpha < \frac{\sqrt{17 + 64\gamma_4}-1}{4(1 + 4\gamma_4)}$.
Combining both conditions on $\alpha$, we obtain the first condition for $\alpha$ in \eqref{eq:stepsizes_choice}.

Now, to show the boundedness of $V_{\eta}^k(\bar{x}^k)$ from below, we have
\begin{equation}\label{eq:V_lowerbound}
\arraycolsep=0.2em
\begin{array}{lcl}
V_{\eta}^k(\bar{x}^k) &= &g(\bar{x}^k) + \frac{1}{n}\sum_{i=1}^n \Big[ f_i(x_i^k) + \iprods{\nabla{f_i}(x_i^k), \bar{x}^k - x_i^k} + \frac{1}{2\eta} \norms{\bar{x}^k - x_i^k}^2 \Big]  \vspace{1ex}\\
&\geq & g(\bar{x}^k) + \frac{1}{n}\sum_{i=1}^n \Big[ f_i(\bar{x}^k) - \frac{L}{2} \norms{\bar{x}^k - x_i^k}^2 + \frac{1}{2\eta}\norms{\bar{x}^k - x_i^k}^2 \Big] \ \  (\text{the $L$-smoothness of $f_i$}) \vspace{1ex}\\
&\geq & f(\bar{x}^k)  + g(\bar{x}^k) + \big(\frac{1}{\eta} - L\big)\frac{1}{2n}\sum_{i=1}^n\norms{\bar{x}^k - x_i^k}^2 \vspace{1ex}\\
&\geq & F^{\star} \quad \text{(since $\eta \leq \frac{1}{L}$ and Assumption~\ref{ass:A1})}.
\end{array}
\end{equation}
Next, from \eqref{eq:xbar_xk_relation}, we have
\begin{equation*}
\frac{\alpha^2}{2(1+L^2\eta^2)(1+\gamma_1)}\sum_{i\in\Sc_k}\norms{\bar{x}^k - x_i^k}^2 \leq \sum_{i\in\Sc_k}\Big[ \norms{x_i^{k+1} - x_i^k}^2 +  \frac{\alpha^2}{(1+L^2\eta^2)\gamma_1}\big( \norms{e^{k+1}_i}^2 + \norms{e_i^k}^2\big)\Big].
\end{equation*}
Moreover, from Assumption~\ref{ass:A3}, for a nonnegative random variable $W_i^k$ with $i \in \Sc_k$, by taking expectation of this random variable w.r.t. $\Sc_k$ conditioned on $\Fc_{k-1}$, we have
\begin{equation*}
\arraycolsep=0.1em
\begin{array}{lcl}
\Exp{\sum_{i\in\Sc_k} W_i^k \mid \Fc_{k-1}} &= & \sum_{\Sc} \mathbb{P}(\Sc_k = \Sc)\sum_{i\in\Sc}W_i^k  = \sum_{i=1}^n\sum_{\Sc : i\in\Sc}\mathbb{P}(\Sc) W_i^k \overset{\tiny\text{Ass.}~\eqref{ass:A3}}{=}  \sum_{i=1}^n\mbf{p}_iW_i^k.
\end{array}
\end{equation*}
Using this relation with $W_i^k := \norms{x_i^k - \bar{x}^k}^2$, $W_i^k := \norms{e_i^k}^2$, and $W_i^k := \norms{e_i^{k+1}}^2$, and then combining the results with the last inequality, we can derive that 
\begin{equation}\label{eq:proof102_est1}
\arraycolsep=0.2em
\begin{array}{lcl}
\Exp{\sum_{i\in\Sc_k}\norms{x_i^{k+1} - x_i^k}^2 \mid \Fc_{k-1}} & \geq & \frac{\alpha^2}{2(1+L^2\eta^2)(1+\gamma_1)}\sum_{i=1}^n\mathbf{p}_i\norms{\bar{x}^k - x_i^k}^2 \vspace{1ex}\\
&& - {~}  \frac{\alpha^2}{(1+L^2\eta^2)\gamma_1}\sum_{i=1}^n\mathbf{p}_i\big( \norms{e^{k+1}_i}^2 + \norms{e_i^k}^2\big)  \vspace{1ex}\\
&\geq &  \frac{\hat{\mbf{p}}\alpha^2}{2(1+L^2\eta^2)(1+\gamma_1)}\sum_{i=1}^n \norms{\bar{x}^k - x_i^k}^2 \vspace{1ex}\\
&& - {~}  \frac{\alpha^2}{(1+L^2\eta^2)\gamma_1}\sum_{i=1}^n\big( \norms{e^{k+1}_i}^2 + \norms{e_i^k}^2\big),
\end{array}
\end{equation}
where we have used  $\hat{\mbf{p}} := \min_{i\in [n]}\mbf{p}_i > 0$ in Assumption~\ref{ass:A3} and $\mathbf{p}_i \leq 1$ for all $i\in [n]$.

Taking expectation both sides of  \eqref{eq:dr_descent_lem0_inexact} w.r.t. $\Sc_k$ conditioned on $\Fc_{k-1}$, and letting $\gamma_3 := 1$, we get
\begin{equation}\label{eq:dr_descent_lem0_inexact11}
\hspace{-0ex}
\arraycolsep=0.2em
\begin{array}{ll}
& \Exp{V_\eta^{k+1}(\bar{x}^{k+1}) \mid \Fc_{k-1}}  \leq  V_\eta^{k}(\bar{x}^k)  +  \frac{(1+\eta^2L^2)}{\eta n}\sum_{i=1}^n\big[ (1+\mathbf{p}_i)\norms{e_i^k}^2 + \mathbf{p}_i\norms{e_i^{k+1}}^2)\big] \vspace{1ex}\\
&\qquad\qquad\qquad + {~}  \frac{2(1 + \eta L)^2}{\gamma_4\eta\alpha^2 n }\sum_{i=1}^n\mathbf{p}_i \big[ \norms{e_i^k}^2 + \norms{e_i^{k+1}}^2 \big]  \vspace{1ex}\\
&\qquad\qquad\qquad - {~} \frac{[ 2 - \alpha(L\eta + 1) - 2L^2\eta^2 - 4\alpha\gamma_4(1 + L^2\eta^2)]}{2\eta\alpha n}\Exp{\sum_{i\in\Sc_k} \norms{ x^{k+1}_i - x^k_i }^2 \mid \Fc_{k-1}}.
\end{array}
\hspace{-0ex}
\end{equation}
Here, we have used $E_{k+1}^2 \leq \frac{1}{n}\sum_{i=1}^n\norms{e_i^k}^2 + \frac{1}{n}\sum_{i\in\Sc_k} \big[\norms{e_i^k}^2 + \norms{e_i^{k+1}}^2 \big]$ and the fact that $\Exp{\sum_{i\in\Sc_k} \big[\norms{e_i^k}^2 + \norms{e_i^{k+1}}^2 \big] \mid \Fc_{k-1}} = \sum_{i=1}^n\mathbf{p}_i\big[\norms{e_i^k}^2 + \norms{e_i^{k+1}}^2 \big]$.
Combining \eqref{eq:proof102_est1} and \eqref{eq:dr_descent_lem0_inexact11} we obtain 
\begin{equation*} 
\hspace{-0ex}
\arraycolsep=0.2em
\begin{array}{lcl}
\Exp{V_\eta^{k+1}(\bar{x}^{k+1}) \mid \Fc_{k-1}} & \leq & V_\eta^{k}(\bar{x}^k)  +  \frac{(1+\eta^2L^2)}{\eta(n+1)}\sum_{i=1}^n \norms{e_i^k}^2 \vspace{1ex}\\
&& + {~} \Big[ \frac{2(1 + \eta L)^2}{\gamma_4\eta\alpha^2 n } + \frac{(1+\eta^2L^2)}{\eta n} \Big] \sum_{i=1}^n\mathbf{p}_i \big[ \norms{e_i^k}^2 + \norms{e_i^{k+1}}^2 \big]  \vspace{1ex}\\
&& + {~} \frac{\alpha [ 2 - \alpha(L\eta + 1) - 2L^2\eta^2 - 4\alpha\gamma_4(1 + L^2\eta^2)]}{2\eta(1 + L^2\eta^2)\gamma_1 n} \sum_{i=1}^n \big[ \norms{e_i^k}^2 + \norms{e_i^{k+1}}^2 \big]  \vspace{1ex}\\
&& - {~} \frac{\hat{\mbf{p}}\alpha [ 2 - \alpha(L\eta + 1) - 2L^2\eta^2 - 4\alpha\gamma_4(1 + L^2\eta^2)]}{4\eta(1 + L^2\eta^2)(1+\gamma_1) n}\sum_{i=1}^n \norms{\bar{x}^k - x^{k}_i}^2.
\end{array}
\hspace{-0ex}
\end{equation*}
Rearranging terms in the last inequality and using $\mathbf{p}_i \leq 1$ and $\norms{e_i^k}^2 \leq \epsilon_{i,k}^2$ for all $i\in [n]$ and $k\geq 0$ from \eqref{eq:exact_prox}, we obtain \eqref{eq:dr_descent_lem_inexact}.
Note that if $\epsilon_{i,k} = 0$ for all $i \in [n]$ and $k \geq 0$, then we allow to set $\gamma_1 = \gamma_2 = \gamma_4 = \rho_1 = \rho_2 = 0$ as a consequence of the last statement in Lemma~\ref{lem:xbar_xk}, Lemma~\ref{le:grad_norm_bound},  and Lemma~\ref{lem:dr_key_est0_inexact}.
\end{proof}

\subsection{Convergence rate and communication complexity of Algorithm~\ref{alg:A1} -- The inexact variant}\label{subsec:inexact_alg1}
\begin{proof}[\textbf{The proof of Theorem \ref{thm:feddr_convergence_inexact}}]
First, from \eqref{eq:dr_descent_lem_inexact}, we have
\begin{equation}\label{eq:dr_descent_lem_inexact_eq1}
\arraycolsep=0.1em
\begin{array}{lcl}
\frac{(1+\eta L)^2(1+\gamma_2)}{n\eta^2}{\displaystyle\sum_{i=1}^n} \norms{x_i^k - \bar{x}^k}^2 & \leq &   \frac{2(1+\eta L)^2(1+\gamma_2) }{\eta^2\beta} \left[ V_\eta^{k}(\bar{x}^k) - \Exp{V_{\eta}^{k+1} (\bar{x}^{k+1}) \mid\Fc_{k-1}} \right], \vspace{0ex}\\
&& + {~} \frac{2(1+\eta L)^2(1+\gamma_2)}{n\eta^2\beta} \displaystyle\sum_{i=1}^n(\rho_1 \epsilon_{i,k}^2 + \rho_2 \epsilon_{i,k+1}^2).
\end{array}
\end{equation}
Substituting these estimates into \eqref{eq:grad_mapping_bound} of Lemma~\ref{le:grad_norm_bound}, we have
\begin{equation*} 
\arraycolsep=0.2em
\begin{array}{lcl}
\Vert \Gc_{\eta}(\bar{x}^k) \Vert^2 & \leq &  \frac{2(1+\eta L)^2(1+\gamma_2) }{\eta^2\beta}  \left[ V_\eta^{k}(\bar{x}^k) - \Exp{V_{\eta}^{k+1}(\bar{x}^{k+1}) \mid \Fc_{k-1}} \right] \vspace{1ex}\\
&& + {~}   \frac{2(1+\eta L)^2(1+\gamma_2)}{n\eta^2\beta} \sum_{i=1}^n(\rho_1 \epsilon_{i,k}^2 + \rho_2 \epsilon_{i,k+1}^2) +   \frac{(1 + \eta L)^2(1+\gamma_2)}{n\eta^2\gamma_2}\sum_{i=1}^n\epsilon_{i,k}^2.
\end{array}
\end{equation*}
Let us introduce three constants 
\begin{equation*}
\arraycolsep=0.2em
\begin{array}{lcl}
C_1 := \frac{2(1+\eta L)^2(1+\gamma_2) }{\eta^2\beta}, \quad C_2 := \rho_1C_1, \quad\text{and}\quad  C_3 :=  \rho_2C_1 +  \frac{(1 + \eta L)^2(1+\gamma_2)}{\eta^2\gamma_2}.
\end{array}
\end{equation*}
Now, taking the total expectation of the last estimate w.r.t. $\Fc_k$ and using the definition of $C_i$ ($i=1,2,3$), we have
\begin{equation*}
\Exp{\norms{\Gc_{\eta}(\bar{x}^k)}^2 }  \leq C_1 \left( \Exp{V_\eta^{k}(\bar{x}^k)} - \Exp{V_{\eta}^{k+1}(\bar{x}^{k+1})} \right) + \frac{C_2}{n}\sum_{i=1}^n\epsilon_{i,k}^2 + \frac{C_3}{n}\sum_{i=1}^n\epsilon_{i,k+1}^2.
\end{equation*}
Summing up this inequality from $k :=0$ to $k := K$, and multiplying the result by $\frac{1}{K+1}$, we get
\begin{equation*}
\arraycolsep=0.2em 
\begin{array}{lcl}
\frac{1}{K+1}\sum_{k=0}^K \Exp{\norms{\Gc_{\eta}(\bar{x}^k)}^2} & \leq &  C_1 \left( \Exp{V_\eta^{0}(\bar{x}^0)} - \Exp{V_{\eta}^{K+1}(\bar{x}^{K+1})} \right)  \vspace{1ex}\\
&& + {~}  \frac{1}{n(K+1)}\sum_{k=0}^K\sum_{i=1}^n\big(C_2\epsilon_{i,k}^2 + C_3\epsilon_{i,k+1}^2 \big).
\end{array}
\end{equation*}
Furthermore, from the initial condition $x^0_i := x^0$ and $\bar{x}^0 := x^0$, we have $V^0_{\eta}(\bar{x}^0) = g(x^0) + \frac{1}{n}\sum_{i=1}^n f_i(x^0) = F(x^0)$. 
In addition, $\Exp{V_{\eta}^{K+1}(\bar{x}^{K+1})} \geq F^{\star}$ due to \eqref{eq:V_lowerbound}.  
Consequently, the last estimate becomes
\begin{equation*} 
\frac{1}{K+1}\sum_{k=0}^K \Exp{\norms{\Gc_{\eta}(\bar{x}^k)}^2}  \leq \frac{C_1}{K+1}\left[ F(x^0) - F^{\star} \right] +  \frac{1}{n(K+1)}\sum_{k=0}^K\sum_{i=1}^n\big(C_2\epsilon_{i,k}^2 + C_3\epsilon_{i,k+1}^2 \big),
\end{equation*}
which proves \eqref{eq:dr_thm1_convergence_inexact}.

Finally, let $\tilde{x}^K$ be selected uniformly at random from $\{\bar{x}^0,\cdots,\bar{x}^K\}$ as the output of Algorithm~\ref{alg:A1}.
Then, from \eqref{eq:dr_thm1_convergence_inexact} and $\frac{1}{n}\sum_{i=1}^n\sum_{k=0}^{K+1}\epsilon_{i,k}^2 \leq M$ for all $K\geq 0$, we have
\begin{equation*}
\Exp{\norms{\Gc_{\eta}(\tilde{x}^K)}^2} = \frac{1}{K+1}\sum_{k=0}^K \Exp{\norms{ \Gc_{\eta}(\bar{x}^k)}^2} \le \frac{C_1\left[ F(x^0) - F^{\star}\right] + (C_2 + C_3)M}{K+1}.
\end{equation*}
Consequently, to guarantee $\Exp{\norms{\Gc_{\eta}(\tilde{x}^K)}^2} \leq \varepsilon^2$, from the last estimate we need to choose $K$ such that $\frac{C_1[ F(x^0) - F^{\star}] +  (C_2 + C_3)M }{K+1} \leq \varepsilon^2$. 
This condition leads to
\begin{equation*}
K + 1 \geq \frac{C_1[ F(x^0) - F^{\star}] + (C_2 + C_3)M}{\varepsilon^2}.
\end{equation*}
Hence, we can take $K := \left\lfloor \frac{C_1[ F(x^0) - F^{\star}] + (C_2 + C_3)M}{\varepsilon^2} \right\rfloor \equiv \BigO{\frac{1}{\varepsilon^2}}$ as its lower bound.
\end{proof}

\subsection{Convergence of Algorithm~\ref{alg:A1} when $\mbf{p}_i = \frac{1}{n}, i \in [n]$  -- The exact variant}\label{subsec:proof_cor1}
\begin{proof}[\textbf{The proof of Corollary~\ref{cor:special_case}}]
Under the exact variant, we can verify that the choice $\alpha = 1$ and $\eta = \frac{1}{3L}$ satisfies \eqref{eq:stepsizes_choice}. 
As a result, using $\hat{\mbf{p}}=\frac{1}{n}$, from \eqref{eq:beta_consts} we can exactly calculate $\beta = \frac{3L}{5n}$, while $\rho_1 = \rho_2 = 0$. 
Consequently, \eqref{eq:dr_descent_lem_inexact_eq1} leads to
\begin{equation*} 
\frac{(1+\eta L)^2}{n\eta^2}{\displaystyle\sum_{i=1}^n} \norms{x_i^k - \bar{x}^k}^2  \leq  \frac{2(1+\eta L)^2 }{\eta^2\beta} \left[ V_\eta^{k}(\bar{x}^k) - \Exp{V_{\eta}^{k+1} (\bar{x}^{k+1}) \mid\Fc_{k-1}} \right].
\end{equation*}
Alternatively, using Lemma~\ref{le:grad_norm_bound}, we have
\begin{equation*}
\Vert \Gc_{\eta}(\bar{x}^k) \Vert^2 \leq  \frac{(1 + \eta L)^2}{n\eta^2}  \sum_{i=1}^n \norms{x_i^k  - \bar{x}^k}^2.
\end{equation*}
Combining the last two inequalities, we obtain
\begin{equation*}
\arraycolsep=0.3em
\begin{array}{lcl}
\Vert \Gc_{\eta}(\bar{x}^k) \Vert^2  &\leq &   \frac{2(1+\eta L)^2 }{\eta^2\beta} \left[ V_\eta^{k}(\bar{x}^k) - \Exp{V_{\eta}^{k+1} (\bar{x}^{k+1}) \mid\Fc_{k-1}} \right] \vspace{1ex}\\
&= & \frac{160Ln}{3}\left[ V_\eta^{k}(\bar{x}^k) - \Exp{V_{\eta}^{k+1} (\bar{x}^{k+1}) \mid\Fc_{k-1}} \right].
\end{array}
\end{equation*}
Now, taking the total expectation of the last estimate w.r.t. $\Fc_k$, we have
\begin{equation*}
\Exp{\Vert \Gc_{\eta}(\bar{x}^k) \Vert^2}  \leq  \frac{160Ln}{3}\left( \Exp{V_\eta^{k}(\bar{x}^k)} - \Exp{V_{\eta}^{k+1} (\bar{x}^{k+1})} \right).
\end{equation*}
Summing this inequality from $k=0$ to $k = K$, and then multiplying the result by $\frac{1}{K+1}$, we obtain
\begin{equation}\label{eq:cor_special_case_eq4}
\frac{1}{K+1}\sum_{k=0}^K \Exp{\Vert \Gc_{\eta}(\bar{x}^k) \Vert^2}  \leq  \frac{160Ln}{3(K+1)}\left( \Exp{V_\eta^{k}(\bar{x}^0)} - \Exp{V_{\eta}^{k+1} (\bar{x}^{K+1})} \right).
\end{equation}
Recall that from the initial condition $x^0_i := x^0$ and $\bar{x}^0 := x^0$, we have $V^0_{\eta}(\bar{x}^0) = g(x^0) + \frac{1}{n}\sum_{i=1}^n f_i(x^0) = F(x^0)$. 
In addition, $\Exp{V_{\eta}^{K+1}(\bar{x}^{K+1})} \geq F^{\star}$ due to \eqref{eq:V_lowerbound}.  
As a result, \eqref{eq:cor_special_case_eq4} can be further upper bounded as
\begin{equation*}
\frac{1}{K+1}\sum_{k=0}^K \Exp{\Vert \Gc_{\eta}(\bar{x}^k) \Vert^2}  \leq  \frac{160Ln}{3(K+1)}\left( F(x^0) - F^{\star} \right),
\end{equation*}
which exactly proves \eqref{eq:dr_cor1_convergence_exact}.

Finally, if $\tilde{x}^K$ is selected uniformly at random from $\{\bar{x}^0,\cdots,\bar{x}^K\}$ as the output of Algorithm~\ref{alg:A1}, then we have
\begin{equation*}
\Exp{\norms{\Gc_{\eta}(\tilde{x}^K)}^2} = \frac{1}{K+1}\sum_{k=0}^K \Exp{\norms{ \Gc_{\eta}(\bar{x}^k)}^2} \le \frac{160Ln}{3(K+1)}\left( F(x^0) - F^{\star} \right).
\end{equation*}
Consequently, to guarantee $\Exp{\norms{\Gc_{\eta}(\tilde{x}^K)}^2} \leq \varepsilon^2$, from the last estimate we need to choose $K$ such that $\frac{160Ln}{3(K+1)}\left( F(x^0) - F^{\star} \right) \leq \varepsilon^2$. 
This condition leads to
\begin{equation*}
K + 1 \geq \frac{160Ln[ F(x^0) - F^{\star}]}{3\varepsilon^2}.
\end{equation*}
Hence, we can take $K := \left\lfloor \frac{160Ln[ F(x^0) - F^{\star}]}{3\varepsilon^2} \right\rfloor \equiv \BigO{\frac{1}{\varepsilon^2}}$ as its lower bound.
\end{proof}

\subsection{Convergence of Algorithm~\ref{alg:A1} under relative accuracies}\label{subsec:warmstart}
As suggested by a reviewer, we provide here an analysis of Algorithm~\ref{alg:A1}, when relative accuracies are used to evaluate $\prox_{\eta f_i}$.
Such a strategy has been widely used in the literature, including \cite{liu2021acceleration,Rockafellar1976b}.
Let us adopt this concept from  \cite[Definition 3.3]{liu2021acceleration} to our context as follows:
\begin{definition}\label{def:rel_error}
For any $i \in \Sc_k$, given $x^k_i$ and $y^{k+1}_i$, we say that $x_i^{k+1}$ approximates $\prox_{\eta f_i}(y^{k+1}_i)$ up to a \textbf{bounded relative error} if there is a constant $\theta_i > 0$ (independent of $k$) such that
\begin{equation}\label{eq:rel_error}
\norms{x^{k+1}_i - \prox_{\eta f_i}(y^{k+1}_i)}^2 \leq \varepsilon^2_{i,k+1} := \theta_i \norms{x^{k+1}_i - x^k_i}^2
\end{equation}
\end{definition}

The following theorem states convergence of Algorithm~\ref{alg:A1} under the bounded relative error \eqref{eq:rel_error}.
\begin{theorem}\label{thm:feddr_rel_error_bound_convergence}
Suppose that Assumptions~\ref{ass:A1}, \ref{ass:A2}, and \ref{ass:A3} hold, and the bounded relative error condition \eqref{eq:rel_error} in Definition~\ref{def:rel_error} holds with $\theta_i := \hat{\theta} \mbf{p}_i$ for a fixed constant $\hat{\theta} > 0$.
Let $\{(x^{k}_i,y^{k}_i, \hat{x}^{k}_i,\bar{x}^{k})\}$ be generated by Algorithm~\ref{alg:A1} using a relaxation stepsize $\alpha = 1$ and $x_i^0 := \prox_{\eta f_i}(y_i^0)$ for $i \in [n]$. 
If $\gamma_4$ and $\hat{\theta}$ are chosen such that $1 - 4\gamma_4 - 8\hat{C}\hat{\theta} > 0$ and $\eta$ is chosen by
\begin{equation}\label{eq:thm_rel_error_eta_cond}
0 < \eta < \bar{\eta} := \tfrac{\sqrt{1 + 8(1+2\gamma_4)(1 - 4\gamma_4 - 8\hat{C}\hat{\theta})} - 1}{4L(1+2\gamma_4)},
\end{equation}
where $\hat{C} := \max\left\{1+\eta^2L^2, \frac{2(1 + \eta L)^2}{\gamma_4 } \right\}$, then the following bound holds
\begin{equation}\label{eq:thm_rel_error}
\frac{1}{K+1}\sum_{k=0}^K \Exp{\Vert \Gc_{\eta}(\bar{x}^k) \Vert^2} \leq \frac{\widetilde{C}\left[ F(x^0) - F^{\star} \right]}{(K+1)},
\end{equation}
where $\widetilde{C} > 0$ is computed by
\begin{equation}\label{eq:c_under_def}
\widetilde{C} := \frac{\hat{\mbf{p}}^2 \eta[ 1 - L\eta - 2L^2\eta^2 - 4\gamma_4(1 + L^2\eta^2) - 8\hat{C}\hat{\theta}]}{4\left[ 4(1+L^2\eta^2 + 2\hat{\theta}) + \hat{\mbf{p}} \hat{\theta}\right](1 + \eta L)^2}.
\end{equation}
The remaining conclusions of this theorem are similar to Theorem~\ref{thm:feddr_convergence_inexact}, and we omit  them here.
\end{theorem}

\begin{proof}
Firstly, starting from \eqref{eq:dr_descent_lem0_inexact}, using $\alpha = 1$, choosing $\gamma_3 = 1$, and noting that $E_{k+1}^2 := \frac{1}{n}\sum_{i\notin\Sc_k}\norms{e_i^k}^2 + \frac{1}{n}\sum_{i\in\Sc_k}\norms{e_i^{k+1}}^2$, we have
\begin{equation*}
\arraycolsep=0.2em
\begin{array}{lcl}
V_\eta^{k+1}(\bar{x}^{k+1}) & \leq & V_\eta^{k}(\bar{x}^k) - \frac{[ 1 - L\eta - 2L^2\eta^2 - 4\gamma_4(1 + L^2\eta^2)]}{2\eta n}\sum_{i\in\Sc_k} \norms{x^{k+1}_i - x^{k}_i}^2 \vspace{1ex}\\
&& + \frac{(1+\eta^2L^2)}{\eta n} \left(\sum_{i\notin\Sc_k}\norms{e_i^k}^2 + \sum_{i\in\Sc_k}\norms{e_i^{k+1}}^2\right) \vspace{1ex}\\
&& + {~}   \frac{2(1 + \eta L)^2}{\gamma_4\eta n }\sum_{i\in\Sc_k}[\norms{e_i^k}^2 + \norms{e_i^{k+1}}^2].
\end{array}
\end{equation*}
If we define $\hat{C} := \max\left\{1+\eta^2L^2, \frac{2(1 + \eta L)^2}{\gamma_4 } \right\}$, then we can further upper bound this estimate as
\begin{equation*}
\arraycolsep=0.3em
\begin{array}{lcl}
V_\eta^{k+1}(\bar{x}^{k+1}) & \leq & V_\eta^{k}(\bar{x}^k) - \frac{[ 1 - L\eta - 2L^2\eta^2 - 4\gamma_4(1 + L^2\eta^2)]}{2\eta n}\sum_{i\in\Sc_k} \norms{x^{k+1}_i - x^{k}_i}^2 \vspace{1ex}\\
&& + {~} \frac{\hat{C}}{ n\eta} \left(\sum_{i\notin\Sc_k}\norms{e_i^k}^2 + \sum_{i\in\Sc_k}\norms{e_i^{k+1}}^2\right)  +  \frac{\hat{C}}{ n\eta }\sum_{i\in\Sc_k}[\norms{e_i^k}^2 + \norms{e_i^{k+1}}^2] \vspace{1ex}\\
& = & V_\eta^{k}(\bar{x}^k) - \frac{[ 1 - L\eta - 2L^2\eta^2 - 4\gamma_4(1 + L^2\eta^2)]}{2\eta n}\sum_{i\in\Sc_k} \norms{x^{k+1}_i - x^{k}_i}^2 \vspace{1ex}\\
&& + {~}  \frac{\hat{C}}{ n \eta} \left(\sum_{i=1}^n \norms{e_i^k}^2 + 2\sum_{i\in\Sc_k}\norms{e_i^{k+1}}^2\right) \vspace{1ex}\\
& \leq & V_\eta^{k}(\bar{x}^k) - \frac{[ 1 - L\eta - 2L^2\eta^2 - 4\gamma_4(1 + L^2\eta^2)]}{2\eta n}\sum_{i\in\Sc_k} \norms{x^{k+1}_i - x^{k}_i}^2 \vspace{1ex}\\
&& + {~} \frac{\hat{C}}{ n \eta} \left(\sum_{i=1}^n \norms{e_i^k}^2 + 2\sum_{i=1}^n\norms{e_i^{k+1}}^2\right) \vspace{1ex}\\
& \leq & V_\eta^{k}(\bar{x}^k) - \frac{[ 1 - L\eta - 2L^2\eta^2 - 4\gamma_4(1 + L^2\eta^2)]}{2\eta n}\sum_{i\in\Sc_k} \norms{x^{k+1}_i - x^{k}_i}^2 \vspace{1ex}\\
&& + {~} \frac{2\hat{C}}{ n \eta}\sum_{i=1}^n \left( \norms{e_i^k}^2 +\norms{e_i^{k+1}}^2\right) .
\end{array}
\end{equation*}
Rearranging terms and noting that $\Exp{\sum_{i\in\Sc_k} \norms{x^{k+1}_i - x^k_i}^2 \mid \Fc_{k-1}} = \sum_{i=1}^n \mbf{p}_i  \norms{x^{k+1}_i - x^k_i}^2 $, we obtain from the last estimate that
\begin{equation*} 
\arraycolsep=0.3em
\begin{array}{lcl}
\frac{[ 1 - L\eta - 2L^2\eta^2 - 4\gamma_4(1 + L^2\eta^2)]}{2\eta n}\sum_{i=1}^n \mbf{p}_i\norms{x^{k+1}_i - x^{k}_i}^2  & \leq & V_\eta^{k}(\bar{x}^k) - V_\eta^{k+1}(\bar{x}^{k+1}) \vspace{1ex}\\
&& + {~} \frac{2\hat{C}}{ n \eta}\sum_{i=1}^n \left( \norms{e_i^k}^2 +\norms{e_i^{k+1}}^2\right) .
\end{array}
\end{equation*}
Now, taking the total expectation of the last inequality w.r.t. $\Fc_k$, we have
\begin{equation*} 
\hspace{-0ex}
\arraycolsep=0.2em
\begin{array}{ll}
\frac{[ 1 - L\eta - 2L^2\eta^2 - 4\gamma_4(1 + L^2\eta^2)]}{2\eta n} & \sum_{i=1}^n \mbf{p}_i \Exp{\norms{x^{k+1}_i - x^{k}_i}^2}  \vspace{1ex}\\
& \leq  \Exp{V_\eta^{k}(\bar{x}^k)}   - \Exp{V_\eta^{k+1}(\bar{x}^{k+1})} + \frac{2\hat{C}}{n \eta} \sum_{i=1}^n\Exp{\norms{e_i^k}^2 + \norms{e_i^{k+1}}^2}.
\end{array}
\end{equation*}
Summing this inequality from $k=0$ to $k = K$, we get
\begin{equation*} 
\hspace{-0ex}
\arraycolsep=0.2em
\begin{array}{ll}
\frac{[ 1 - L\eta - 2L^2\eta^2 - 4\gamma_4(1 + L^2\eta^2)]}{2\eta n}  & \sum_{k=0}^K\sum_{i=1}^n  \mbf{p}_i\Exp{\norms{x^{k+1}_i - x^k_i}^2}
 \leq   \Exp{V_\eta^{0}(\bar{x}^0)}   - \Exp{V_\eta^{K+1}(\bar{x}^{K+1})}  \vspace{1ex}\\
& + {~} \frac{2\hat{C}}{n \eta}\sum_{k=0}^K \sum_{i=1}^n\Exp{\norms{e_i^k}^2 + \norms{e_i^{k+1}}^2}.
\end{array}
\end{equation*}
If we choose $\varepsilon_{i,0} = 0$ for $i\in[n]$, then the last estimate reduces to
\begin{equation}\label{eq:rel_error_eq5}
\hspace{-0ex}
\arraycolsep=0.2em
\begin{array}{ll}
\frac{[ 1 - L\eta - 2L^2\eta^2 - 4\gamma_4(1 + L^2\eta^2)]}{2\eta n} & \displaystyle\sum_{k=0}^K\sum_{i=1}^n  \mbf{p}_i\Exp{\norms{x^{k+1}_i - x^k_i}^2}   \leq  \Exp{V_\eta^{0}(\bar{x}^0)}   - \Exp{V_\eta^{K+1}(\bar{x}^{K+1})}  \vspace{0ex}\\
& + {~} \frac{4\hat{C}}{n \eta} \displaystyle \sum_{k=0}^K \sum_{i=1}^n\Exp{\norms{e_i^{k+1}}^2}.
\end{array}
\hspace{-4ex}
\end{equation}
From \eqref{eq:rel_error} in Definition~\ref{def:rel_error}, we have $\norms{e_i^{k+1}}^2 = \norms{x^{k+1}_i - \prox_{\eta f_i}(y^{k+1}_i)}^2 \leq \varepsilon^2_{i,{k+1}} := \theta_i \norms{x^{k+1}_i - x^{k}_i}^2$. 
Using this condition in \eqref{eq:rel_error_eq5}, we have
\begin{equation*} 
\hspace{-0ex}
\arraycolsep=0.2em
\begin{array}{lcl}
\frac{[ 1 - L\eta - 2L^2\eta^2 - 4\gamma_4(1 + L^2\eta^2)]}{2\eta n}  \displaystyle \sum_{k=0}^K\sum_{i=1}^n  \mbf{p}_i\Exp{\norms{x^{k+1}_i - x^k_i}^2}
& \leq & \Exp{V_\eta^{0}(\bar{x}^0)}   - \Exp{V_\eta^{K+1}(\bar{x}^{K+1})} \vspace{0ex}\\
&& + {~} \frac{4\hat{C}}{n \eta}  \displaystyle \sum_{k=0}^K \sum_{i=1}^n \theta_i\Exp{\norms{x^{k+1}_i - x^k_i}^2}.
\end{array}
\end{equation*}
Now, we can choose $\theta_i$ such that $\theta_i = \hat{\theta} \mbf{p}_i$ for given $\hat{\theta} > 0$. 
Plugging this choice of $\theta_i$  into the last estimate, we have
\begin{equation*}
\hspace{-0ex}
\arraycolsep=0.2em
\begin{array}{lcl}
\frac{[ 1 - L\eta - 2L^2\eta^2 - 4\gamma_4(1 + L^2\eta^2)]}{2\eta n}  \displaystyle \sum_{k=0}^K\sum_{i=1}^n  \mbf{p}_i\Exp{\norms{x^{k+1}_i - x^k_i}^2} & \leq & \Exp{V_\eta^{0}(\bar{x}^0)}   - \Exp{V_\eta^{K+1}(\bar{x}^{K+1})} \vspace{0ex}\\
&&  + {~} \frac{4\hat{C}\hat{\theta}}{n \eta}  \displaystyle \sum_{k=0}^K \sum_{i=1}^n \mbf{p}_i\Exp{\norms{x^{k+1}_i - x^k_i}^2}.
\end{array}
\end{equation*}
Rearranging terms in the above estimate, we arrive at
\begin{equation*} 
\hspace{-0ex}
\arraycolsep=0.2em
\begin{array}{lcl}
\frac{[ 1 - L\eta - 2L^2\eta^2 - 4\gamma_4(1 + L^2\eta^2) - 8\hat{C}\hat{\theta}]}{2\eta n}  \displaystyle \sum_{k=0}^K\sum_{i=1}^n  \mbf{p}_i\Exp{\norms{x^{k+1}_i - x^k_i}^2}  \leq \Exp{V_\eta^{0}(\bar{x}^0)}   - \Exp{V_\eta^{K+1}(\bar{x}^{K+1})}.
\end{array}
\end{equation*}
From the initial condition $x^0_i := x^0$ and $\bar{x}^0 := x^0$, we have $V^0_{\eta}(\bar{x}^0) = g(x^0) + \frac{1}{n}\sum_{i=1}^n f_i(x^0) = F(x^0)$. 
In addition, $\Exp{V_{\eta}^{K+1}(\bar{x}^{K+1})} \geq F^{\star}$ due to \eqref{eq:V_lowerbound}. 
Using these conditions, the last estimate can be further upper bounded by
\begin{equation}\label{eq:rel_error_eq8}
\frac{\hat{\mbf{p}}[ 1 - L\eta - 2L^2\eta^2 - 4\gamma_4(1 + L^2\eta^2) - 8\hat{C}\hat{\theta}]}{2\eta n}\sum_{k=0}^K\sum_{i=1}^n  \Exp{\norms{x^{k+1}_i - x^k_i}^2} \leq  F(x^0) - F^{\star},
\end{equation}
where we have used $\mbf{p}_i \ge \hat{\mbf{p}}$ for all $i\in [n]$.

Now, we need to choose $\eta$ and $\hat{\theta}$ such that $1 - L\eta - 2L^2\eta^2 - 4\gamma_4(1 + L^2\eta^2) - 8\hat{C}\hat{\theta} > 0$.
First, we need to choose $\gamma_4 > 0$ and $\hat{\theta} > 0$ such that $1 - 4\gamma_4 - 8\hat{C}\hat{\theta} > 0$.
Then, the condition for $\eta$ is
\begin{equation*}
\arraycolsep=0.2em
\begin{array}{lcl}
0 < \eta < \bar{\eta} := \frac{\sqrt{1 + 8(1+2\gamma_4)(1 - 4\gamma_4 - 8\hat{C}\hat{\theta})} - 1}{4L(1+2\gamma_4)}.
\end{array}
\end{equation*}
Next, we connect the term $\norms{x^{k+1}_i - x^k_i}^2$ with $\Vert \Gc_{\eta}(\bar{x}^k) \Vert^2 $ as follows. From \eqref{eq:xbar_xk_relation} with $\alpha = 1$ and $\gamma_1 = 1$, we have
\begin{equation*}
\frac{1}{4(1+L^2\eta^2)}\sum_{i\in\Sc_k}\norms{\bar{x}^k - x_i^k}^2 \leq \sum_{i\in\Sc_k}\Big[ \norms{x_i^{k+1} - x_i^k}^2 +  \frac{1}{(1+L^2\eta^2)}\big( \norms{e^{k+1}_i}^2 + \norms{e_i^k}^2\big)\Big].
\end{equation*}
Taking expecatation w.r.t. $\Sc_k$ given $\Fc_{k-1}$, and then taking full expectation, we obtain
\begin{equation*}
\begin{array}{lcl}
\frac{1}{4(1+L^2\eta^2)}\sum_{i=1}^n \mbf{p}_i \Exp{\norms{\bar{x}^k - x_i^k}^2}  &\leq& \sum_{i=1}^n \mbf{p}_i \Exp{ \norms{x_i^{k+1} - x_i^k}^2}  \vspace{1ex}\\
&& + {~} \frac{1}{(1+L^2\eta^2)}\sum_{i=1}^n \mbf{p}_i\Exp{ \norms{e^{k+1}_i}^2 + \norms{e_i^k}^2} \vspace{1ex}\\
&\leq& \sum_{i=1}^n \mbf{p}_i \Exp{ \norms{x_i^{k+1} - x_i^k}^2} \vspace{1ex}\\
&& + {~} \frac{1}{(1+L^2\eta^2)}\sum_{i=1}^n \Exp{ \norms{e^{k+1}_i}^2 + \norms{e_i^k}^2}.
\end{array}
\end{equation*}
Summing this inequality from $k=0$ to $k = K$, we get
\begin{equation*}
\arraycolsep=0.3em
\begin{array}{lcl}
\frac{1}{4(1+L^2\eta^2)}\sum_{k=0}^K\sum_{i=1}^n \mbf{p}_i \Exp{\norms{\bar{x}^k - x_i^k}^2}  &\leq& \sum_{k=0}^K\sum_{i=1}^n \mbf{p}_i \Exp{ \norms{x_i^{k+1} - x_i^k}^2} \vspace{1ex}\\
&& + {~}  \frac{1}{(1+L^2\eta^2)}\sum_{k=0}^K\sum_{i=1}^n \Exp{ \norms{e^{k+1}_i}^2 + \norms{e_i^k}^2}.
\end{array}
\end{equation*}
Using the condition that $\epsilon_{i,0} = 0$, similar to \eqref{eq:rel_error_eq5}, we have
\begin{equation*} 
\begin{array}{lcl}
\frac{1}{4(1+L^2\eta^2)}\sum_{k=0}^K\sum_{i=1}^n \mbf{p}_i \Exp{\norms{\bar{x}^k - x_i^k}^2}  &\leq& \sum_{k=0}^K\sum_{i=1}^n \mbf{p}_i \Exp{ \norms{x_i^{k+1} - x_i^k}^2} \vspace{1ex}\\
&& + {~}  \frac{2}{(1+L^2\eta^2)}\sum_{k=0}^K\sum_{i=1}^n \Exp{ \norms{e^{k+1}_i}^2}
\vspace{1ex}\\
&\leq& \sum_{k=0}^K\sum_{i=1}^n \mbf{p}_i \Exp{ \norms{x_i^{k+1} - x_i^k}^2} \vspace{1ex}\\
&& + {~}  \frac{2}{(1+L^2\eta^2)}\sum_{k=0}^K\sum_{i=1}^n \theta_i \Exp{ \norms{x^{k+1}_i - x^k_i}^2}\vspace{1ex}\\
&\leq& \frac{1+L^2\eta^2 + 2 \hat{\theta}}{(1+L^2\eta^2)}\sum_{k=0}^K\sum_{i=1}^n \mbf{p}_i \Exp{ \norms{x_i^{k+1} - x_i^k}^2} .
\end{array}
\end{equation*}
In fact, we can further bound this estimate as
\begin{equation*}
\frac{\hat{\mbf{p}}}{4(1+L^2\eta^2)}\sum_{k=0}^K\sum_{i=1}^n  \Exp{\norms{\bar{x}^k - x_i^k}^2} \leq \frac{1+L^2\eta^2 + 2 \hat{\theta}}{(1+L^2\eta^2)}\sum_{k=0}^K\sum_{i=1}^n \Exp{ \norms{x_i^{k+1} - x_i^k}^2} ,
\end{equation*}
where we have used $\hat{\mbf{p}}\le \mbf{p}_i \le 1$.
Next, multiply both sides of this inequality by $\frac{8(1+L^2\eta^2)(1 + \eta L)^2}{\hat{\mbf{p}}\eta^2 n}$, we obtain
\begin{equation}\label{eq:rel_error_eq11}
 \hspace{-1ex}
\frac{2(1 + \eta L)^2}{n\eta^2} \sum_{k=0}^K\sum_{i=1}^n \Exp{\norms{\bar{x}^k - x_i^k}^2} 
 \leq  \frac{8(1+L^2\eta^2 + 2\hat{\theta})(1 + \eta L)^2}{ \hat{\mbf{p}}\eta^2 n}\sum_{k=0}^K\sum_{i=1}^n \Exp{ \norms{x_i^{k+1} {\!\!\!} - x_i^k}^2}.
 \hspace{-3ex}
\end{equation}
Furthermore, from \eqref{eq:grad_mapping_bound}, choosing $\gamma_2 = 1$ and  summing the result from $k=0$ to $k= K$, we get
\begin{equation}\label{eq:rel_error_eq12}
\arraycolsep=0.3em
\begin{array}{lcl}
\sum_{k=0}^K \Exp{\Vert \Gc_{\eta}(\bar{x}^k) \Vert^2}  &\leq&  \frac{2(1 + \eta L)^2}{n\eta^2} \sum_{k=0}^K\sum_{i=1}^n  \Exp{\norms{x_i^k  - \bar{x}^k}^2} \\
&&+ {~} \frac{2(1 + \eta L)^2}{n \eta^2}\sum_{k=0}^K \sum_{i=1}^n \Exp{\norms{e_i^k}^2} \vspace{1ex}\\
&\leq&  \frac{2(1 + \eta L)^2}{n\eta^2} \sum_{k=0}^K\sum_{i=1}^n  \Exp{\norms{x_i^k  - \bar{x}^k}^2} \\
&&+ {~} \frac{2(1 + \eta L)^2}{n \eta^2}\sum_{k=0}^K \sum_{i=1}^n \theta_i \Exp{\norms{x^{k+1}_i - x^k_i}^2} \vspace{1ex}\\
&=&  \frac{2(1 + \eta L)^2}{n\eta^2} \sum_{k=0}^K\sum_{i=1}^n  \Exp{\norms{x_i^k  - \bar{x}^k}^2} \\
&&+ {~} \frac{2(1 + \eta L)^2\hat{\theta}}{n \eta^2 }\sum_{k=0}^K \sum_{i=1}^n \mbf{p}_i \Exp{\norms{x^{k+1}_i - x^k_i}^2}\vspace{1ex}\\
\end{array}
\end{equation}
where the last equality comes from the fact that $\theta_i = \hat{\theta} \mbf{p}_i$.

Now, plugging \eqref{eq:rel_error_eq11} into \eqref{eq:rel_error_eq12} and using $\mbf{p}_i \le 1$, we can get
\begin{equation}\label{eq:rel_error_eq13}
 \hspace{0ex}
\arraycolsep=0.2em
\begin{array}{lcl}
\sum_{k=0}^K \Exp{\Vert \Gc_{\eta}(\bar{x}^k) \Vert^2} &\leq& \left[ \frac{8\left[1+L^2\eta^2 + 2 \hat{\theta}\right](1 + \eta L)^2}{ \hat{\mbf{p}} \eta^2 n} + \frac{2(1 + \eta L)^2 \hat{\theta}}{n\eta^2}\right] \sum_{k=0}^K \sum_{i=1}^n  \Exp{\norms{x^{k+1}_i - x^k_i}^2} \vspace{1ex}\\
&=& \frac{2\left[ 4(1+L^2\eta^2 + 2\hat{\theta}) + \hat{\mbf{p}} \hat{\theta}\right](1 + \eta L)^2}{ \hat{\mbf{p}} n \eta^2} \sum_{k=0}^K \sum_{i=1}^n  \Exp{\norms{x^{k+1}_i - x^k_i}^2} .
\end{array}
 \hspace{-3ex}
\end{equation}
From the definition of $\widetilde{C}$ in \eqref{eq:c_under_def}, we can verify that
\begin{equation*}
\frac{\hat{\mbf{p}}[ 1 - L\eta - 2L^2\eta^2 - 4\gamma_4(1 + L^2\eta^2) - 8\hat{C}\hat{\theta}]}{2\eta n \widetilde{C}} = \frac{2\left[ 4(1+L^2\eta^2 + 2\hat{\theta}) + \hat{\mbf{p}} \hat{\theta}\right](1 + \eta L)^2}{ \hat{\mbf{p}} n \eta^2}.
\end{equation*}
Next, multiplying both sides of \eqref{eq:rel_error_eq8} by $\frac{1}{\widetilde{C}}$, and then using \eqref{eq:rel_error_eq13}, we obtain
\begin{equation*} 
\begin{array}{lcl}
\sum_{k=0}^K \Exp{\Vert \Gc_{\eta}(\bar{x}^k) \Vert^2} &\leq& \frac{2\left[ 4(1+L^2\eta^2 + \hat{\theta}) + \hat{\mbf{p}} \hat{\theta}\right](1 + \eta L)^2}{ \hat{\mbf{p}} n \eta^2} \sum_{k=0}^K \sum_{i=1}^n  \Exp{\norms{x^{k+1}_i - x^k_i}^2} \vspace{1ex}\\
&\overset{\eqref{eq:rel_error_eq8}}{\leq}&  \widetilde{C}\left[ F(x^0) - F^{\star} \right].
\end{array}
\end{equation*}
Finally, multiplying both sides of this inequality by $\frac{1}{K+1}$, we obtain \eqref{eq:thm_rel_error}.
\end{proof}

\section{Analysis of Algorithm~\ref{alg:A2}: The Asynchronous Variant --- \textbf{asyncFedDR}}\label{apdx:sec:asynFedDR}
This section provides the full proof of  Lemma~\ref{lem:asdr_key_est} and Theorem~\ref{thm:asdr_convergence} in the main text.
However, let us first discuss an asynchronous implementation of Algorithm~\ref{alg:A2} and present the full description of our probabilistic models based on \cite{cannelli2019asynchronous} used in Section~\ref{sec:asynFedDR}.

\subsection{Asynchronous implementation: Dual-memory approach}
Let us provide more details on the implementation of our asynchronous algorithm. When a user finishes its local update, the updated model (or model difference) is sent to the server for a proximal aggregation step. 
When the server is performing a proximal aggregation step, other users might need to read from the global model. 
To allow concurrent read/write operations, one easy method is to have two models stored on the server, denoted as model 1 and model 2. 
At any given time, one model is on ``read'' state (it is supposed to be read from) and the other will be on ``write'' state (it will be written on when the server finishes aggregation). 
Suppose model 1 is on a ``read'' state and model 2 is on a ``write'' state, then all users can read from model 1. 
When the server completes the proximal aggregation, model 2 becomes the latest model and it will change to a ``read'' state while model 1 is on a ``write'' state. 
This implementation detail is also discussed in \cite{peng2016arock}, which is termed by a \textit{dual-memory approach}.

\subsection{Probabilistic model}\label{apdx:subsec:prob_model}
Let $\xi^k := (i_k, d^k)$ be a realization of a joint random vector $\hat{\xi}^k := (\hat{i}_k, \hat{d}^k)$ of the user index $\hat{i}_k \in [n]$ and the delay vector $\hat{d}^k = (\hat{d}^k_1,\cdots, \hat{d}^k_n) \in \Dc := \set{0,1,\cdots, \tau}^n$ presented at the current iteration $k$.
We consider $k+1$ random vectors $\hat{\xi}^l$ $(0\leq l\leq k)$ that form a concatenate  random vector $\hat{\xi}^{0:k} := (\hat{\xi}^0, \cdots, \hat{\xi}^k)$.
We also use $\xi^{0:k} = (\xi^0, \xi^1, \cdots, \xi^k)$ for $k+1$ possible values of the random vector $\hat{\xi}^{0:k}$.
Let $\Omega$ be the sample space of all sequences $\omega := \{(i_k, d^k)\}_{k\geq 0} \equiv \sets{\xi^k}_{k\geq 0}$.
We define a cylinder $\Cc_k(\xi^{0:k}) := \sets{\omega \in\Omega : (\omega_0, \cdots, \omega_k) = \xi^{0:k}}$ as a subset in $\Omega$ and $\Cc_k$ is the set of all possible subsets $\Cc_k(\xi^{0:k})$ when $\xi^t$, $t=0,\cdots, k$, take all possible values, where $\omega_l$ is the $l$-th coordinate of $\omega$.
Note that $\set{\Cc_k}_{k\geq 0}$ forms a partition of $\Omega$ and measurable.
Let $\Fc_k := \sigma(\Cc_k)$ be the $\sigma$-algebra generated by $\Cc_k$ and $\Fc := \sigma(\cup_{k=0}^{\infty}\Cc_k)$.
Clearly, $\set{\Fc_k}_{k\geq 0}$ forms a filtration such that $\Fc_k \subseteq \Fc_{k+1} \subseteq \cdots \subseteq \Fc$ for $k \geq 0$ that is sufficient to cope with the evolution of Algorithm~\ref{alg:A2}.

For each $\Cc_k(\xi^{0:k})$ we also equip with a probability $\mathbf{p}(\xi^{0:k}) := \mathbb{P}(\Cc_k(\xi^{0:k}))$.
Then, $(\Omega, \Fc, \mathbb{P})$ forms a probability space.
Our conditional probability is defined as $\mathbf{p}( (i, d) \mid \xi^{0:k}) := \mathbb{P}(\Cc_{k+1}(\xi^{0:k+1}))/\mathbb{P}(\Cc_k(\xi^{0:k}))$, where we set $\mathbf{p}( (i, d) \mid \xi^{0:k}) := 0$ if $\mathbf{p}(\xi^{0:k}) = 0$.
We do not need to know these probabilities in advance.
They are determined based on the particular system such as hardware architecture, software implementation, asynchrony, and our strategy for selecting active user.

Now, if $X$ is a random variable defined on $\Omega$, then as shown in \cite{cannelli2019asynchronous}, we have
\begin{equation}\label{eq:cond_expectation}
\mathbb{E}[X \mid \Fc_k] = \sum_{(i,d)\in [n]\times\Dc} \mathbf{p}((i,d) \mid \xi^{0:k})X(\xi^{0:k}, (i,d)).
\end{equation}
Note from Assumption~\ref{ass:A4} that 
\begin{equation}\label{eq:assA4_01}
\mathbf{p}(i \mid \xi^{0:k}) := \sum_{d\in\Dc}\mathbf{p}((i, d) \mid \xi^{0:k}) \geq \hat{\mbf{p}}.
\end{equation}
Our probability model described above allows us to handle a variety class of asynchronous algorithms derived from the DR splitting scheme.
Here, we do not make independent assumption between the active user $\hat{i}_k$ and the delay vector $\hat{d}^k$.

\subsection{Preparatory lemmas}
For the asynchronous algorithm, Algorithm~\ref{alg:A2}, the following facts hold.
\begin{compactitem}
\item For $x^k_i$ and $y^k_i$ updated by Algorithm~\ref{alg:A2}, since $\Sc_k = \sets{i_k}$ and the update of $y^k_i$ and $x_i^k$ remain the same as in Algorithm~\ref{alg:A1} when the error $e_i^k = 0$, the relation \eqref{eq:yk_xk_relation} remains true, i.e. $y_i^k = x^k_i + \eta\nabla{f_i}(x^k_i)$ and $\hat{x}_i^k = 2x_i^k - y_i^k$ for all $i \in [n]$ and $k\geq 0$.

\item Let $\bar{\mbf{x}}^{k-d^k} := [\bar{x}^{k-d_1^k}, \bar{x}^{k-d_2^k}, \cdots, \bar{x}^{k-d_n^k}]$ be a delayed copy of the vector $\bar{\mbf{x}}^k := [\bar{x}^k, \cdots, \bar{x}^k] \in \R^{np}$.
Since at each iteration $k$, there is only one block $i_k$ being updated, as shown in \cite{cannelli2019asynchronous,peng2016arock}, for all $i\in [n]$, we can write 
\begin{equation}\label{eq:delay_update}
\bar{x}^{k-d_i^k} = \bar{x}^k + \sum_{l \in J_i^k}(\bar{x}^l - \bar{x}^{l+1}),
\end{equation}
where $J^k_i := \sets{k-d_i^k, k-d_i^k+1, \cdots, k-1} \subseteq \set{k-\tau, \cdots, k-1}$.
\end{compactitem}
These facts will be repeatedly used in the sequel. 

Now, let us first prove the following lemma to provide a key estimate for establishing Lemma~\ref{lem:asdr_key_est}.

\begin{lemma}[Sure descent]\label{lem:asdr_key_est0}
Suppose that Assumptions~\ref{ass:A1}, \ref{ass:A2}, and \ref{ass:A4} hold for \eqref{eq:fed_prob}.
Let $\sets{ (x^{k}_i, y^{k}_i, \hat{x}^{k}_i, \tilde{x}^k, \bar{x}^k)}$ be generated by Algorithm~\ref{alg:A2} and $V_{\eta}^{k}(\cdot)$ be defined as in \eqref{eq:lyapunov_func}.
Then, for all $k \geq 0$, the following estimate holds:
\begin{equation}\label{eq:asdr_lem0}
\hspace{-0ex}
\arraycolsep=0.2em
\begin{array}{ll}
V_\eta^{k+1}(\bar{x}^{k+1}) & + {~}  \frac{\tau}{n\eta}\sum_{l=k+1-\tau}^{k} (l-k+\tau)\norms{\bar{x}^{l +1} - \bar{x}^l }^2  \leq V_\eta^{k}(\bar{x}^k)  -  \frac{\rho}{2} \norms{x^{k+1}_{i_k} - x^k_{i_k}}^2 \vspace{1ex}\\
& + {~} \frac{\tau}{n\eta} \sum_{l=k-\tau}^{k-1}  (l - (k-1) +\tau)\norms{\bar{x}^{l +1}  - \bar{x}^l }^2, 
\end{array}
\hspace{-4ex}
\end{equation}
where 
\begin{equation*}
\arraycolsep=0.2em
\begin{array}{lcl}
\rho & := & \begin{cases}
\frac{2(1-\alpha) - (2+\alpha)L^2\eta^2 - L\alpha\eta}{\alpha\eta n} &\text{if} \quad 2\tau^2 \leq n, \vspace{1ex}\\
\frac{n^2[2(1-\alpha) - (2+\alpha)L^2\eta^2 - L\alpha\eta] - \alpha(1+\eta^2L^2)(2\tau^2 - n)}{\alpha\eta n^3} &\text{otherwise}.
\end{cases}
\end{array}
\end{equation*}
\end{lemma}

\begin{proof}
Let $V_{\eta}^k$ be defined by \eqref{eq:lyapunov_func}.
For $(x_i^k, \hat{x}^k_i, y_i^k)$ updated as in Algorithm~\ref{alg:A2}, the results of Lemma~\ref{lem:yk_xk} still hold true.
Hence, \eqref{eq:V_pro2_inexact} still holds for Algorithm~\ref{alg:A2} with $\gamma_3 = 0$ and $E_{k+1}^2 = 0$, i.e.:
\begin{equation*} 
\hspace{-0ex}
\arraycolsep=0.2em
\begin{array}{lcl}
V_{\eta}^{k+1}(\bar{x}^{k+1}) & \leq & g(\bar{x}^k) + \frac{1}{n}\sum_{i=1}^n \big[ f_i(x^{k+1}_i) +  \iprods{\nabla f_i(x^{k+1}_i), \bar{x}^{k} - x^{k+1}_i} + \frac{1}{2\eta} \norms{\bar{x}^{k} - x^{k+1}_i}^2 \big] \vspace{1ex}\\
&& - {~} \frac{1}{2\eta}\norms{\bar{x}^{k+1} - \bar{x}^k}^2.
\end{array}
\hspace{-3ex}
\end{equation*}
Using this inequality, the update of $x^{k+1}_{i_k}$ for $i=i_k$, and $x^{k+1}_i = x_i^k$ for $i\neq i_k$, we can expand
\begin{equation}\label{eq:asdr_eq1_1}
\hspace{-0.15ex}
\arraycolsep=0.2em
\begin{array}{lcl}
V_\eta^{k+1}(\bar{x}^{k+1}) & \leq  & g(\bar{x}^k) +  \frac{1}{n}\sum_{i\neq i_k} \big[ f_i(x^{k}_i) +  \iprods{\nabla f_i(x^{k}_i), \bar{x}^{k} - x^{k}_i} + \frac{1}{2\eta } \norms{\bar{x}^{k} - x^{k}_i}^2 \big] \vspace{1ex}\\
&& +  {~} \frac{1}{n}\big[ f_{i_k}(x^{k+1}_{i_k}) + \iprods{\nabla f_{i_k}(x^{k+1}_{i_k}), x^{k}_{i_k} - x^{k+1}_{i_k}} \big]  + \frac{1}{n}\iprods{\nabla f_{i_k}(x^{k+1}_{i_k}), \bar{x}^{k} - x^{k}_{i_k}} \vspace{1ex}\\
&& +  {~} \frac{1}{2\eta n}\norms{\bar{x}^{k} - x^{k}_{i_k} + x^{k}_{i_k} - x^{k+1}_{i_k}}^2 - \frac{1}{2\eta}\norms{\bar{x}^{k+1} - \bar{x}^k}^2.
\end{array}
\hspace{-7ex}
\end{equation}
Now, by the $L$-smoothness of $f_{i_k}$, we have
\begin{equation*} 
f_{i_k}(x^{k+1}_{i_k}) + \iprods{\nabla f_{i_k}(x^{k+1}_{i_k}), x^{k}_{i_k} - x^{k+1}_{i_k}} \le f_{i_k}(x^{k}_{i_k})  + \frac{L}{2}\norms{x^{k}_{i_k} - x^{k+1}_{i_k}}^2.
\end{equation*}
Plugging this inequality into \eqref{eq:asdr_eq1_1} and expanding the third last term of \eqref{eq:asdr_eq1_1}, we obtain
\begin{equation}\label{eq:asdr_eq1}
\arraycolsep=0.2em
\begin{array}{lcl}
V_\eta^{k+1}(\bar{x}^{k+1}) &\leq &  g(\bar{x}^k) + \frac{1}{n}\sum_{i \neq i_k } \big[ f_i(x^{k}_i) +  \iprods{ \nabla{f_i}(x^{k}_i), \bar{x}^{k} - x^{k}_i } + \frac{1}{2\eta} \norms{\bar{x}^{k} - x^{k}_i}^2 \big]  \vspace{1ex}\\
&& + {~} \frac{1}{n} f_{i_k}(x^{k}_{i_k}) + \frac{L}{2n}\norms{x^{k+1}_{i_k} - x^k_{i_k}}^2  + \frac{1}{n}\iprods{\nabla f_{i_k}(x^{k+1}_{i_k}), \bar{x}^{k} - x^{k}_{i_k}} \vspace{1ex}\\
&& + {~} \frac{1}{2\eta n}\norms{\bar{x}^{k} - x^{k}_{i_k}}^2  +  \frac{1}{2\eta n}\norms{x^{k+1}_{i_k} - x^{k}_{i_k}}^2 +  \frac{1}{\eta n}\iprods{x^{k+1}_{i_k} - x^{k}_{i_k}, x^{k}_{i_k} - \bar{x}^{k}}  \vspace{1ex}\\
&& - {~} \frac{1}{2\eta}\norms{\bar{x}^{k+1} - \bar{x}^k}^2  \vspace{1ex}\\
&=  & g(\bar{x}^k) +  \frac{1}{n}\sum_{i=1}^n \big[ f_i(x^{k}_i) +  \iprods{\nabla f_i(x^{k}_i), \bar{x}^{k} - x^{k}_i} + \frac{1}{2\eta} \norms{\bar{x}^{k} - x^{k}_i}^2 \big] \vspace{1ex}\\
&& + {~} \frac{(1+\eta L)}{2n\eta}\norms{x^{k+1}_{i_k} - x^k_{i_k}}^2 + \frac{1}{n}\iprods{\nabla f_{i_k}(x^{k+1}_{i_k}) - \nabla f_{i_k}(x^{k}_{i_k}), \bar{x}^{k} - x^{k}_{i_k}}\vspace{1ex}\\
&& + {~}  \frac{1}{\eta n}\iprods{x^{k+1}_{i_k} - x^{k}_{i_k}, x^{k}_{i_k} - \bar{x}^{k}} - \frac{1}{2\eta}\norms{\bar{x}^{k+1} - \bar{x}^k}^2  \vspace{1ex}\\
&\overset{\tiny\eqref{eq:lyapunov_func}}{=} & V_\eta^{k}(\bar{x}^k)  + \frac{1}{n}\iprods{\nabla f_{i_k}(x^{k+1}_{i_k}) - \nabla f_{i_k}(x^{k}_{i_k}), \bar{x}^{k-d^k_{i_k}} - x^{k}_{i_k}} \vspace{1ex}\\
&&  + {~} \frac{1}{n}\iprods{\nabla f_{i_k}(x^{k+1}_{i_k}) - \nabla f_{i_k}(x^{k}_{i_k}), \bar{x}^{k} - \bar{x}^{k-d^k_{i_k}}} + \frac{(1 + L\eta)}{2\eta n}\norms{x^{k+1}_{i_k} - x^k_{i_k}}^2  \vspace{1ex}\\
&& + {~}  \frac{1}{\eta n}\iprods{x^{k+1}_{i_k} - x^{k}_{i_k}, x^{k}_{i_k} - \bar{x}^{k-d^k_{i_k}}} +  \frac{1}{\eta n}\iprods{x^{k+1}_{i_k} - x^{k}_{i_k}, \bar{x}^{k-d^k_{i_k}} - \bar{x}^{k}} \vspace{1ex}\\
&& - {~} \frac{1}{2\eta}\norms{\bar{x}^{k+1} - \bar{x}^k}^2.
\end{array}
\hspace{-1ex}
\end{equation}
From $y^{k+1}_{i_k} :=  y_{i_k}^k + \alpha(\bar{x}^{k-d^k_{i_k}} - x^{k}_{i_k})$ at Step~\ref{step:A2o3} of Algorithm~\ref{alg:A2} and the relation \eqref{eq:yk_xk_relation}, we have
\begin{equation}\label{eq:asdr_eq5_10}
\bar{x}^{k-d^k_{i_k}} - x^{k}_{i_k} = \frac{1}{\alpha}(y^{k+1}_{i_k} - y^{k}_{i_k}) \overset{\eqref{eq:yk_xk_relation}}{=} \frac{1}{\alpha}(x^{k+1}_{i_k} - x^{k}_{i_k}) + \frac{\eta}{\alpha }(\nabla f_{i_k}(x^{k+1}_{i_k}) - \nabla f_{i_k}(x^{k}_{i_k})).
\end{equation}
This relation leads to
\begin{equation}\label{eq:asdr_eq3}
\arraycolsep=0.2em
\begin{array}{lcl}
\frac{1}{n}\iprods{\nabla f_{i_k}(x^{k+1}_{i_k}) - \nabla f_{i_k}(x^{k}_{i_k}), \bar{x}^{k-d^k_{i_k}} - x^{k}_{i_k}} &= & \frac{1}{\alpha n}\iprods{\nabla f_{i_k}(x^{k+1}_{i_k}) - \nabla f_{i_k}(x^{k}_{i_k}), x^{k+1}_{i_k} - x^{k}_{i_k}} \vspace{1ex}\\
&& + {~} \frac{\eta}{\alpha n}\norms{\nabla f_{i_k}(x^{k+1}_{i_k}) - \nabla f_{i_k}(x^{k}_{i_k})}^2,
\end{array}
\hspace{-1ex}
\end{equation}
and
\begin{equation}\label{eq:asdr_eq4}
\arraycolsep=0.2em
\begin{array}{lcl}
\frac{1}{\eta n} \iprods{x^{k+1}_{i_k} - x^{k}_{i_k}, x^{k}_{i_k} - \bar{x}^{k-d^k_{i_k}}} & = &  - \frac{1}{\alpha n}\iprods{ \nabla f_{i_k}(x^{k+1}_{i_k}) - \nabla f_{i_k}(x^{k}_{i_k}), x^{k+1}_{i_k} - x^{k}_{i_k} } \vspace{1ex}\\
&& - {~} \frac{1}{\eta\alpha n}\norms{x^{k+1}_{i_k} - x^{k}_{i_k}}^2.
\end{array}
\end{equation}
Substituting \eqref{eq:asdr_eq3} and \eqref{eq:asdr_eq4} into \eqref{eq:asdr_eq1}, we obtain
\begin{equation*} 
\arraycolsep=0.2em
\begin{array}{lcl}
V_\eta^{k+1}(\bar{x}^{k+1}) &\leq & V_\eta^{k}(\bar{x}^k) + \frac{(1 + L\eta)}{2\eta n}\norms{x^{k+1}_{i_k} - x^k_{i_k}}^2 + \frac{\eta}{\alpha n}\norms{\nabla f_{i_k}(x^{k+1}_{i_k}) - \nabla f_{i_k}(x^{k}_{i_k})}^2 \vspace{1ex}\\
&& + {~} \frac{1}{n}\iprods{\nabla f_{i_k}(x^{k+1}_{i_k}) - \nabla f_{i_k}(x^{k}_{i_k}), \bar{x}^{k} - \bar{x}^{k-d^k_{i_k}}} + \frac{1}{\eta n}\iprods{x^{k+1}_{i_k} - x^{k}_{i_k},  \bar{x}^{k-d^k_{i_k}} - \bar{x}^{k}} \vspace{1ex}\\
&&  - {~} \frac{1}{\eta\alpha n}\norms{x^{k+1}_{i_k} - x^{k}_{i_k}}^2  - \frac{1}{2\eta}\norms{\bar{x}^{k+1} - \bar{x}^k}^2 \vspace{1ex}\\
&\overset{\eqref{eq:L_smooth}}{\leq} & V_\eta^{k}(\bar{x}^k) + \frac{\alpha(L\eta + 1) - 2}{2\eta\alpha n}\norms{x^{k+1}_{i_k} - x^k_{i_k}}^2 + \frac{\eta L^2}{\alpha n}\norms{x^{k+1}_{i_k} - x^{k}_{i_k}}^2\vspace{1ex}\\
&&  + {~} \frac{1}{n}\iprods{\nabla f_{i_k}(x^{k+1}_{i_k}) - \nabla f_{i_k}(x^{k}_{i_k}), \bar{x}^{k} -  \bar{x}^{k-d^k_{i_k}}} + \frac{1}{\eta n}\iprods{x^{k+1}_{i_k} - x^{k}_{i_k},  \bar{x}^{k-d^k_{i_k}} - \bar{x}^{k}} \vspace{1ex}\\
&& - {~} \frac{1}{2\eta}\norms{\bar{x}^{k+1} - \bar{x}^k}^2.
\end{array}
\end{equation*}
Next, using Young's inequality twice in the above estimate, we can further expand 
\begin{equation}\label{eq:asdr_eq5}
\hspace{-0.25ex}
\arraycolsep=0.2em
\begin{array}{lcl}
V_\eta^{k+1}(\bar{x}^{k+1}) & \leq & V_\eta^{k}(\bar{x}^k) + \frac{\alpha(L\eta + 1) + 2L^2\eta^2 - 2}{2\eta\alpha n}\norms{x^{k+1}_{i_k} - x^k_{i_k}}^2 + \frac{\eta}{2n}\norms{\nabla f_{i_k}(x^{k+1}_{i_k}) - \nabla f_{i_k}(x^{k}_{i_k})}^2 \vspace{1ex}\\
&& + {~} \frac{1}{2\eta n}\norms{ \bar{x}^{k} -  \bar{x}^{k-d^k_{i_k}}}^2 + \frac{1}{2\eta n}\norms{x^{k+1}_{i_k} - x^{k}_{i_k}}^2 + \frac{1}{2\eta n}\norms{ \bar{x}^{k} -   \bar{x}^{k-d^k_{i_k}} }^2  \vspace{1ex}\\
&& - {~} \frac{1}{2\eta}\norms{\bar{x}^{k+1} - \bar{x}^k}^2 \vspace{1ex}\\
&\overset{\eqref{eq:L_smooth}}{\leq} & V_\eta^{k}(\bar{x}^k) + \frac{[ \alpha(L\eta + 2) + 2L^2\eta^2 - 2 ]}{2\alpha\eta n}\norms{x^{k+1}_{i_k} - x^k_{i_k}}^2 + \frac{L^2\eta}{2n}\norms{x^{k+1}_{i_k} - x^{k}_{i_k}}^2 \vspace{1ex}\\
&& + {~}  \frac{1}{\eta n}\norms{\bar{x}^{k} -  \bar{x}^{k-d^k_{i_k}}}^2 - \frac{1}{2\eta}\norms{\bar{x}^{k+1} - \bar{x}^k}^2  \vspace{1ex}\\
&= & V_\eta^{k}(\bar{x}^k) + \frac{[ \alpha(L^2\eta^2 + L\eta + 2) + 2L^2\eta^2  - 2] }{2\alpha\eta n}\norms{x^{k+1}_{i_k} - x^k_{i_k}}^2 - \frac{1}{2\eta}\norms{\bar{x}^{k+1} - \bar{x}^k}^2 \vspace{1ex}\\
&&  + {~}   \frac{1}{n\eta} \norms{\bar{x}^{k-d^k_{i_k}} - \bar{x}^k}^2.
\end{array}
\hspace{-5ex}
\end{equation}
Using \eqref{eq:delay_update}, we can bound $\norms{\bar{x}^{k - d^k_{i_k}} - \bar{x}^k}^2$ as follows:
\begin{equation}\label{eq:asdr_eq5_1}
\hspace{-0.0ex}
\arraycolsep=0.0em
\begin{array}{lcl}
\norms{\bar{x}^{k-d^k_{i_k}} - \bar{x}^k }^2 & \overset{\tiny\eqref{eq:delay_update}}{=} & \big\Vert \sum_{l\in J_{i_k}^k}(\bar{x}^l - \bar{x}^{l+1}) \big\Vert^2   \vspace{1ex}\\
&\leq & d_{i_k}^k \sum_{l =k - d^k_{i_k}}^{k-1}\norms{\bar{x}^{l+1} - \bar{x}^{l}}^2 \quad \text{(Young's inequality and the definition of $J_{i_k}^k$)} \vspace{1ex}\\
&\leq & \tau \sum_{l=k - \tau}^{k-1} \norms{\bar{x}^{l+1} - \bar{x}^{l}}^2 \quad \text{(since $d_{i_k}^k \leq \tau$ in Assumption~\ref{ass:A4})} \vspace{1ex}\\
&= & \tau \Big[\sum\limits_{l=k-\tau}^{k-1} [l - (k-\tau) +1] \norms{\bar{x}^{l+1} - \bar{x}^l }^2  - {\!\!\!\!} \sum\limits_{l= k -\tau+1}^{k} {\!\!\!} (l - (k - \tau)) \norms{\bar{x}^{l+1} - \bar{x}^l}^2 \Big]  \vspace{1ex}\\
&& + {~} \tau^2 \norms{\bar{x}^{k+1}- \bar{x}^k}^2.
\end{array}
\hspace{-9ex}
\end{equation}
Now, we consider two cases as follows.

\textbf{Case 1:} If $n \geq 2\tau^2$, then by plugging \eqref{eq:asdr_eq5_1} into \eqref{eq:asdr_eq5}, we finally arrive at
\begin{equation*} 
\arraycolsep=0.2em
\begin{array}{ll}
V_\eta^{k+1}(\bar{x}^{k+1}) & + {~}  \frac{\tau}{n\eta} \sum_{l=k -\tau + 1}^{k} [ l - (k - \tau)] \norms{\bar{x}^{l+1} - \bar{x}^l}^2 \leq  V_\eta^{k}(\bar{x}^k) \vspace{1ex}\\
& + {~} \frac{\tau}{n\eta} \sum_{l =k-\tau}^{k-1} [l - (k-\tau) + 1]\norms{\bar{x}^{l+1} - \bar{x}^l}^2   \vspace{1ex}\\ 
& - {~} \frac{[2(1-\alpha) - (2+\alpha)\eta^2L^2 - \alpha \eta  L] }{2\alpha\eta n}  \norms{x^{k+1}_{i_k} - x^k_{i_k}}^2  - \frac{( n - 2\tau^2 )}{2n\eta}\norms{\bar{x}^{k+1} - \bar{x}^k}^2.
\end{array}
\end{equation*}
Rearranging the last estimate, we finally arrive at \eqref{eq:asdr_lem0}.

\textbf{Case 2:} if $2\tau^2 > n$, then using \eqref{eq:yk_xk_relation}, we can show that 
\begin{equation}\label{eq:asdr_eq5_2b}
\hspace{-0.0ex}
\arraycolsep=0.0em
\begin{array}{lcl}
\norms{\bar{x}^{k+1} - \bar{x}^k}^2 &= & \big\Vert \prox_{\eta g}\big( \tilde{x}^{k+1} \big) - \prox_{\eta g}\big( \tilde{x}^{k} \big) \big\Vert^2 \leq  \Vert \tilde{x}^{k+1} - \tilde{x}^k \Vert^2 \vspace{1ex}\\
& = &  \Vert \frac{1}{n}\sum_{i=1}^n(\hat{x}_i^{k+1} - \hat{x}_i^k) \Vert^2  \vspace{1ex}\\
& = & \frac{1}{n^2} \Vert \hat{x}_{i_k}^{k+1} - \hat{x}_{i_k}^k \Vert^2  \ \ \text{(since only block $i_k$ is updated)} \vspace{1ex}\\
&\overset{\eqref{eq:yk_xk_relation}}{=} &   \frac{1}{n^2} \norms{ (x^{k+1}_{i_k} - x^k_{i_k}) -  \eta(\nabla f_{i_k}(x^{k+1}_{i_k}) - \nabla f_{i_k}(x^{k}_{i_k})) }^2  \vspace{1ex}\\
&\leq &  \frac{2}{n^2} \norms{ x^{k+1}_{i_k} - x^k_{i_k} }^2 + \frac{2\eta^2}{n^2}\norms{ \nabla f_{i_k}(x^{k+1}_{i_k}) - \nabla f_{i_k}(x^{k}_{i_k}) }^2  \vspace{1ex}\\
&\leq & \frac{2(1+\eta^2L^2)}{n^2}\norms{ x^{k+1}_{i_k} - x^k_{i_k} }^2.
\end{array}
\hspace{-3ex}
\end{equation}
Substituting this inequality into the previous one, we can get
\begin{equation*} 
\arraycolsep=0.2em
\begin{array}{ll}
V_\eta^{k+1}(\bar{x}^{k+1}) & + {~}  \frac{\tau}{n\eta} \sum_{l=k - \tau +1}^{k} [ l - (k - \tau)] \norms{\bar{x}^{l+1} - \bar{x}^l}^2 \leq  V_\eta^{k}(\bar{x}^k) \vspace{1ex}\\
& + {~} \frac{\tau}{n\eta} \sum_{l =k-\tau}^{k-1} [l - (k-\tau) + 1]\norms{\bar{x}^{l+1} - \bar{x}^l}^2  \vspace{1ex}\\ 
& - {~} \left[ \frac{2(1-\alpha) - (2 +\alpha)\eta^2 L^2 - \alpha \eta  L }{2\alpha\eta n} - \frac{(1+\eta^2L^2)(2\tau^2 - n)}{2n^3\eta} \right] \norms{x^{k+1}_{i_k} - x^k_{i_k}}^2.
\end{array}
\end{equation*}
Simplifying the coefficients of  this estimate, we finally arrive at \eqref{eq:asdr_lem0}.
\end{proof}

To analyze Algorithm~\ref{alg:A2}, we need the following key lemma.

\begin{lemma}[Sure descent lemma]\label{lem:asdr_key_est}
Suppose that Assumptions~\ref{ass:A1}, \ref{ass:A2}, and \ref{ass:A4} hold.
Let $\set{ (x^{k}_i, y^{k}_i, \hat{x}^{k}_i, \tilde{x}^k, \bar{x}^k)}$ be generated by Algorithm~\ref{alg:A2} and $V_{\eta}^{k}$ be defined as in \eqref{eq:lyapunov_func}. 
Let 
\begin{equation}\label{eq:asdr_merit_func}
\begin{array}{l}
\widetilde{V}_\eta^{k}(\bar{x}^k) :=  V_{\eta}^{k}(\bar{x}^k) + \frac{1}{n\eta}   \sum_{l=k - \tau}^{k-1}  [l- (k - \tau) + 1] \norms{\bar{x}^{l +1}  - \bar{x}^l}^2.
\end{array}
\end{equation}
Suppose that we choose $0 < \alpha < \bar{\alpha}$ and $0 < \eta < \bar{\eta}$, where $c := \frac{2\tau^2 - n}{n^2}$, 
\begin{equation}\label{eq:choice_of_params}
\begin{array}{ll}
& \bar{\alpha} := \begin{cases}
1 &\text{if}~2\tau^2 \leq n, \vspace{1ex}\\ 
\frac{2}{2 + c} &\text{otherwise},
\end{cases}\vspace{1ex}\\
\text{and} &  \bar{\eta} := \begin{cases}
\frac{\sqrt{16 - 8\alpha - 7\alpha^2} - \alpha}{2L(2+\alpha)} &\text{if}~2\tau^2 \leq n, \vspace{1ex}\\ 
\tfrac{\sqrt{16 - 8\alpha - (7 + 4c + 4c^2)\alpha^2} - \alpha}{2L[2 + (1+c)\alpha]} &\text{otherwise}.
\end{cases}
\end{array}
\end{equation}
Then, the following statement holds:
\begin{equation}\label{eq:asdr_lem}
\frac{\rho}{2} \norms{x^{k+1}_{i_k} - x^k_{i_k}}^2  \leq \widetilde{V}_\eta^{k}(\bar{x}^k)  -  \widetilde{V}_\eta^{k+1}(\bar{x}^{k+1}),
\end{equation}
where 
\begin{equation*}
\rho := \begin{cases}
\frac{2(1-\alpha) - (2+\alpha)L^2\eta^2 - L\alpha\eta}{\alpha\eta n} &\text{if} \quad 2\tau^2 \leq n, \vspace{1ex} \\
\frac{n^2[2(1-\alpha) - (2+\alpha)L^2\eta^2 - L\alpha\eta] - \alpha(1+\eta^2L^2)(2\tau^2 - n)}{\alpha\eta n^3} &\text{otherwise}.
\end{cases}
\end{equation*}
Moreover, $\rho$ is positive.
\end{lemma}

\begin{proof}
If we define $\widetilde{V}_\eta^k$ as in \eqref{eq:asdr_merit_func} of Lemma~\ref{lem:asdr_key_est}, i.e.:
\begin{equation*}
\widetilde{V}_\eta^k(\bar{x}^k) := V_\eta^{k}(\bar{x}^k) + \frac{\tau}{\eta n^2} \sum_{l = k-\tau}^{k-1} [l - (k-\tau) + 1] \norms{\bar{x}^{l+1}  - \bar{x}^l}^2,
\end{equation*}
then from \eqref{eq:asdr_lem0}, we have
\begin{equation*} 
\widetilde{V}_\eta^{k+1}(\bar{x}^{k+1}) \leq \widetilde{V}^k(\bar{x}^k) - \frac{\rho}{2}\norms{x^{k+1}_{i_k} - x^k_{i_k}}^2,  
\end{equation*}
which is equivalent to \eqref{eq:asdr_lem}.

Now, we find conditions of $\alpha$ and $\eta$ such that $\rho$ and $\theta$ are positive.
We consider two cases as follows.

\textbf{Case 1:} If $2\tau^2 \leq n$, then 
\begin{equation*}
\rho :=  \tfrac{2(1-\alpha) - (2+\alpha)L^2\eta^2 - L\alpha\eta}{\alpha\eta n}.  
\end{equation*}
Let us choose $0 < \alpha < 1$.
To guarantee $\rho > 0$, we require $2(1-\alpha) > (2+\alpha)L^2\eta^2 + L\alpha\eta$.
In this case, we need to choose $0 < \eta < \frac{\sqrt{L^2\alpha^2 + 8(1-\alpha)(2+\alpha)L^2} - L\alpha}{2L^2(2+\alpha)} = \frac{\sqrt{16 - 8\alpha - 7\alpha^2} - \alpha}{2L(2+\alpha)}$.
These are the choices in \eqref{eq:choice_of_params} when $2\tau^2 \leq n$.

\textbf{Case 2:} If $2\tau^2 > n$, then 
\begin{equation*}
\rho :=  \tfrac{n^2[2(1-\alpha) - (2+\alpha)L^2\eta^2 - L\alpha\eta] - \alpha(1+\eta^2L^2)(2\tau^2 - n)}{\alpha\eta n^3}. 
\end{equation*}
Let $c := \frac{2\tau^2 - n}{n^2} > 0$.
In order to guarantee that $\rho > 0$, we need to choose $0 < \alpha < 1$ and $\eta > 0$ such that 
\begin{equation*}
\begin{array}{ll}
& 2 - 2\alpha - \frac{\alpha(2\tau^2 - n)}{n^2} > \big[ 2 + \alpha + \frac{\alpha(2\tau^2 - n)}{n^2} \big] L^2\eta^2 + L \alpha \eta,  \vspace{1ex}\\
\text{and} & 0 < \alpha < \frac{2n^2}{2n^2 + (2\tau^2 - n)} = \frac{2}{2 + c}.
\end{array}
\end{equation*}
Using the definition of $c$, the first condition becomes $2 - 2\alpha - c\alpha > L\alpha \eta + (2 + \alpha + c\alpha)L^2\eta^2$.
First, we need to impose $2 - 2\alpha - c\alpha > 0$, leading to $0 < \alpha < \frac{2}{2 + c}$.
Next, we solve the above inequality w.r.t. $\eta > 0$ to get 
\begin{equation*}
0 < \eta < \bar{\eta} := \tfrac{\sqrt{16 - 8\alpha - (7 + 4c + 4c^2)\alpha^2} - \alpha}{2L[2 + (1+c)\alpha]}.
\end{equation*}
These are the choices in \eqref{eq:choice_of_params} when $2\tau^2 > n$.
To guarantee $\bar{\eta} > 0$, we need to choose $\alpha < \frac{4}{1 + \sqrt{1 + 4(2 + c + c^2)}}$.
Combining four conditions of $\alpha$, we get $0 < \alpha < \frac{2}{2 + c}$.
Finally, we conclude that under the choice of $\alpha$ and $\eta$ as in \eqref{eq:choice_of_params}, we have $\rho > 0$ and $\theta > 0$.
\end{proof}

Next lemma bounds the term $\sum_{i=1}^n \Exp{\norms{\bar{x}^k - x^k_i}^2}$ in order to bound $\mathbb{E}\big[ \norms{\Gc_{\eta}(\bar{x}^k) }^2 \big]$.

\begin{lemma}\label{le:bound_of_grad}
Suppose that Assumptions~\ref{ass:A1}, \ref{ass:A2}, and \ref{ass:A4} hold.
Let $\set{ (x^{k}_i, y^{k}_i,  \hat{x}^{k}_i, \bar{x}^k)}$ be generated by Algorithm~\ref{alg:A2}.
Then, we have
\begin{equation}\label{eq:asdr_thm_eq11}
\sum_{i=1}^n \Exp{\norms{\bar{x}^k - x^k_i}^2}  \leq  D \sum_{t=k-\tau}^{k+T} \Exp{\norms{x^{t+1}_{i_t} - x^t_{i_t}}^2}, 
\end{equation}
where $D := \frac{8\alpha^2(1 + L^2\eta^2)(\tau^2 + 2Tn\hat{\mbf{p}}) \ + \ 8n^2(1 + L^2\eta^2 + T\alpha^2\hat{\mbf{p}})}{\hat{\mbf{p}}\alpha^2n^2}$. 
\end{lemma}

\begin{proof}
Let $t_k(i) := \min\set{ t \in \sets{0,\cdots, T} : \mathbf{p}(i \mid \xi^{0:k+t-1}) \geq \hat{\mbf{p}}}$.
In fact, $t_k(i)$ is the first time in the iteration window $[k, k+T]$, user $i$ is active, i.e. gets updated. 
For any $\gamma \in (0, 1)$, we have
\begin{equation*} 
\arraycolsep=0.0em
\begin{array}{lcl}
\sum_{t=k}^{k+T} \Exp{\norms{\bar{x}^t  - x^t_{\hat{i}_t}}^2 \mid \Fc_{t-1}}(\omega) &= & \sum_{t=k}^{k+T}\sum_{i=1}^n \mathbf{p}(i \mid \xi^{0:t-1})\norms{ \bar{x}^t - x^t_i}^2 \vspace{1ex}\\
&\overset{\eqref{eq:assA4_01}}{\geq} & \sum_{i=1}^n \hat{\mbf{p}} \norms{\bar{x}^{k+ t_k(i)} - x^{k + t_k(i)}_i}^2 \vspace{1ex}\\
&\overset{(*)}{\geq} & \hat{\mbf{p}} \sum_{i=1}^n \left[\norms{\bar{x}^k - x^k_i} - \norms{\bar{x}^{k+t_k(i)} - x^{k+t_k(i)}_i - (\bar{x}^k - x^k_i)} \right]^2 \vspace{1ex}\\
&\geq & - 2\hat{\mbf{p}} \sum_{i=1}^n \norms{\bar{x}^k - x^k_i}\norms{\bar{x}^{k+t_k(i)} - x^{k+t_k(i)}_i - (\bar{x}^k - x^k_i)} \vspace{1ex}\\
&& + {~} \hat{\mbf{p}} \sum_{i=1}^n \norms{\bar{x}^k - x^k_i}^2 \vspace{1ex}\\
&\overset{(**)}{\geq} & \hat{\mbf{p}} \sum_{i=1}^n \left[\norms{\bar{x}^k - x^k_i}^2 - \frac{1}{2}\norms{\bar{x}^k - x^k_i}^2\right] \vspace{1ex}\\
&& - {~} 4\hat{\mbf{p}} \sum_{i=1}^n \norms{\bar{x}^{k+t_k(i)} - \bar{x}^k}^2 - 4\hat{\mbf{p}} \sum_{i=1}^n\norms{x^{k+t_k(i)}_i -  x^k_i}^2,
\end{array}
\end{equation*}
where (*) comes from the reverse triangle inequality $\norms{a - b}^2 \ge (\norms{a} - \norms{b})^2$ and (**) is due to $4\norms{v}^2 + 4\norms{s}^2 + \frac{1}{2}\norms{u}^2 \geq 2\norms{u}\norms{v+s}$.
Note that the conditional expectation above is only taken w.r.t. $\hat{i}_k$, which is $\sigma(d^k, \Fc_{k-1})$-measurable.
For simplicity of notation, we drop $(\omega)$ in the sequel.

Rearranging the last inequality, we obtain
\begin{equation}\label{eq:asdr_thm_eq1}
\arraycolsep=0.2em
\begin{array}{lcl}
\frac{\hat{\mbf{p}}}{2} \sum_{i=1}^n \norms{\bar{x}^k - x^k_i}^2 
&\leq & \sum_{t=k}^{k+T} \Exp{\norms{\bar{x}^t - x^t_{\hat{i}_t}}^2 \mid \Fc_{t-1}} + 4\hat{\mbf{p}} \sum_{i=1}^n \norms{\bar{x}^{k+t_k(i)} - \bar{x}^k}^2 \vspace{1ex}\\
&& + {~} 4\hat{\mbf{p}} \sum_{i=1}^n\norms{x^{k+t_k(i)}_i - x^k_i}^2.
\end{array}
\hspace{-3ex}
\end{equation}
Next, we bound the term $\sum_{i=1}^n\norms{\bar{x}^{k+t_k(i)} - \bar{x}^k}^2$ as follows:
\begin{equation}\label{eq:asdr_thm_eq9}
\arraycolsep=0.2em
\begin{array}{lcl}
\sum_{i=1}^n \norms{\bar{x}^{k+t_k(i)} - \bar{x}^k}^2 &= &  \sum_{i=1}^n \norms{\sum_{t=k}^{k+t_k(i)-1} (\bar{x}^{t+1} - \bar{x}^t)}^2 \vspace{1ex}\\
&\leq &  \sum_{i=1}^n t_k(i)\sum_{t=k}^{k+t_k(i)-1} \norms{\bar{x}^{t+1} - \bar{x}^t}^2 \qquad \text{(Young's inequality)} \vspace{1ex}\\
&\leq & T\sum_{i=1}^n \sum_{t=k}^{k+t_k(i)-1} \norms{\bar{x}^{t+1} - \bar{x}^t}^2 \qquad \text{(since $t_k(i) \leq T$)}\vspace{1ex}\\
&= & nT \sum_{t=k}^{k+T} \norms{\bar{x}^{t+1} - \bar{x}^t}^2 \vspace{1ex}\\
&\overset{\tiny\eqref{eq:asdr_eq5_2b}}{\leq} & \frac{2T(1+\eta^2L^2)}{n}\sum_{t=k}^{k+T}\norms{ x^{t+1}_{i_t} - x^t_{i_t} }^2.  
\end{array}
\hspace{-4ex}
\end{equation}
We can also bound $\sum_{i=1}^n \norms{x^{k+t_k(i)}_i - x^k_i}^2$ as follows:
\begin{equation}\label{eq:asdr_thm_eq10}
\arraycolsep=0.2em
\begin{array}{lcl}
\sum_{i=1}^n\norms{x^{k + t_k(i)}_i - x^k_i}^2 & = & \sum_{i=1}^n  \norms{\sum_{t=k}^{k + t_k(i)-1} (x^{t+1}_{i} - x^t_{i})}^2  \vspace{1ex}\\
&\leq & \sum_{i=1}^n t_k(i)  \sum_{t=k}^{k + t_k(i)-1} \norms{x^{t+1}_{i} - x^t_{i}}^2 \qquad \text{(Young's inequality)} \vspace{1ex}\\
&\leq & T \sum_{t=k}^{k+T-1} \sum_{i=1}^n \norms{x^{t+1}_{i} - x^t_{i}}^2 \qquad \text{(since $t_k(i) \leq T$)} \vspace{1ex}\\
&= & T \sum_{t=k}^{k+T} \norms{x^{t+1}_{i_t} - x^t_{i_t}}^2 \quad\text{(since only user $i_t$ is updated at iteration $t$)}.
\end{array}
\hspace{-6ex}
\end{equation}
Let us bound the first term on the right-hand side of \eqref{eq:asdr_thm_eq1} as follows:
\begin{equation}\label{eq:asdr_thm_eq2}
\arraycolsep=0.2em
\begin{array}{lcl}
\sum_{t=k}^{k+T}\Exp{\norms{\bar{x}^t - x^t_{\hat{i}_t}}^2 \mid \Fc_{t-1}} 
&\leq &  2\sum_{t=k}^{k+T}\Exp{\norms{ \bar{x}^{t-d^t_{\hat{i}_t}}   - x^t_{\hat{i}_t}}^2 \mid \Fc_{t-1}} \vspace{1ex}\\
&& + {~} 2\sum_{t=k}^{k+T}\Exp{ \norms{\bar{x}^t - \bar{x}^{t-d^t_{\hat{i}_t}}  }^2 \mid \Fc_{t-1}}.
\end{array}
\hspace{-3ex}
\end{equation}
However, similar to the proof of \eqref{eq:asdr_eq5_1} and \eqref{eq:asdr_eq5_2b}, we can show that
\begin{equation}\label{eq:asdr_thm_eq3}
\arraycolsep=0.2em
\begin{array}{lcl}
\sum_{t=k}^{k+T} \norms{\bar{x}^t -   \bar{x}^{t-d^t_{\hat{i}_t}} }^2 & \overset{\tiny\eqref{eq:asdr_eq5_1}}{\leq} &  \tau \sum_{t=k}^{k+T} \sum_{l = t - \tau}^{t-1} \norms{\bar{x}^{l+1} - \bar{x}^{l} }^2 \vspace{1ex}\\
&\leq & \tau^2 \sum_{t=k-\tau}^{k+T}  \norms{\bar{x}^{t+1} - \bar{x}^{t} }^2 \vspace{1ex}\\ 
&\overset{\eqref{eq:asdr_eq5_2b}}{\leq} & \sum_{t=k-\tau}^{k+T}  \frac{2\tau^2(1+\eta^2L^2)}{n^2}  \norms{ x^{t+1}_{i_t} - x^t_{i_t} }^2. 
\end{array}
\end{equation}
On the other hand, by using \eqref{eq:asdr_eq5_10}, we have
\begin{equation}\label{eq:asdr_thm_eq6}
\arraycolsep=0.2em
\begin{array}{lcl}
\norms{\bar{x}^{t-d^t_{\hat{i}_t}} - x^t_{\hat{i}_t}}^2 &= & \norms{y^{t+1}_{\hat{i}_t} - y^t_{\hat{i}_t}}^2 \qquad \text{(by the update of $y_{\hat{i}_k}^k$ in Algorithm~\ref{alg:A2})}  \vspace{1ex}\\
&\overset{\eqref{eq:asdr_eq5_10}}{=} &  \big\Vert \frac{1}{\alpha}(x^{t+1}_{\hat{i}_t} - x^{t}_{\hat{i}_t}) + \frac{\eta}{\alpha }(\nabla f_{\hat{i}_t}(x^{t+1}_{\hat{i}_t}) - \nabla f_{\hat{i}_t}(x^{t}_{\hat{i}_t})) \big\Vert^2 \vspace{1ex}\\
&\leq &  \frac{2}{\alpha^2} \norms{x^{t+1}_{\hat{i}_t} - x^{t}_{\hat{i}_t}}^2 + \frac{2\eta^2}{\alpha^2}\norms{\nabla f_{\hat{i}_t}(x^{t+1}_{\hat{i}_t}) - \nabla f_{\hat{i}_t}(x^{t}_{\hat{i}_t})}^2 \vspace{1ex}\\
&\leq & \frac{2(1  + \eta^2 L^2)}{\alpha^2}\norms{x^{t+1}_{\hat{i}_t} - x^{t}_{\hat{i}_t}}^2.
\end{array}
\end{equation}
Therefore, plugging \eqref{eq:asdr_thm_eq3} and \eqref{eq:asdr_thm_eq6} into \eqref{eq:asdr_thm_eq2}, we have
\begin{equation}\label{eq:asdr_thm_eq8}
\hspace{-0.0ex}
\arraycolsep=0.2em
\begin{array}{lcl}
\sum_{t=k}^{k+T}\Exp{\norms{\bar{x}^t - x^t_{\hat{i}_t}}^2 \mid \Fc_{t-1}} 
&\leq & \frac{4\tau^2(1+\eta^2L^2)}{n^2} \sum_{t=k-\tau}^{k+T}\Exp{\norms{x^{t+1}_{\hat{i}_t} - x^t_{\hat{i}_t}}^2 \mid \Fc_{t-1}}  \vspace{1ex}\\
&& + {~} \frac{4(1 + \eta^2 L^2)}{\alpha^2} \sum_{t=k-\tau}^{k+T}\Exp{\norms{x^{t+1}_{\hat{i}_t} - x^{t}_{\hat{i}_t}}^2 \mid \Fc_{t-1}} \vspace{1ex}\\
&= & \frac{4(1+\eta^2L^2)[ \tau^2\alpha^2 + n^2 ] }{n^2\alpha^2} \sum_{t=k-\tau}^{k+T}\Exp{\norms{x^{t+1}_{\hat{i}_t} - x^{t}_{\hat{i}_t}}^2 \mid \Fc_{t-1}}.
\end{array}
\hspace{-3.5ex}
\end{equation}
Substituting \eqref{eq:asdr_thm_eq9}, \eqref{eq:asdr_thm_eq10}, and \eqref{eq:asdr_thm_eq8}  into \eqref{eq:asdr_thm_eq1}, we obtain
\begin{equation*} 
\arraycolsep=0.2em
\begin{array}{lcl}
\frac{\hat{\mbf{p}}}{2}\sum_{i=1}^n \norms{\bar{x}^k - x^k_i}^2  &\leq & \frac{4(1+\eta^2L^2)[ \tau^2\alpha^2 + n^2 ] }{n^2\alpha^2} \sum_{t=k-\tau}^{k+T}\Exp{\norms{x^{t+1}_{\hat{i}_t} - x^{t}_{\hat{i}_t}}^2 \mid \Fc_{t-1}} \vspace{1ex}\\
&& + \frac{8\hat{\mbf{p}}T(1+\eta^2L^2)}{n} \sum_{t=k}^{k+T}\norms{ x^{t+1}_{i_t} - x^t_{i_t} }^2 + 4\hat{\mbf{p}} T \sum_{t=k}^{k+T} \norms{x^{t+1}_{i_t} - x^t_{i_t}}^2 \vspace{1ex}\\
&\leq &   \frac{4(1+\eta^2L^2)[ \tau^2\alpha^2 + n^2 ]}{n^2\alpha^2}  \sum_{t=k-\tau}^{k+T}\Exp{\norms{x^{t+1}_{\hat{i}_t} - x^{t}_{\hat{i}_t}}^2 \mid \Fc_{t-1}}  \vspace{1ex}\\
&& + {~} \frac{4\hat{\mbf{p}}T[ 2(1+\eta^2L^2) + n]}{n} \sum_{t=k-\tau}^{k+T} \norms{x^{t+1}_{i_t} - x^t_{i_t}}^2.
\end{array}
\end{equation*}
Finally, taking full expectation both sides of the last inequality w.r.t. $\sigma(d^k, \Fc_{k-1})$, and multiplying the result by $\frac{2}{\hat{\mbf{p}}}$, we arrive at
\begin{equation*} 
\sum_{i=1}^n \Exp{\norms{\bar{x}^k - x^k_i}^2}  \leq  D \sum_{t=k-\tau}^{k+T} \Exp{\norms{x^{t+1}_{i_t} - x^t_{i_t}}^2},  
\end{equation*}
where $D := \frac{8(1+\eta^2L^2)(\tau^2\alpha^2 + n^2)}{\hat{\mbf{p}}n^2\alpha^2} +   \frac{8T[ 2(1+\eta^2L^2) + n]}{n}$. 
This inequality is exactly \eqref{eq:asdr_thm_eq11}.
\end{proof}

\subsection{The proof of Theorem~\ref{thm:asdr_convergence}: Convergence of Algorithm~\ref{alg:A2}}
By Assumption~\ref{ass:A4}, for each $T$ iterations, the probability of each user $i$ getting updated is at least $\hat{\mbf{p}} > 0$.
Hence, from \eqref{eq:asdr_lem} of Lemma~\ref{lem:asdr_key_est}, we sum up from $t := k-\tau$ to $t := k + T$, and have
\begin{equation*} 
\frac{\rho}{2} \sum_{t=k-\tau}^{k + T} \norms{x^{t+1}_{i_t} - x^t_{i_t}}^2  \leq  \sum_{t=k-\tau}^{k+T} \big[ \widetilde{V}_\eta^{t}(\bar{x}^t) - \widetilde{V}_\eta^{t+1}(\bar{x}^{t+1}) \big],
\end{equation*}
where $\rho > 0$ is given in Lemma~\ref{lem:asdr_key_est}.
Now, take full expectation both sides of this inequality w.r.t. $\Fc_k$, we obtain
\begin{equation}\label{eq:th41_proof1}
\frac{\rho}{2} \sum_{t=k-\tau}^{k + T} \mathbb{E}\big[ \norms{x^{t+1}_{i_t} - x^t_{i_t}}^2 \big]  \leq  \sum_{t=k-\tau}^{k+T} \Big[ \mathbb{E}\big[ \widetilde{V}_\eta^{t}(\bar{x}^t) \big]  - \mathbb{E}\big[ \widetilde{V}_\eta^{t+1}(\bar{x}^{t+1})\big] \Big].
\end{equation}

Next, using \eqref{eq:grad_mapping_bound} from Lemma \ref{le:grad_norm_bound} with $\gamma_2 = 0$, we have
\begin{equation*} 
\Vert \Gc_{\eta}(\bar{x}^k) \Vert^2 \leq \frac{(1 + \eta L)^2}{n\eta^2} \sum_{i=1}^n\norms{x_i^k - \bar{x}^k}^2. 
\end{equation*}
Taking full expectation both sides of this inequality, and then combining the result and \eqref{eq:asdr_thm_eq11}, we obtain 
\begin{equation*}
\arraycolsep=0.2em
\begin{array}{lcl}
\Exp{\norms{\Gc_{\eta}(\bar{x}^k) }^2} & \leq &  \frac{(1 + \eta L)^2D}{n\eta^2} \sum_{t=k-\tau}^{k+T}\Exp{\norms{x^{t+1}_{i_t} - x^t_{i_t}}^2}, 
\end{array}
\end{equation*}
where $D$ is given in Lemma~\ref{le:bound_of_grad}.

Combining the last inequality and \eqref{eq:th41_proof1}, we arrive at
\begin{equation*}
\begin{array}{l}
\Exp{\norms{\Gc_{\eta}(\bar{x}^k) }^2} \leq \frac{2(1 + \eta L)^2D}{n\eta^2\rho} \sum_{t=k-\tau}^{k+T} \left(  \mathbb{E}\big[ \widetilde{V}_\eta^{t}(\bar{x}^t) \big]  - \mathbb{E}\big[ \widetilde{V}_\eta^{t+1}(\bar{x}^{t+1})\big] \right).
\end{array}
\end{equation*}
Averaging this inequality from $k := 0$ to $k := K$, we get
\begin{equation}\label{eq:asdr_thm_eq13}
\arraycolsep=0.2em
\begin{array}{lcl}
\frac{1}{K+1}\sum_{k=0}^K\Exp{\norms{\Gc_{\eta}(\bar{x}^k) }^2} & \leq & \frac{\hat{C}}{K+1}\sum_{k=0}^K \sum_{t=k-\tau}^{k+T}\left[  \mathbb{E}\big[ \widetilde{V}_\eta^{t}(\bar{x}^t) \big]  - \mathbb{E}\big[ \widetilde{V}_\eta^{t+1}(\bar{x}^{t+1})\big] \right]  \vspace{1ex}\\
& \leq & \frac{\hat{C}}{ K+1} \big[ \widetilde{V}_{\eta}^{0}(\bar{x}^0) - \mathbb{E}\big[\widetilde{V}_{\eta}^{K+T+1}(\bar{x}^{K+T+1})\big] \big],
\end{array}
\end{equation}
where $\hat{C} := \frac{2(1+\eta L)^2D}{n \rho \eta^2}$.
Here, we have used the monotonicity of $\sets{\mathbb{E}\big[\widetilde{V}^k_{\eta}(\bar{x}^k)\big]}_{k\geq 0}$ and $\mathbb{E}\big[\widetilde{V}_{\eta}^{0}(\bar{x}^0)] = \widetilde{V}_{\eta}^0(\bar{x}^0)$ in the last equality.

Now, recall from the definition of $\widetilde{V}^k_{\eta}(\cdot)$ and $V^k_{\eta}(\cdot)$ that 
\begin{equation*} 
\widetilde{V}_{\eta}^{0}(\bar{x}^0) = V_{\eta}^0(\bar{x}^0) = F(x^0) \qquad\text{and}\qquad \Exp{\widetilde{V}_{\eta}^k(\bar{x}^k) }  \geq \Exp{V_{\eta}^{k}(\bar{x}^k)} \overset{\eqref{eq:V_lowerbound}}{\geq} F^{\star}.
\end{equation*}
Substituting these relations into \eqref{eq:asdr_thm_eq13}, we eventually get
\begin{equation*}
\frac{1}{K+1}\sum_{k=0}^K\Exp{\norms{\Gc_{\eta}(\bar{x}^k) }^2}  \leq  \frac{\hat{C}}{(K+1)}\big[ F(x^0) - F^{\star}\big],
\end{equation*}
which is exactly \eqref{eq:asdr_thm_key_est}.
Using the definition of $\rho$, $\theta$, and $D$ into $\hat{C}$, we obtain its simplified formula as in Theorem~\ref{thm:asdr_convergence}.
The remaining conclusion of the theorem is a direct consequence of \eqref{eq:asdr_thm_key_est}.
\Eproof

\section{Implementation Details and Additional Numerical Examples}\label{app:add_num_exp}
In this section, we provide more details on the set up of numerical experiments and present additional numerical results to illustrate the performance of our algorithms compared to others.

\subsection{Details on numerical experiments}
\paragraph{\textbf{Parameter selection.}}
We use the learning rate for local solver (SGD) as reported in \cite{Li_MLSYS2020} to approximately evaluate $\prox_{\eta f_i}(y^k_i)$ at each user $i \in [n]$. 
The learning rates are $0.01$ for all synthetic datasets, $0.01$ for MNIST, and $0.003$ for FEMNIST. 
We also perform a grid-search over multiple values to select the parameter and stepsizes  for FedProx, FedPD and FedDR. 
In particular, we choose $\mu \in [0.001,1]$ for FedProx, $\eta \in [1,1000]$ for FedPD, and $\eta \in [1,1000]$, $\alpha \in [0,1.99]$ for FedDR.  
All algorithms perform local SGD updates with 20 epochs to approximately evaluate $\prox_{\eta f_i}(y^k_i)$ before sending the results to server for [proximal] aggregation. 

\paragraph{\textbf{Training models.}}
For all datasets, we use fully-connected neural network as training models. 
For all synthetic datasets, we use a neural network of size $60 \times 32 \times 10$ where we use the format $\text{input size} \times \text{hiddden layer} \times \text{output size}$. 
For MNIST, we use a network of size $784\times 128\times10$. 
For FEMNIST used in the main text,  we reuse the dataset from \cite{Li_MLSYS2020} and a $784 \times 128 \times 26$ model.

\paragraph{\textbf{Composite examples.}}
We test our algorithm under composite setting where we set $g(x) = 0.01 \norm{x}_1$. 
In the first test, we choose $\eta = 500$, $\alpha = 1.95$ and select the local learning rate (\textit{lr}) for SGD to approximately evaluate $\prox_{\eta f_i}(y^k_i)$ from the set $\{0.0025,0.005,0.0075,0.01,0.025\}$ for \texttt{synthetic-(0,0)} and $\{0.001,0.003,0.005,0.008,0.01\}$ for \texttt{FEMNIST}. 
Next, we fix the local learning rate at $0.01$ for \texttt{synthetic-(0,0)} and $0.003$ for \texttt{FEMNIST} then adjust the number of local epochs in the set $\{5,10,15,20,30\}$ to evaluate $\prox_{\eta f_i}(y^k_i)$. 
Finally, we test our algorithm when changing the total number of users participating at each communication round $|\mathcal{S}_k|$. For \texttt{synthetic-(0,0)} dataset, we set $|\mathcal{S}_k| \in \{5,10,15,20,25\}$. For \texttt{FEMNIST} dataset, we set $|\mathcal{S}_k| \in \{10,25,50,75,100\}$.

\paragraph{\textbf{Asynchronous example.}}
To make the sample size larger for each user, we generate the FEMNIST dataset using Leaf \cite{caldas2018leaf}. 
In the new dataset, there are actually 62 classes instead of 26 classes as used in \cite{Li_MLSYS2020}. Therefore, we denote this dataset as \textbf{FEMNIST - 62 classes}. 
In this new dataset, each user has sample size ranging from 97 to 356. 
We implement the communication between server and user using the distributed package in \texttt{Pytorch} \footnote{See \url{https://pytorch.org/tutorials/beginner/dist_overview.html} for more details.} as in \cite{distBelief}. There are 21 threads created, one acts as server and 20 others are users. 
To simulate the case when users have different computing power, we add a certain amount of delay at the end of each user's local update such that the total update time varies between all users. 
For \textbf{FEMNIST - 62 classes} dataset, the model is a fully-connected neural network of the size $784 \times 128 \times 62$.

\subsection{Additional numerical results}
We first present two experiments on iid and non-iid datasets without using user sampling scheme as shown in Figure~\ref{fig:add_exp_synth_iid}.
That is all users participate into the system at each communication round. 

\begin{figure}[ht!]
\begin{center}
    \includegraphics[width = 1\textwidth]{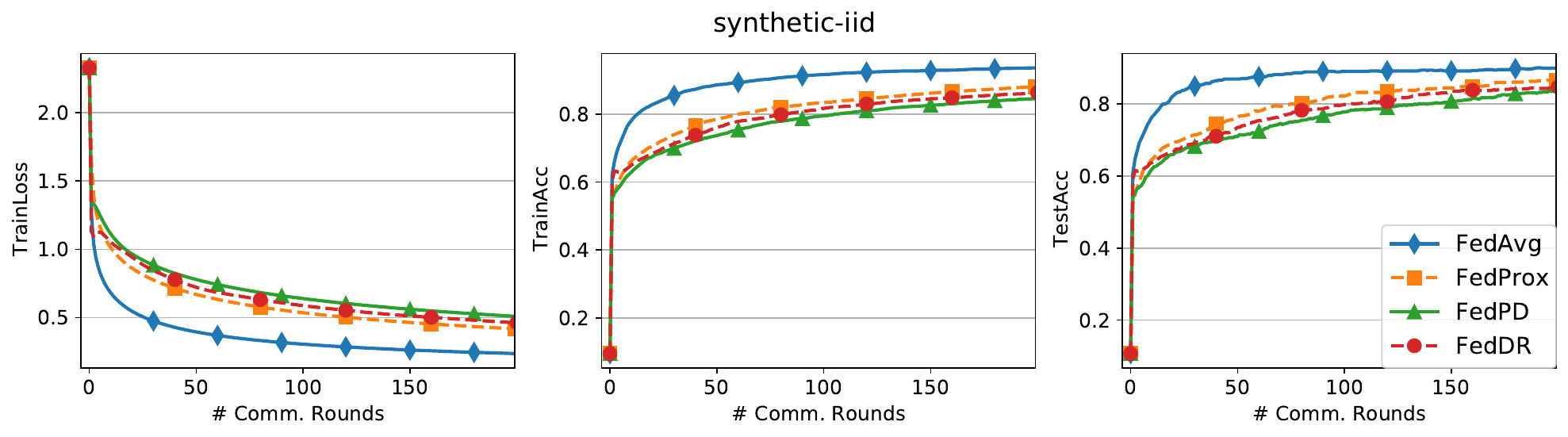}
    \includegraphics[width = 1\textwidth]{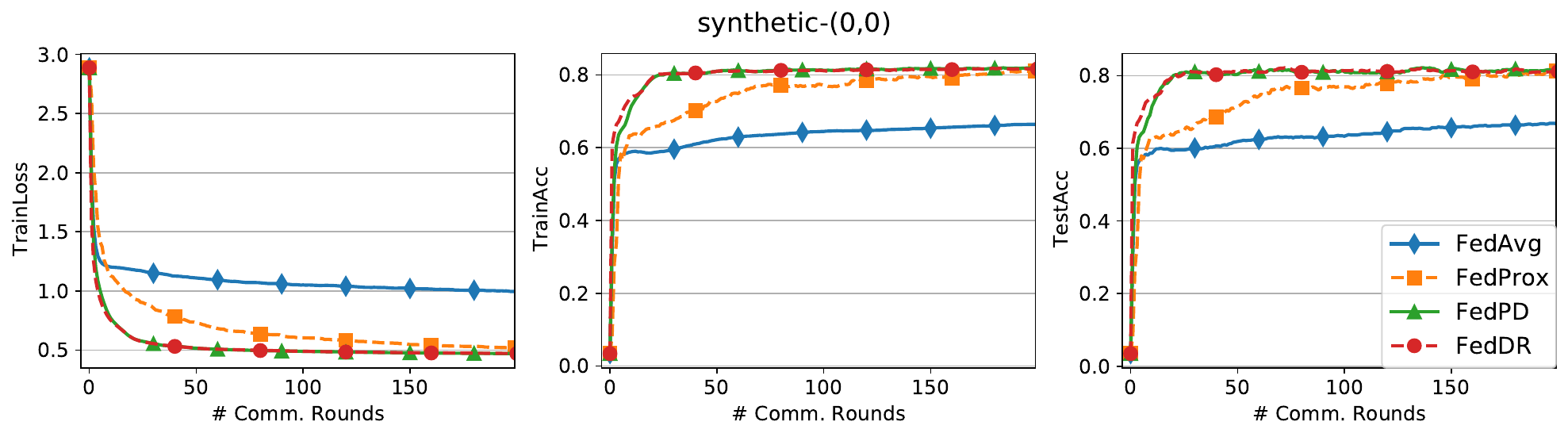}
    \includegraphics[width = 1\textwidth]{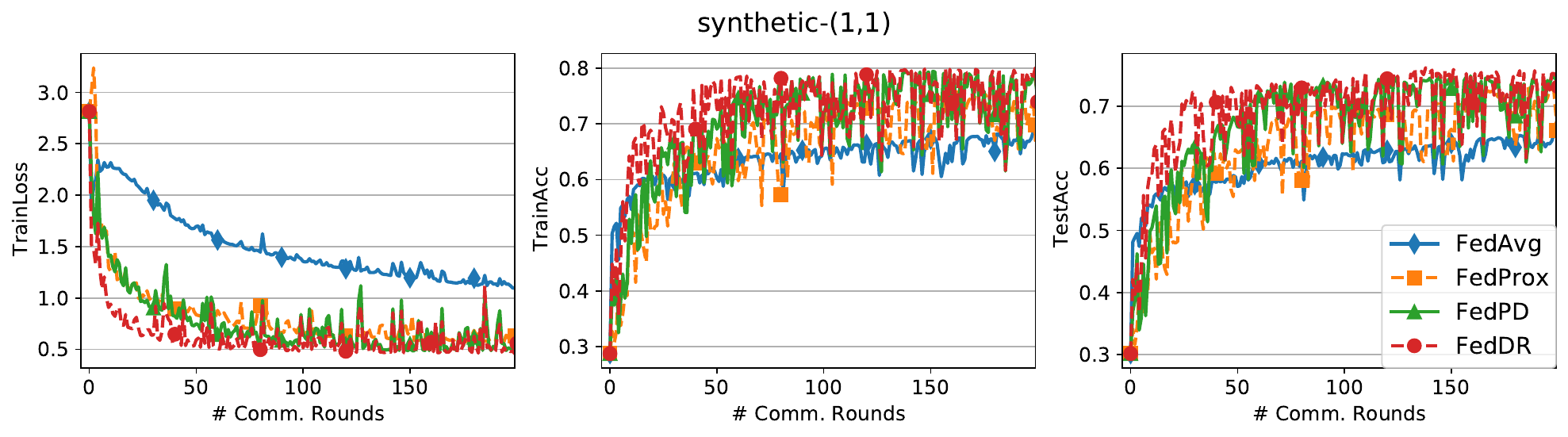}
    \vspace{-2ex}
    \caption{The performance of 4 algorithms on iid and non-iid synthetic datasets without user sampling scheme.
    The first row is for one iid dataset, and the last two rows are for non-iid datasets.
    }\label{fig:add_exp_synth_iid}
\end{center}
\end{figure}
\begin{figure}[ht!]
\begin{center}
    \includegraphics[width = 1\textwidth]{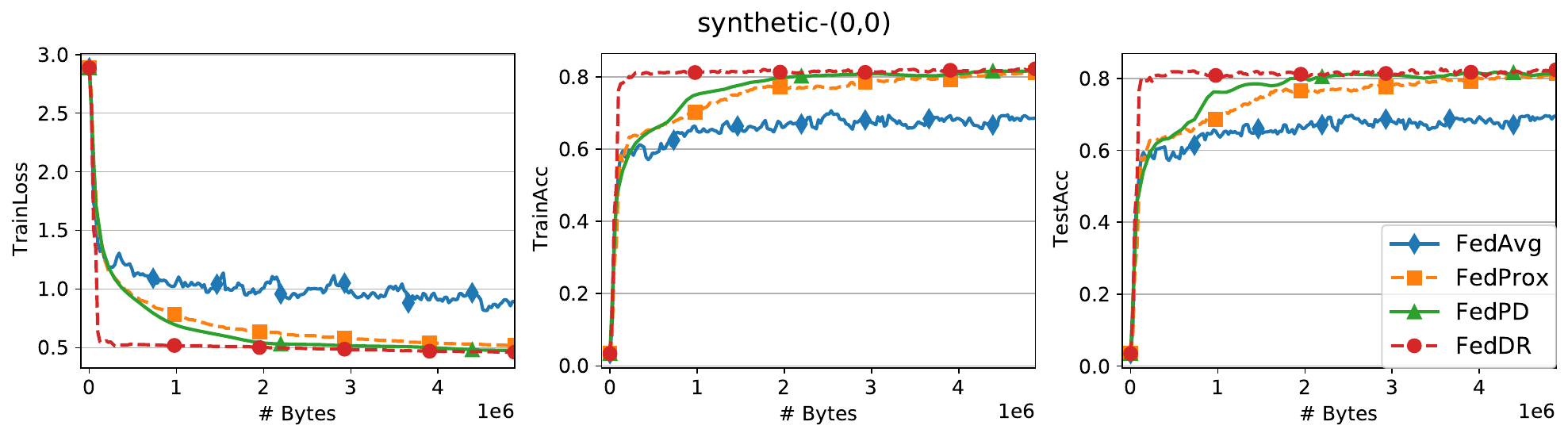}
    \includegraphics[width = 1\textwidth]{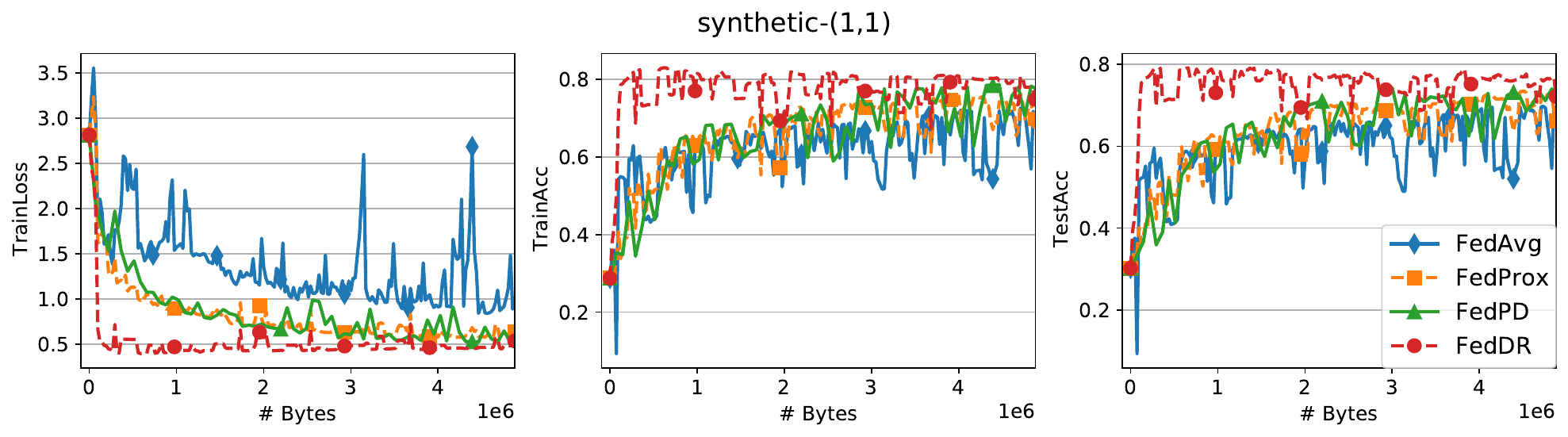}
    \vspace{-1ex}
    \caption{The performance of 4 algorithms without user sampling scheme on non-iid datasets in terms of communication effort.}\label{fig:add_exp_synth_non_iid}
\end{center}
\end{figure}

From Figure~\ref{fig:add_exp_synth_iid}, FedAvg appears to perform best while the other three algorithms are comparable in the iid setting. 
Similar behavior is also observed in \cite{Li_MLSYS2020}. For the non-iid datasets along with Figure~\ref{fig:exp_synth_non_iid}, we observe that the more non-iid the dataset is, the more unstable these algorithms behave. 
In the \texttt{synthetic-(1,1)} dataset, FedDR appears to be the best followed by FedPD. FedProx also performs much better than FedAvg in this test.

Figure~\ref{fig:add_exp_synth_non_iid} depicts the performance of 4 algorithms in terms of communication cost on the \texttt{synthetic-(1,1)} dataset. We still observe that FedDR works well while FedProx and FedPD are comparable but still better than FedAvg.

\begin{figure}[ht!]
\begin{center}
    \includegraphics[width = 1\textwidth]{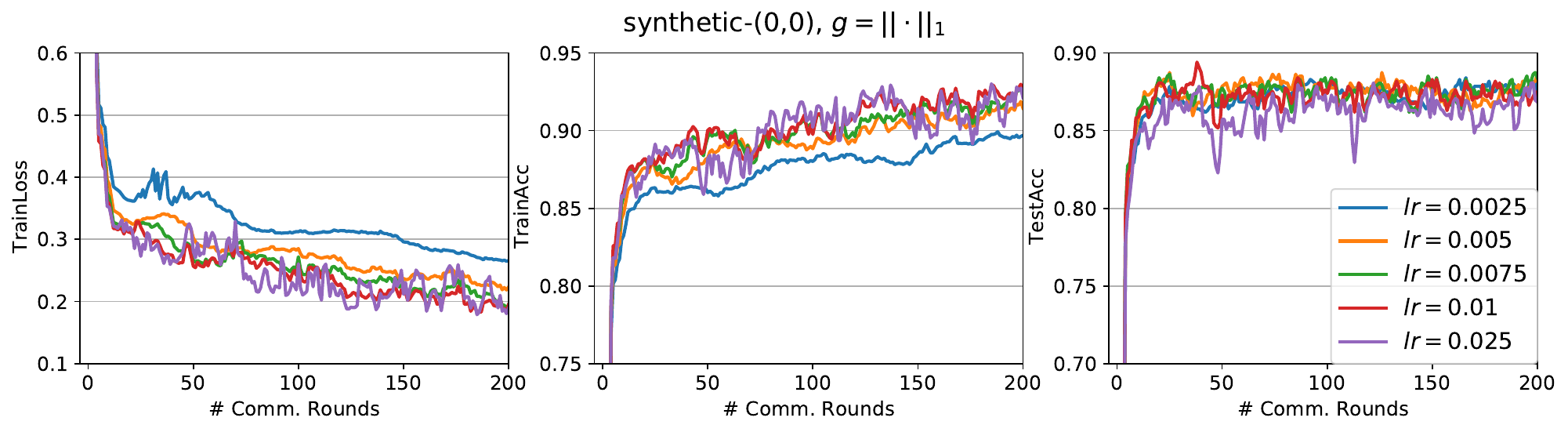}
    \includegraphics[width = 1\textwidth]{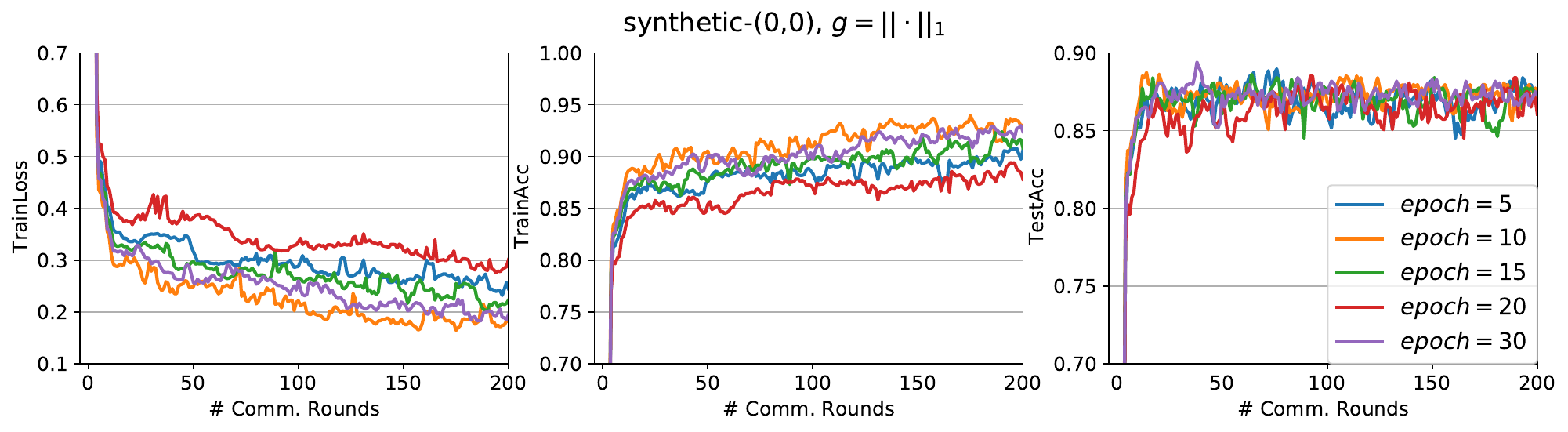}
    \vspace{-1ex}
    \caption{The performance of \textbf{FedDR} on synthetic dataset in composite setting.}\label{fig:exp_comp_synth}
\end{center}
\end{figure}

More results of experiments on the composite setting are presented in Figure~\ref{fig:exp_comp_synth}. 
We observe that the learning rate (\textit{lr}) of SGD needs to be tuned for each dataset and the local iteration should be selected carefully to  trade-off between local computation cost and inexactness of the evaluation of $\prox_{\eta f_i}(y^k_i)$.

\begin{figure}[ht!]
\begin{center}
    \includegraphics[width = 0.99\textwidth]{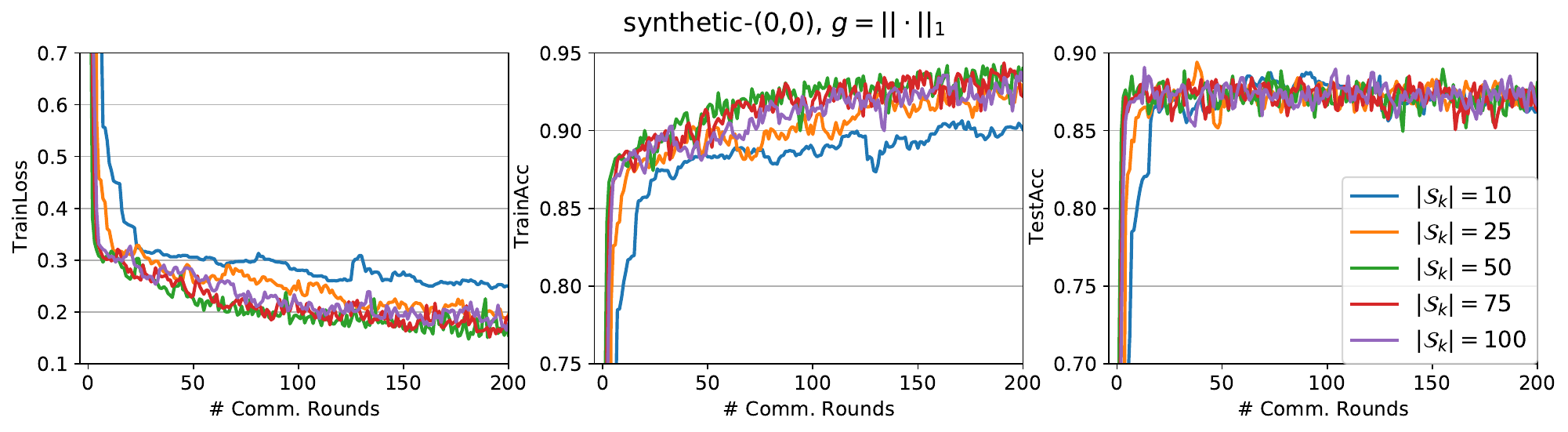} 
    \includegraphics[width = 0.99\textwidth]{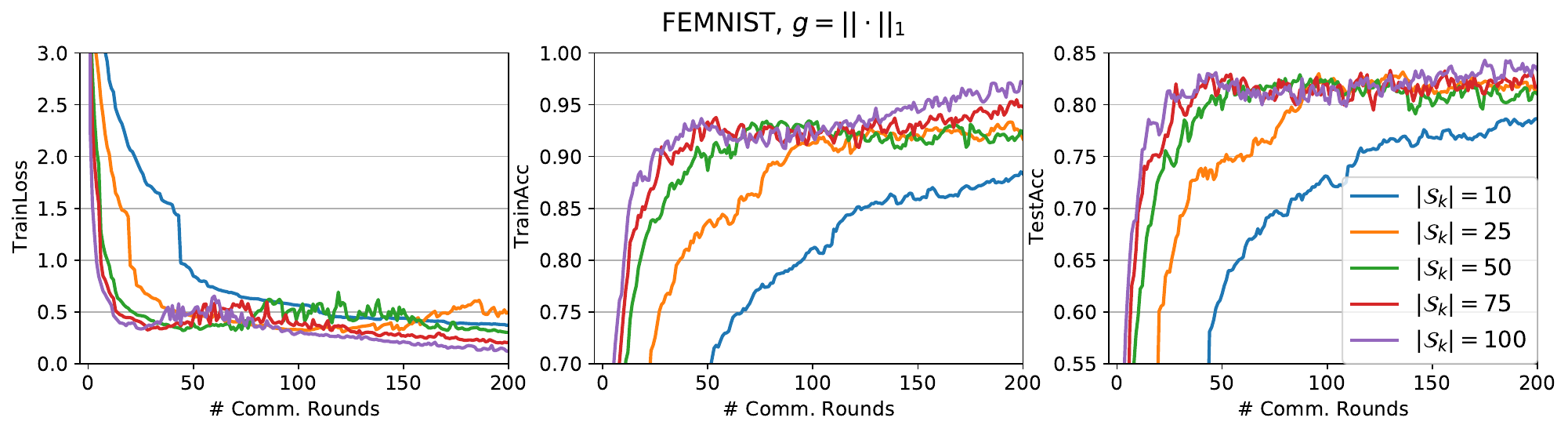} 
    \caption{The performance of \textbf{FedDR} in composite setting in terms of communication rounds.}\label{fig:exp_comp_client1}
\end{center}
\end{figure}

We also vary the number of users sampled at each communication round. 
The results are depicted in Figure~\ref{fig:exp_comp_client1} for two datasets. 
We observe that the performance when we sample smaller number of user per round is not as good as larger ones in terms of communication rounds. 
However, this might not be a fair comparison since fewer clients also require less communication cost. 
Therefore, we plot these results in terms of number of bytes communicated. 
The results are depicted in Figure~\ref{fig:exp_comp_client2}. 
From Figure~\ref{fig:exp_comp_client2}, FedDR performs very similarly under different choices of $\Sc_k$.

\begin{figure}[ht!]
\begin{center}
    \includegraphics[width = 0.99\textwidth]{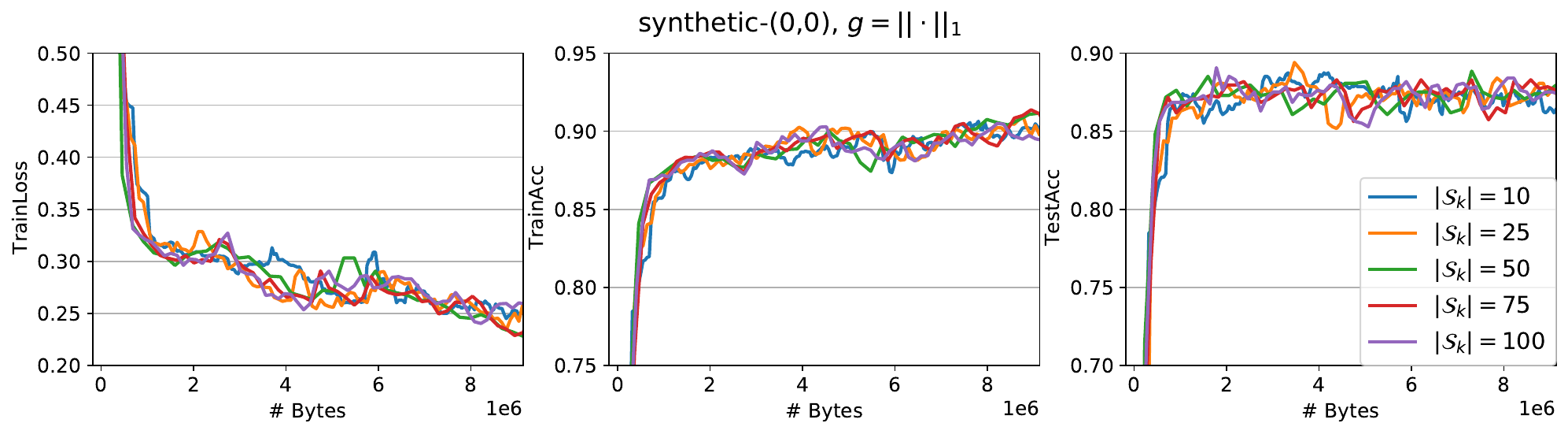}
    \includegraphics[width = 0.99\textwidth]{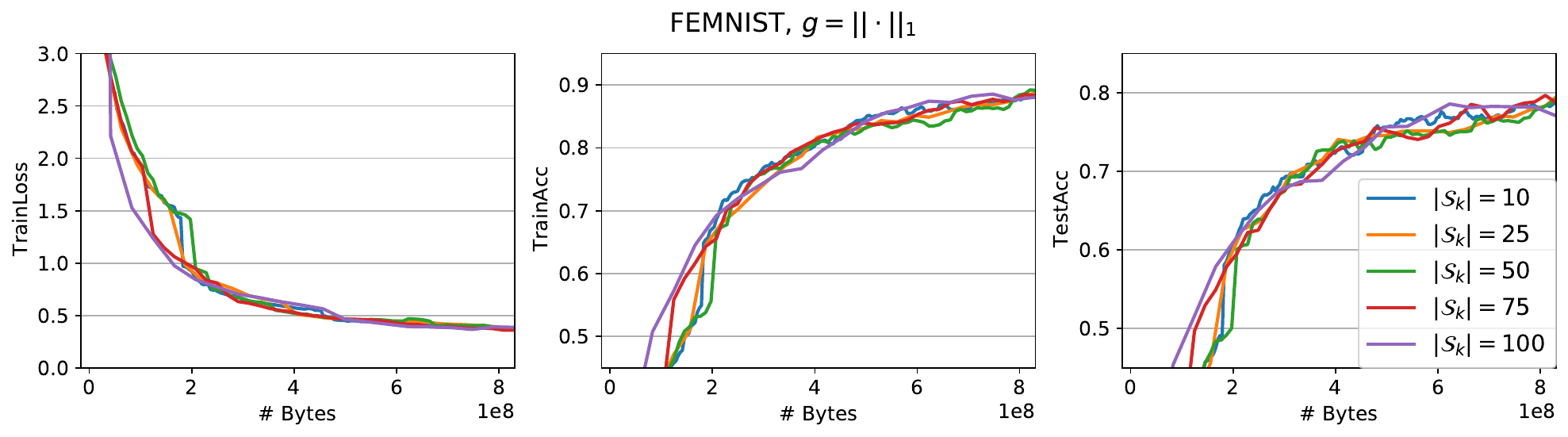}
    \caption{The performance of \textbf{FedDR} in composite setting in terms of number of bytes.}\label{fig:exp_comp_client2}
\end{center}
\end{figure}

We also compare FedDR and asyncFedDR using the FEMNIST dataset. 
The results are depicted in Figure~\ref{fig:add_exp_femnist}. 
We can see that asyncFedDR is advantageous over FedDR to achieve lower loss value and higher accuracies.

\begin{figure}[ht!]
\begin{center}
\includegraphics[width = 0.99\textwidth]{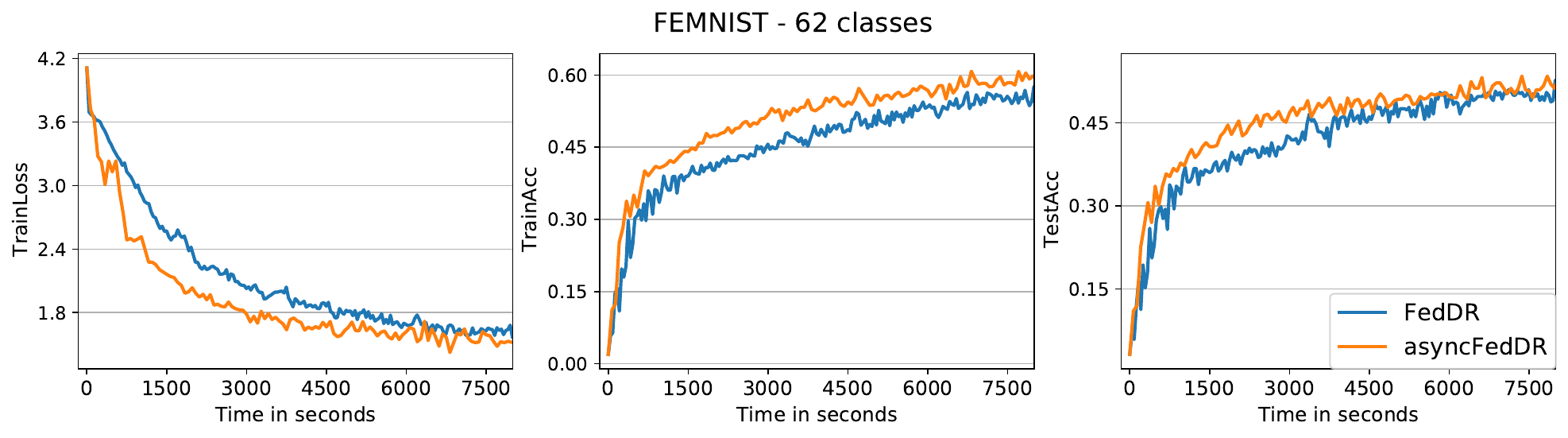}
    \caption{The performance of FedDR and asyncFedDR on \textbf{FEMNIST - 62 classes} dataset.}\label{fig:add_exp_femnist}
\end{center}
\end{figure}

%% file: neurips_2021_FedDR_FinalVersion.bbl
\begin{thebibliography}{10}

\bibitem{Bauschke2011}
H.~H. Bauschke and P.~Combettes.
\newblock {\em Convex analysis and monotone operators theory in {H}ilbert
  spaces}.
\newblock Springer-Verlag, 2nd edition, 2017.

\bibitem{Bertsekas1989b}
D.P. Bertsekas and J.~N. Tsitsiklis.
\newblock {\em {P}arallel and distributed computation: {N}umerical methods}.
\newblock Prentice Hall, 1989.

\bibitem{distBelief}
Jesse Cai.
\newblock {Implementing DistBelief}.
\newblock \url{https://jcaip.github.io/Distbelief/}, 2018.

\bibitem{caldas2018leaf}
S.~Caldas, S.~M.~K. Duddu, P.~Wu, T.~Li, J.~Kone{\v{c}}n{\`y}, H.~B. McMahan,
  V.~Smith, and A.~Talwalkar.
\newblock Leaf: A benchmark for federated settings.
\newblock {\em arXiv preprint arXiv:1812.01097}, 2018.

\bibitem{cannelli2019asynchronous}
L.~Cannelli, F.~Facchinei, V.~Kungurtsev, and G.~Scutari.
\newblock Asynchronous parallel algorithms for nonconvex optimization.
\newblock {\em Math. Program.}, pages 1--34, 2019.

\bibitem{charles2021convergence}
Z.~Charles and J.~Kone{\v{c}}n{\`y}.
\newblock Convergence and accuracy trade-offs in federated learning and
  meta-learning.
\newblock In {\em International Conference on Artificial Intelligence and
  Statistics}, pages 2575--2583. PMLR, 2021.

\bibitem{combettes2018asynchronous}
P.~Combettes and J.~Eckstein.
\newblock Asynchronous block-iterative primal-dual decomposition methods for
  monotone inclusions.
\newblock {\em Math.  Program.}, 168(1):645--672, 2018.

\bibitem{combettes2015stochastic}
P.~L. Combettes and J.-C. Pesquet.
\newblock Stochastic quasi-{F}ej{\'e}r block-coordinate fixed point iterations
  with random sweeping.
\newblock {\em SIAM J. Optim.}, 25(2):1221--1248, 2015.

\bibitem{dao2019lyapunov}
M.~N. Dao and M.~K. Tam.
\newblock A {L}yapunov-type approach to convergence of the
  {D}ouglas--{R}achford algorithm for a nonconvex setting.
\newblock {\em J. Global Optim.}, 73(1):83--112, 2019.

\bibitem{gorbunov2021local}
E.~Gorbunov, F.~Hanzely, and P.~Richt{\'a}rik.
\newblock Local {SGD}: {U}nified theory and new efficient methods.
\newblock In {\em International Conference on Artificial Intelligence and
  Statistics}, pages 3556--3564. PMLR, 2021.

\bibitem{haddadpour2019local}
F.~Haddadpour, M.~M. Kamani, M.~Mahdavi, and V.~Cadambe.
\newblock Local {SGD} with periodic averaging: {T}ighter analysis and adaptive
  synchronization.
\newblock In {\em Advances in Neural Information Processing Systems}, pages
  11082--11094, 2019.

\bibitem{haddadpour2021federated}
F.~Haddadpour, M.~M. Kamani, A.~Mokhtari, and M.~Mahdavi.
\newblock Federated learning with compression: Unified analysis and sharp
  guarantees.
\newblock In {\em International Conference on Artificial Intelligence and
  Statistics}, pages 2350--2358. PMLR, 2021.

\bibitem{haddadpour2019convergence}
F.~Haddadpour and M.~Mahdavi.
\newblock On the convergence of local descent methods in federated learning.
\newblock {\em arXiv preprint arXiv:1910.14425}, 2019.

\bibitem{hanzely2020lower}
F.~Hanzely, S.~Hanzely, S.~Horv{\'a}th, and P.~Richt{\'a}rik.
\newblock Lower bounds and optimal algorithms for personalized federated
  learning.
\newblock {\em arXiv preprint arXiv:2010.02372}, 2020.

\bibitem{karimireddy2020mime}
S.~P. Karimireddy, M.~Jaggi, S.~Kale, M.~Mohri, S.~J. Reddi, S.~U. Stich, and
  A.~A. Suresh.
\newblock Mime: Mimicking centralized stochastic algorithms in federated
  learning.
\newblock {\em arXiv preprint arXiv:2008.03606}, 2020.

\bibitem{karimireddy2020scaffold}
S.~P. Karimireddy, S.~Kale, M.~Mohri, S.~Reddi, S.~Stich, and A.~T. Suresh.
\newblock Scaffold: Stochastic controlled averaging for federated learning.
\newblock In {\em International Conference on Machine Learning}, pages
  5132--5143. PMLR, 2020.

\bibitem{khaled2019first}
A.~Khaled, K.~Mishchenko, and P.~Richt{\'a}rik.
\newblock First analysis of local {GD} on heterogeneous data.
\newblock {\em arXiv preprint arXiv:1909.04715}, 2019.

\bibitem{konevcny2016federated}
J.~Kone{\v{c}}n{\`y}, H.~B. McMahan, D.~Ramage, and P.~Richt{\'a}rik.
\newblock Federated optimization: Distributed machine learning for on-device
  intelligence.
\newblock {\em arXiv preprint arXiv:1610.02527}, 2016.

\bibitem{lecun1998gradient}
Y.~LeCun, L.~Bottou, Y.~Bengio, and P.~Haffner.
\newblock Gradient-based learning applied to document recognition.
\newblock {\em Proceedings of the IEEE}, 86(11):2278--2324, 1998.

\bibitem{li2015global}
G.~Li and T.~K. Pong.
\newblock Global convergence of splitting methods for nonconvex composite
  optimization.
\newblock {\em SIAM J. Optim.}, 25(4):2434--2460, 2015.

\bibitem{li2016douglas}
G.~Li and T.~K. Pong.
\newblock {D}ouglas--{R}achford splitting for nonconvex optimization with
  application to nonconvex feasibility problems.
\newblock {\em Math. Program.}, 159(1-2):371--401, 2016.

\bibitem{li2020federated}
T.~Li, A.~K. Sahu, A.~Talwalkar, and V.~Smith.
\newblock Federated learning: {C}hallenges, methods, and future directions.
\newblock {\em IEEE Signal Processing Magazine}, 37(3):50--60, 2020.

\bibitem{Li_MLSYS2020}
T.~Li, A.~K. Sahu, M.~Zaheer, M.~Sanjabi, A.~Talwalkar, and V.~Smith.
\newblock Federated optimization in heterogeneous networks.
\newblock In I.~Dhillon, D.~Papailiopoulos, and V.~Sze, editors, {\em
  Proceedings of Machine Learning and Systems}, volume~2, pages 429--450, 2020.

\bibitem{li2019convergence}
X.~Li, K.~Huang, W.~Yang, S.~Wang, and Z.~Zhang.
\newblock On the convergence of {FedAvg} on non-iid data.
\newblock In  {\em International Conference on Learning Representations} (ICLR), 2019.

\bibitem{li2021fedbn}
X.~Li, M.~Jiang, X.~Zhang, M.~Kamp, and Q.~Dou.
\newblock {FedBN}: {F}ederated learning on non-iid features via local batch
  normalization.
\newblock In {\em International Conference on Learning Representations} (ICLR), 2020.

\bibitem{lin2018don}
T.~Lin, S.~U. Stich, K.~K. Patel, and M.~Jaggi.
\newblock Don't use large mini-batches, use local {SGD}.
\newblock In {\em International Conference on Learning Representations} (ICLR), 2019.

\bibitem{Lions1979}
P.~L. Lions and B.~Mercier.
\newblock Splitting algorithms for the sum of two nonlinear operators.
\newblock {\em SIAM J. Num. Anal.}, 16:964--979, 1979.

\bibitem{liu2021acceleration}
Y.~Liu, Y.~Xu, and W.~Yin.
\newblock Acceleration of primal--dual methods by preconditioning and simple
  subproblem procedures.
\newblock {\em J. Sci. Comput.}, 86(2):1--34, 2021.

\bibitem{mcmahan2017communication}
B.~McMahan, E.~Moore, D.~Ramage, S.~Hampson, and B.~A. y~Arcas.
\newblock Communication-efficient learning of deep networks from decentralized
  data.
\newblock In {\em Artificial Intelligence and Statistics}, pages 1273--1282.
  PMLR, 2017.

\bibitem{mcmahan_ramage_2017}
B.~McMahan and D.~Ramage.
\newblock Federated learning: Collaborative machine learning without
  centralized training data.
 \newblock {\em Google Research Blog}, Jul 2017.

\bibitem{mcmahan2021advances}
H~Brendan McMahan et~al.
\newblock Advances and open problems in federated learning.
\newblock {\em Foundations and Trends{\textregistered} in Machine Learning},
  14(1), 2021.

\bibitem{nguyen2018sgd}
L.~Nguyen, P.~H. Nguyen, M.~Dijk, P.~Richt{\'a}rik, K.~Scheinberg, and
  M.~Tak{\'a}c.
\newblock {SGD} and {H}ogwild! convergence without the bounded gradients
  assumption.
\newblock In {\em International Conference on Machine Learning}, pages
  3750--3758. PMLR, 2018.

\bibitem{pathak2020fedsplit}
R.~Pathak and M.~J. Wainwright.
\newblock {FedSplit}: {A}n algorithmic framework for fast federated
  optimization.
\newblock In {\em Advances in Neural Information Processing Systems}, vol. 33, pages 7057--7066, 2020.

\bibitem{peng2016arock}
Z.~Peng, Y.~Xu, M.~Yan, and W.~Yin.
\newblock {AR}ock: an algorithmic framework for asynchronous parallel
  coordinate updates.
\newblock {\em SIAM J. Scientific Comput.}, 38(5):2851--2879, 2016.

\bibitem{Recht2011}
Benjamin Recht, Christopher Re, Stephen Wright, and Feng Niu.
\newblock {Hogwild!: A Lock-Free Approach to Parallelizing Stochastic Gradient
  Descent}.
\newblock In J.~Shawe-Taylor, R.~S. Zemel, P.~L. Bartlett, F.~Pereira, and
  K.~Q. Weinberger, editors, {\em Advances in Neural Information Processing
  Systems 24}, pages 693--701. Curran Associates, Inc., 2011.

\bibitem{richtarik2016parallel}
P.~Richt{\'a}rik and M.~Tak{\'a}{\v{c}}.
\newblock Parallel coordinate descent methods for big data optimization.
\newblock {\em Math. Program.}, 156(1-2):433--484, 2016.

\bibitem{Rockafellar1976b}
R.T. Rockafellar.
\newblock Monotone operators and the proximal point algorithm.
\newblock {\em SIAM J. Control Optim.}, 14:877--898, 1976.

\bibitem{shamir2014communication}
O.~Shamir, N.~Srebro, and T.~Zhang.
\newblock Communication-efficient distributed optimization using an approximate
  newton-type method.
\newblock In {\em International conference on machine learning}, pages
  1000--1008, 2014.

\bibitem{stich2018local}
S.~U. Stich.
\newblock Local {SGD} converges fast and communicates little.
\newblock In {\em International Conference on Learning Representations}, pages 1--17, 2018.

\bibitem{themelis2020douglas}
A.~Themelis and P.~Patrinos.
\newblock {D}ouglas--{R}achford splitting and {ADMM} for nonconvex
  optimization: {T}ight convergence results.
\newblock {\em SIAM J. Optim.}, 30(1):149--181, 2020.

\bibitem{wang2018cooperative}
J.~Wang and G.~Joshi.
\newblock Cooperative {SGD}: {A} unified framework for the design and analysis
  of communication-efficient {SGD} algorithms.
\newblock {\em ICML Workshop on Coding Theory for Machine Learning}, pages 1--5, 2019.

\bibitem{woodworth2020minibatch}
B.~Woodworth, K.~K. Patel, and N.~Srebro.
\newblock {M}inibatch vs local {SGD} for heterogeneous distributed learning.
\newblock  In {\em Advances in Neural Information Processing Systems} 33 (NeurIPS 2020),  2020.

\bibitem{woodworth2020local}
B.~Woodworth, K.~K. Patel, S.~U. Stich, Z.~Dai, B.~Bullins, H.~B. McMahan,
  O.~Shamir, and N.~Srebro.
\newblock {I}s local {SGD} better than minibatch {SGD}?
\newblock In {\em International Conference on Machine Learning} (ICML), pages 10334--10343, 2020.

\bibitem{xie2019asynchronous}
C.~Xie, S.~Koyejo, and I.~Gupta.
\newblock Asynchronous federated optimization.
\newblock {\em arXiv preprint arXiv:1903.03934}, 2019.

\bibitem{yu2019parallel}
H.~Yu, S.~Yang, and S.~Zhu.
\newblock Parallel restarted sgd with faster convergence and less
  communication: Demystifying why model averaging works for deep learning.
\newblock In {\em Proceedings of the AAAI Conference on Artificial
  Intelligence}, volume~33, pages 5693--5700, 2019.

\bibitem{yu2020fed+}
P.~Yu, L.~Wynter, and S.~H. Lim.
\newblock Fed+: A family of fusion algorithms for federated learning.
\newblock {\em arXiv preprint arXiv:2009.06303}, 2020.

\bibitem{yuan2020federated}
H.~Yuan, M.~Zaheer, and S.~Reddi.
\newblock Federated composite optimization.
\newblock {\em International Conference on Machine Learning}, pages 12253--12266, 2021.

\bibitem{zhang2016parallel}
J.~Zhang, C.~De~Sa, I.~Mitliagkas, and C.~R{\'e}.
\newblock Parallel {SGD}: {W}hen does averaging help?
\newblock {\em arXiv preprint arXiv:1606.07365}, 2016.

\bibitem{zhang2020fedpd}
X.~Zhang, M.~Hong, S.~Dhople, W.~Yin, and Y.~Liu.
\newblock {FedPD}: {A} federated learning framework with optimal rates and
  adaptivity to non-iid data.
\newblock {\em arXiv preprint arXiv:2005.11418}, 2020.

\bibitem{zhao2018federated}
Y.~Zhao, M.~Li, L.~Lai, N.~Suda, D.~Civin, and V.~Chandra.
\newblock Federated learning with non-iid data.
\newblock {\em arXiv preprint arXiv:1806.00582}, 2018.

\bibitem{zhu2019deep}
L.~Zhu, Z.~Liu, and S.~Han.
\newblock Deep leakage from gradients.
\newblock In {\em Advances in Neural Information Processing Systems}, pages
  14774--14784, 2019.

\end{thebibliography}
